%% file: iclr2020_conference.tex
\pdfoutput=1
\documentclass{article} % For LaTeX2e
\usepackage{iclr2020_conference,times}

\input{math_commands.tex}

\usepackage{hyperref}
\usepackage{url}
\usepackage{epsfig}
\usepackage{graphicx}
\usepackage{diagbox}
\graphicspath{ {./images/} }
\usepackage{amsmath}
\usepackage{amssymb}
\usepackage{amsthm}
\usepackage{mathrsfs}
\usepackage{algorithm,algorithmic}

\usepackage{graphicx}
\usepackage{subcaption}
\usepackage{enumitem}
\usepackage[normalem]{ulem}
\usepackage[nocompress]{cite}

\newtheorem{proposition}{Proposition}[section]

\DeclareMathOperator*{\argminB}{argmin}   % Jan Hlavacek
\newcommand{\norm}[1]{\left\lVert#1\right\rVert}

\title{Robust Subspace Recovery Layer for \\Unsupervised Anomaly Detection}

\makeatletter
\newcommand{\printfnsymbol}[1]{%
  \textsuperscript{\@fnsymbol{#1}}%
}

\author{Chieh-Hsin Lai\thanks{ Equal contribution.}, Dongmian Zou\printfnsymbol{1}\& Gilad Lerman  \\
School of Mathematics\\
University of Minnesota\\
Minneapolis, MN 55455 \\
\texttt{\{laixx313, dzou, lerman\}@umn.edu} 
}

\iclrfinalcopy 
\begin{document}

\maketitle

\begin{abstract}
We propose a neural network for unsupervised anomaly detection with a novel robust subspace recovery layer (RSR layer). This layer seeks to extract the underlying subspace from a latent representation of the given data and removes outliers that lie away from this subspace. It is used within an autoencoder. The encoder maps the data into a latent space, from which the RSR layer extracts the subspace. The decoder then smoothly maps back the underlying subspace to a ``manifold" close to the original inliers. Inliers and outliers are distinguished according to the distances between the original and mapped positions (small for inliers and large for outliers). Extensive numerical experiments with both image and document datasets demonstrate state-of-the-art precision and recall. 
\end{abstract}

\section{Introduction}\label{sec:intro}

Finding and utilizing patterns in data is a common task for modern machine learning  systems. 
However, there is often some anomalous information that does not follow a common pattern and has to be recognized. For this purpose, anomaly detection aims to identify data points that ``do not conform to expected behavior'' \citep{chandola2009anomaly}. We refer to such points as either anomalous or outliers. In many applications, there is no ground truth available to distinguish anomalous from normal points, and they need to be detected in an unsupervised fashion. For example, one may need to remove anomalous images from a set of images obtained by a search engine without any prior knowledge about how a normal image should look \citep{xia2015learning}. Similarly, one may need to distinguish unusual news items from a large collection of news documents without any information whether a news item is usual or not \citep{kannan2017outlier}. In these examples, the only assumptions are that normal data points appear more often than anomalous ones and have a simple underlying structure which is unknown to the user.

%there should be a space here

Some early methods for anomaly detection relied on Principal Component Analysis (PCA) \citep{shyu2003novelnew}. Here one assumes that the underlying unknown structure of the normal samples is linear. However, PCA is sensitive to outliers and will often not succeed in recovering the linear structure or identifying the outliers \citep{lerman2018overview,vaswani2018static}. More recent ideas of Robust PCA (RPCA) \citep{wright2009robust,vaswani2018static} have been considered for some specific problems of anomaly detection or removal \citep{zhou2017anomaly, paffenroth2018robust}. RPCA assumes sparse corruption, that is, few elements of the data matrix are corrupted. This assumption is natural for some special problems in computer vision, in particular, background subtraction \citep{Torre:03, wright2009robust,vaswani2018static}. However, a natural setting of anomaly detection with hidden linear structure may assume instead that a large portion of the data points are fully corrupted. The mathematical framework that addresses this setting is referred to as robust subspace recovery (RSR) \citep{lerman2018overview}.

While Robust PCA and RSR try to extract linear structure or identify outliers lying away from such structure, the underlying geometric structure of many real datasets is nonlinear. Therefore, one needs to extract crucial features of the nonlinear structure of the data while being robust to outliers. In order to achieve this goal, we propose to use an autoencoder (composed of an encoder and a decoder) with an RSR layer. We refer to it as RSRAE (RSR autoencoder).
It aims to robustly and nonlinearly reduce the dimension of the data in the following way. The encoder maps the data into a high-dimensional space. The RSR layer linearly maps 
the embedded points into a low-dimensional subspace that aims to learn the hidden linear structure of the embedded normal points. The decoder maps the points from this subspace to the original space. It aims to map the normal points near their original locations, and the anomalous points far from their original locations.

Ideally, the encoder maps the normal data to a linear space and any anomalies lie away from this subspace. In this ideal scenario, anomalies can be removed by an RSR method directly applied to the data embedded by the encoder. Since the linear model for the normal data embedded by the encoder is only approximate, we do not directly apply RSR to the embedded data. Instead, we  minimize a sum of the reconstruction error of the autoencoder and the RSR error for the data embedded by the encoder. We advocate for an alternating procedure, so that the parameters of the autoencoder and the RSR layer are optimized in turn.

\subsection{Structure of the rest of the paper}
\Secref{sec:relatedwork} reviews works that are directly related to the proposed RSRAE and highlights the original contributions of this paper. \Secref{sec:alg_describe} explains the proposed RSRAE, and in particular, its RSR layer and total energy function. \Secref{sec:real} includes extensive experimental evidence demonstrating  effectiveness of RSRAE with both image and document data. 
\Secref{sec:relatedtheory} discusses theory for the relationship of the RSR penalty with the WGAN penalty.
\Secref{sec:conclude} summarizes this work and mentions future directions. 

\section{Related Works and Contribution}\label{sec:relatedwork}

We review related works in \Secref{subsec:relatedworks} and highlight our contribution in \Secref{subsec:contribution}.

\subsection{Related Works}\label{subsec:relatedworks}
Several recent works have used autoencoders for anomaly detection. 
\citet{xia2015learning} proposed the earliest work on anomaly detection via an autoencoder, while utilizing large reconstruction error of outliers. They apply an iterative and cyclic scheme, where in each iteration, they determine the inliers and use them for updating the parameters of the autoencoder. \citet{aytekin2018clustering} apply $\ell_{2}$ normalization for the latent code of the autoencoder and also consider the case of multiple modes for the normal samples. Instead of using the reconstruction error, they apply $k$-means clustering for the latent code, and identify outliers as points whose latent representations are far from all the cluster centers. \citet{zong2018deep} also use an autoencoder with clustered latent code, but they fit a Gaussian Mixture Model using an additional neural network.
Restricted Boltzmann Machines (RBMs) are similar to autoencoders. \citet{zhai2016deep} define ``energy functions'' for RBMs that are similar to the reconstruction losses for autoencoders. They identify anomalous samples according to large energy values. \citet{chalapathy2017robust} propose using ideas of RPCA within an autoencoder, where they alternatively optimize the parameters of the autoencoder and a sparse residual matrix.

The above works are designed for datasets with a small fraction of outliers. However, when this fraction increases, outliers are often not distinguished by high reconstruction errors or low similarity scores. In order to identify them, additional assumptions on the structure
of the normal data need to be incorporated. For example, \citet{zhou2017anomaly} decompose the input data into two parts: low-rank and sparse (or column-sparse). The low-rank part is fed into an autoencoder and the sparse part is imposed as a penalty term with the $\ell_1$-norm (or $\ell_{2,1}$-norm for column-sparsity). 

In this work, we use a term analogous to the $\ell_{2,1}$-norm, which can be interpreted as the sum of absolute deviations from a latent subspace. However, we do not decompose the data a priori, but minimize an energy combining this term and the reconstruction error.  Minimization of the former term is known as least absolute deviations in RSR  \citep{lerman2018overview}. %This energy is rotation invariant and its minimization
It was first suggested for RSR and related problems in \citet{watson2001some, ding2006r, zhang2009median}. The robustness to outliers of this energy, or of relaxed versions of it, was studied in \citet{mccoy2011two, xu2012robust, lp_recovery_part1_11, zhang2014novel, lerman2015robust, lerman2017fast, maunu2017well}. In particular, \citet{maunu2017well} established its well-behaved landscape under special, though natural, deterministic conditions. Under similar conditions, they guaranteed fast subspace recovery by a simple algorithm that aims to minimize this energy.

Another directly related idea 
for extracting useful latent features 
is an addition of a linear self-expressive layer to an autoencoder \citep{ji2017deep}. It is used in the different setting of unsupervised subspace clustering. By imposing the self-expressiveness, the autoencoder is robust to an increasing number of clusters. Although self-expressiveness also improves robustness to noise and outliers,
\citet{ji2017deep} aims at clustering and thus its goal is different than ours. Furthermore, their self-expressive energy does not explicitly consider robustness, while ours does.
\citet{lezama2018ole} consider a somewhat parallel idea of imposing a loss function to increase the robustness of representation. However, their goal is to increase the margin between classes and their method only applies to a supervised setting in anomaly detection, where the normal data is multi-modal.

%Other types of methods were reviewed in \citep{chandola2009anomaly}, but they are not directly relevant to our work.

\subsection{Contribution of this work}
\label{subsec:contribution}
This work introduces an RSR layer within an autoencoder. 
It incorporates a special regularizer that enforces an outliers-robust linear structure in the embedding obtained by the encoder. 
We clarify that the method does not alternate between application of the autoencoder and the RSR layer, but fully integrates these two components. Our experiments demonstrate that a simple incorporation of a ``robust loss'' within a regular autoencoder does not work well for anomaly detection. We try to explain this and also the improvement obtained by incorporating an additional RSR layer.

Our proposed architecture is simple to implement. Furthermore, the RSR layer is not limited to a specific design of RSRAE but can be put into any well-designed autoencoder structure. 
The epoch time of the proposed algorithm is comparable to those of other common autoencoders. 
Furthermore, our experiments show that RSRAE competitively performs in unsupervised anomaly detection tasks.

{RSRAE addresses the unsupervised setting, but is not designed to be highly competitive in the semi-supervised or supervised settings, where one has access to training data from the normal class or from both classes, respectively. In these settings, RSRAE functions like a regular autoencoder without taking an advantage of its RSR layer, unless the training data for the normal class is  corrupted with outliers.}

The use of RSR is not restricted to autoencoders. 
We establish some preliminary analysis for RSR within a generative adversarial network (GAN) \citep{goodfellow2014generative, arjovsky2017wasserstein} in \Secref{sec:relatedtheory}.
More precisely, we show that a linear WGAN intrinsically incorporates RSR in some special settings, although it is unclear how to impose 
an RSR layer. 

\section{RSR layer for outlier removal}
\label{sec:alg_describe}
We assume input data $\{\rvx^{(t)}\}_{t=1}^N$ in $\R^M$, and denote by $\rmX$ its corresponding data matrix, whose $t$-th column is $\rvx^{(t)}$. The encoder of RSRAE, $\mathscr{E}$, is a neural network that maps each data point, $\rvx^{(t)}$, to its latent code $\rvz^{(t)} = \mathscr{E}(\rvx^{(t)}) \in \R^D$. The RSR layer is a linear transformation $\rmA \in \R^{d \times D}$ that reduces the dimension to $d$. That is, $\tilde{\rvz}^{(t)} = \rmA \rvz^{(t)} \in \R^d$. The decoder $\mathscr{D}$ is a neural network that maps $\tilde{\rvz}^{(t)}$ to $\tilde{\rvx}^{(t)}$ in the original ambient space $\R^M$.

We can write the forward maps in a compact form using the corresponding data matrices as follows:
\begin{equation}
    \rmZ = \mathscr{E}(\rmX), \hspace{0.3cm} \tilde{\rmZ} = \rmA \rmZ, \hspace{0.3cm} \tilde{\rmX} = \mathscr{D}(\tilde{\rmZ}).
\end{equation}

Ideally, we would like to optimize RSRAE so it only maintains the underlying structure of the normal data. We assume that the original normal data lies on a $d$-dimensional ``manifold'' in $\R^D$ and thus the RSR layer embeds its latent code into $\R^d$. In this ideal optimization setting, the similarity between the input and the output of RSRAE is large whenever the input is normal and small whenever the input is anomalous. Therefore, by thresholding a similarity measure,
one may distinguish between normal and anomalous data points.

In practice, the matrix $\rmA$ and the parameters of $\mathscr{E}$ and $\mathscr{D}$ are  obtained by minimizing a loss function, which is a sum of two parts: the reconstruction loss from the autoencoder and the loss from the RSR layer. For $p>0$, an $\ell_{2,p}$ reconstruction loss for the autoencoder is
\begin{equation}\label{eq:RSRAElossAE}
    L_{\rm{AE}}^p (\mathscr{E}, \rmA, \mathscr{D}) = \sum_{t=1}^N \norm{\rvx^{(t)} - \tilde{\rvx}^{(t)}}_2^p ~.
\end{equation}

In order to motivate our choice of RSR loss, we review a common formulation for the original RSR problem. In this problem one needs to recover a linear subspace, or equivalently an orthogonal projection $\rmP$ onto this subspace. Assume a dataset $\{\rvy^{(t)}\}_{t=1}^N$ and let $\rmI$ denote the identity matrix in the ambient space of the dataset. The goal is to find an orthogonal projector $\rmP$ of dimension $d$ whose subspace robustly approximates this dataset. The least $q$-th power deviations formulation for $q>0$, or least absolute deviations when $q=1$ \citep{lerman2018overview}, seeks $\rmP$ that minimizes %the loss function 
\begin{equation}\label{eq:lossrsrorig}
    \hat{L}(\rmP) = \sum_{t=1}^N \norm{ \left( \rmI - \rmP \right) \rvy^{(t)}}_2^q ~.
\end{equation}
The solution of this problem is robust to some outliers when $q \leq 1$ \citep{lp_recovery_part1_11, lerman2017fast}; furthermore, $q<1$ can result in a wealth of local minima and thus $q=1$ is preferable \citep{lp_recovery_part1_11, lerman2017fast}.

A similar loss function to \eqref{eq:lossrsrorig} for RSRAE is
\begin{equation}\label{eq:lossRSR}
\begin{split}
    L_{\rm{RSR}}^q (\rmA) ~=~ & \lambda_1 L_{\rm{RSR_1}}(\rmA) + \lambda_2 L_{\rm{RSR_2}}(\rmA) \\
    ~:=~ & \lambda_1 \sum_{t=1}^N \norm{\rvz^{(t)} - \rmA^{\rm{T}}\underset{\tilde{\rvz}^{(t)}}{\underbrace{
    \rmA \rvz^{(t)}}}}_2^q + \lambda_2 \norm{\rmA \rmA^{\rm{T}} - \rmI_d}_{\rm{F}}^2 ~,
\end{split}
\end{equation}
where $\rmA^{\rm{T}}$ denotes the transpose of $\rmA$, $\rmI_d$ denotes the $d \times d$ identity matrix and $\norm{\cdot}_{\rm{F}}$ denotes the Frobenius norm. Here $\lambda_1, \lambda_2 > 0$ are predetermined hyperparameters, though we later show that one may solve the underlying problem without using them. We note that the first term in the weighted sum of \eqref{eq:lossRSR} is close to \eqref{eq:lossrsrorig} as long as $\rmA^{\rm{T}} \rmA$ is close to an orthogonal projector. To enforce this requirement we introduced the second term in the weighted sum of \eqref{eq:lossRSR}. 
In Appendix~\ref{sec:more_RSR_term} we discuss further properties of the RSR energy and its minimization. 

To emphasize the effect of outlier removal, we take $p=1$ in \eqref{eq:RSRAElossAE} and $q=1$ in \eqref{eq:lossRSR}. That is, we use the $l_{2,1}$ norm, or the formulation of least absolute deviations, for both  reconstruction and RSR. The loss function of RSRAE is the sum of
the two loss terms in \eqref{eq:RSRAElossAE} and \eqref{eq:lossRSR}, that is, \begin{equation}
\label{eq:combined}
L_{\rm{RSRAE}}(\mathscr{E}, \rmA, \mathscr{D}) = L_{\rm{AE}}^1 (\mathscr{E}, \rmA, \mathscr{D}) + L_{\rm{RSR}}^1 (\rmA). \end{equation}

We remark that the sole minimization of $L_{\rm{AE}}^1$, without $L_{\rm{RSR}}^1$, is not effective for anomaly detection. We  numerically demonstrate this in \Secref{subsec:cprnorm} and also try to 
explain it in Section \ref{sec:heuristic}.

Our proposed algorithm for optimizing 
\eqref{eq:combined}, which we refer to as the RSRAE algorithm, uses alternating minimization. It iteratively backpropagates the three terms $L_{\rm{AE}}^1$, $L_{\rm{RSR_1}}$, $L_{\rm{RSR_2}}$ and accordingly updates the parameters of the RSR autoencoder. 
For clarity, we describe this basic procedure in Algorithm~\ref{alg:the_alg} of Appendix~\ref{sec:alg}.
It is independent of the values of the parameters $\lambda_1$ and $\lambda_2$.
Note that the additional gradient step with respect to the RSR loss just updates the parameters in $\rmA$. Therefore it does not significantly increase the epoch time of a standard autoencoder for anomaly detection. Another possible method, which we refer to as RSRAE+, is direct minimization of $L_{\rm{RSRAE}}$ with predetermined $\lambda_1$ and $\lambda_2$ via auto-differentiation (see Algorithm \ref{alg:the_alg_plus} of Appendix~\ref{sec:alg}). \Secref{subsec:cprnorm} and Appendix~\ref{subsec:cprnormnotshown} demonstrate that in general, RSRAE performs better than RSRAE+, though it is possible that similar performance can be achieved by carefully tuning the parameters $\lambda_1$ and $\lambda_2$ when implementing RSRAE+. 

We remark that a standard autoencoder is obtained by minimizing only $L_{\rm{AE}}^2$, without the RSR loss. One might hope that minimizing $L_{\rm{AE}}^1$ may introduce the needed robustness. However, \Secref{subsec:cprnorm} and Appendix~\ref{subsec:cprnormnotshown} demonstrate that results obtained by minimizing $L_{\rm{AE}}^1$ or $L_{\rm{AE}}^2$ are comparable, and are worse than those of RSRAE and RSRAE+.

\section{{Experimental Results}}
\label{sec:real}

We test our method \footnote{{Our implementation is available at \url{https://github.com/dmzou/RSRAE.git}}}on five datasets: Caltech 101 \citep{fei2007learning}, Fashion-MNIST \citep{xiao2017fashion}, Tiny Imagenet (a small subset of Imagenet \citep{russakovsky2015imagenet}), Reuters-21578 \citep{lewis1997reuters} and 20 Newsgroups \citep{Lang95}.

\textbf{Caltech 101} contains 9,146 RGB images labeled according to 101 distinct object categories. We take the 11 categories that contain at least 100 images and randomly choose 100 images per category. We preprocess all 1100 images to have  size 32 $\times$ 32 $\times$ 3 and pixel values normalized between $-1$ and $1$. In each experiment, the inliers are the 100 images from a certain category and we sample $c$ $\times$ 100 outliers from the rest of 1000 images of other categories, where $c \in \{0.1, 0.3, 0.5, 0.7, 0.9\}$.

\textbf{Fashion-MNIST} contains 28 $\times$ 28 grayscale images of clothing and accessories, which are categorized into 10 classes. We use the test set which contains 10,000 images and  normalize pixel values to lie in $[-1,1]$. In each experiment, we fix a class and the inliers are the test images in this class. We randomly sample $c$ $\times$ 1,000 outliers from the rest of classes (here and below $c$ is as above). Since there are around 1000 test images in each class, the outlier ratio is approximately $c$.

\textbf{Tiny Imagenet} contains 200 classes of RGB images from a distinct subset of Imagenet. We select 10 classes with 500 training images per class. We preprocess the images to have size 32 $\times$ 32 $\times$ 3 and pixel values in $[-1,1]$. We further represent the images by deep features obtained by a ResNet \citep{he2016identity} with dimension 256 (Appendix~\ref{subsec:tinyimagenetwithoutdeep} provides results for the raw images).
In each experiment, 500 inliers are from a fixed class and $c \times 500$ outliers are from the rest of classes. 

\textbf{Reuters-21578} contains 90 text categories with multi-labels. We consider the five largest classes with single labels and randomly sample from them 360 documents per class. The documents are preprocessed into vectors of size 26,147 by sequentially applying the TFIDF transformer and Hashing vectorizer \citep{rajaraman2011mining}. In each experiment, the inliers are the documents of a fixed class and $c$ $\times$ 360 outliers are randomly sampled from the other classes. 

\textbf{20 Newsgroups} contains newsgroup documents with 20 different labels. We sample 360 documents per class and preprocess them as above into vectors of size 10,000.
In each experiment, the inliers are the documents from a fixed class and $c$ $\times$ 360 outliers are sampled from the other classes.

\subsection{Benchmarks and setting}\label{subsec:benchmark}

We compare RSRAE with the following benchmarks: Local Outlier Factor (LOF) \citep{breunig2000lof}, One-Class SVM (OCSVM) \citep{scholkopf2000support, amer2013enhancing}, Isolation Forest (IF) \citep{liu2012isolation}, Deep Structured Energy Based Models (DSEBMs) \citep{zhai2016deep}, Geometric Transformations (GT) \citep{golan2018deep}, and Deep Autoencoding Gaussian Mixture Model (DAGMM) \citep{zong2018deep}. Of those benchmarks, LOF, OCSVM and IF are traditional, while powerful methods, for unsupervised anomaly detection and do not involve neural networks. DSEBMs, DAGMM and GT are more recent and all involve neural networks. DSEBMs is built for unsupervised anomaly detection. DAGMM and GT are designed for semi-supervised anomaly detection, but allow corruption. We use them to learn a model for the inliers and assign anomaly scores using the combined set of both inliers and outliers. GT only applies to image data. We briefly describe these methods in Appendix \ref{sec:describe_baselines}.

We implemented DSEBMs, DAGMM and GT using the codes\footnote{\url{https://github.com/izikgo/AnomalyDetectionTransformations}} from \citet{golan2018deep} with minimal modification so that they adapt to the data described above and the available GPUs in our machine. The LOF, OCSVM and IF methods are adapted from the scikit-learn packages.

We describe the structure of the RSRAE as follows. For the image datasets without deep features, the encoder consists of three convolutional layers: 5 $\times$ 5 kernels with 32 output channels, strides 2; 5 $\times$ 5 kernels with 64 output channels, strides 2; and 3 $\times$ 3 kernels with 128 output channels, strides 2. The output of the encoder is flattened and the RSR layer transforms it into a 10-dimensional vector. That is, we fix $d=10$ in all experiments. The decoder consists of a dense layer that maps the output of the RSR layer into a vector of the same shape as the output of the encoder, and three deconvolutional layers: 3 $\times$ 3 kernels with 64 output channels, strides 2; $5 \times 5$ kernels with 32 output channels, strides 2; $5 \times 5$ kernels with 1 (grayscale) or 3 (RGB) output channels, strides 2. For the preprocessed document datasets or the deep features of Tiny Imagenet, the encoder is a fully connected network with size (32, 64, 128), the RSR layer linearly maps the output of the encoder to dimension 10, and the decoder is a fully connected network with size (128, 64, 32, $D$) where $D$ is the dimension of the input. Batch normalization is applied to each layer of the encoders and the decoders. The output of the RSR layer is $\ell_2$-normalized before applying the decoder. For DSEBMs and DAGMM we use the same number of layers and the same dimensions in each layer for the autoencoder as in RSRAE. For each experiment, the RSRAE model is optimized with Adam using a learning rate of 0.00025 and 200 epochs. The batch size is 128 for each gradient step. The setting of training is consistent for all the neural network based methods. 

{The two main hyperparameters of RSRAE are the intrinsic dimension $d$ and learning rate. Their values were fixed above. Appendix~\ref{sec:sensitivity} demonstrates stability to changes in these values.}

All experiments were executed on a Linux machine with 64GB RAM and four GTX1080Ti GPUs.
For all experiments with neural networks, we used TensorFlow and Keras. We report runtimes in  Appendix~\ref{sec:runtime}.

\subsection{Results}\label{subsec:res}

We summarize the precision and recall of our experiments by the AUC (area under curve) and AP (average precision) scores. For completeness, we include the definitions of these common scores in Appendix~\ref{sec:describe_baselines}.  We
compute them by considering the outliers as ``positive''.
We remark that we did not record the precision-recall-F1 scores, as in \citet{xia2015learning,zong2018deep}, since in practice it requires knowledge of the outlier ratio.

Figs.~\ref{fig:aucapall} and \ref{fig:aucapall2} present the AUC and AP scores of RSRAE and the methods described in \Secref{subsec:benchmark} for the datasets described above, where GT is only applied to image data without deep features. For each constant $c$ (the outlier ratio) and each method, we average the AUC and AP scores over 5 runs with different random initializations and also compute the standard deviations. For brevity of presentation, we report the averaged scores among all classes and designate the averaged standard deviations by bars. 

The results indicates that RSRAE clearly outperforms other methods in most cases, especially when $c$ is large. Indeed, the RSR layer was designed to handle large outlier ratios. For Fashion MNIST and Tiny Imagenet with deep features, IF performs similarly to RSRAE, but IF performs poorly on the document datasets. OCSVM is the closest to RSRAE for the document datasets but it is generally not so competitive for the image datasets.

\begin{figure}[b!]
\centering
\begin{minipage}[t]{0.48\textwidth}
\rotatebox{90}{\null \qquad Caltech 101}
\centering
\includegraphics[width=6cm]{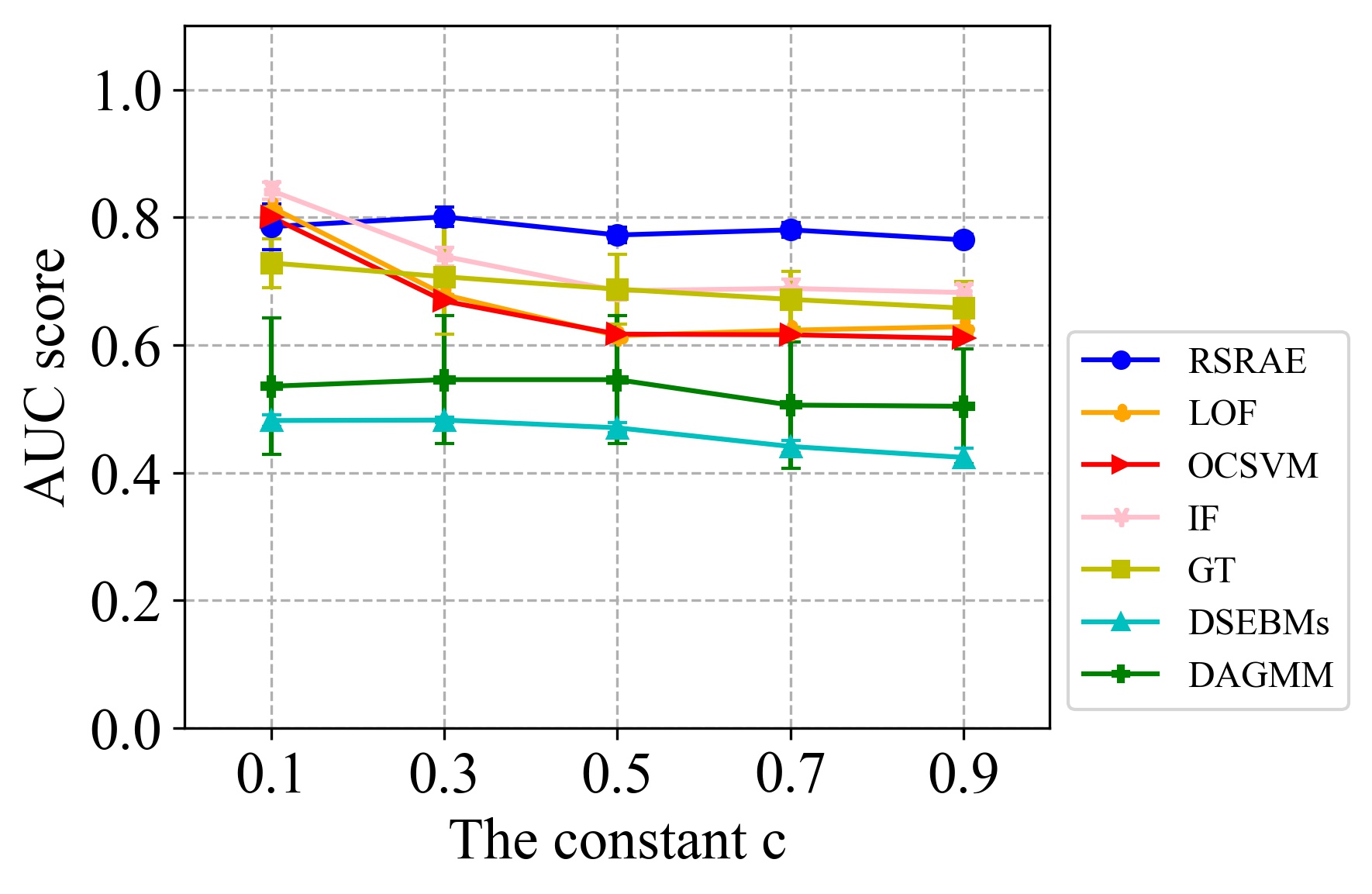}
\end{minipage}
\begin{minipage}[t]{0.48\textwidth}
\centering
\includegraphics[width=6cm]{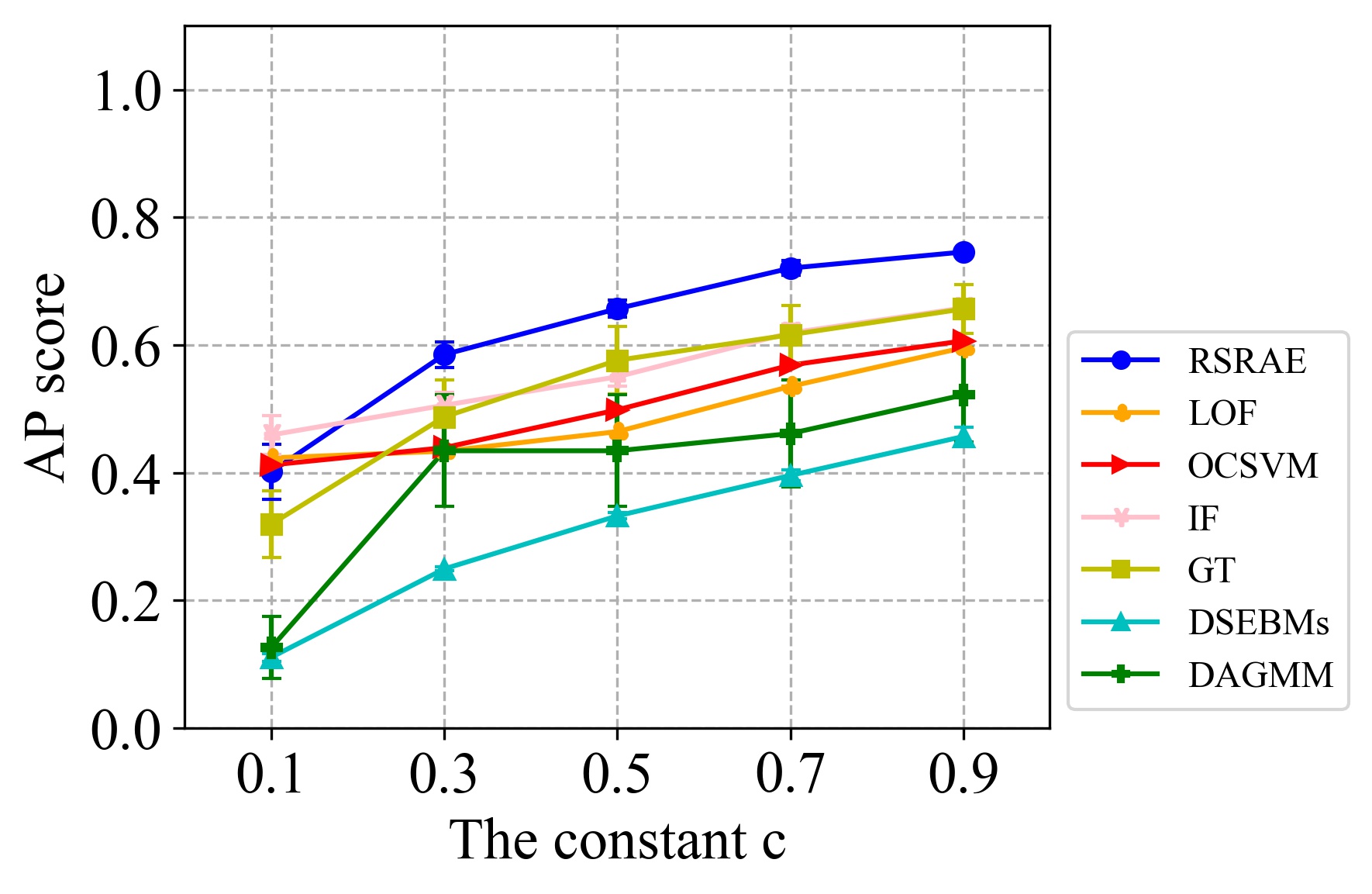}
\end{minipage}
 
\centering
\begin{minipage}[t]{0.48\textwidth}
\rotatebox{90}{\null \qquad Fashion MNIST}
\centering
\includegraphics[width=6cm]{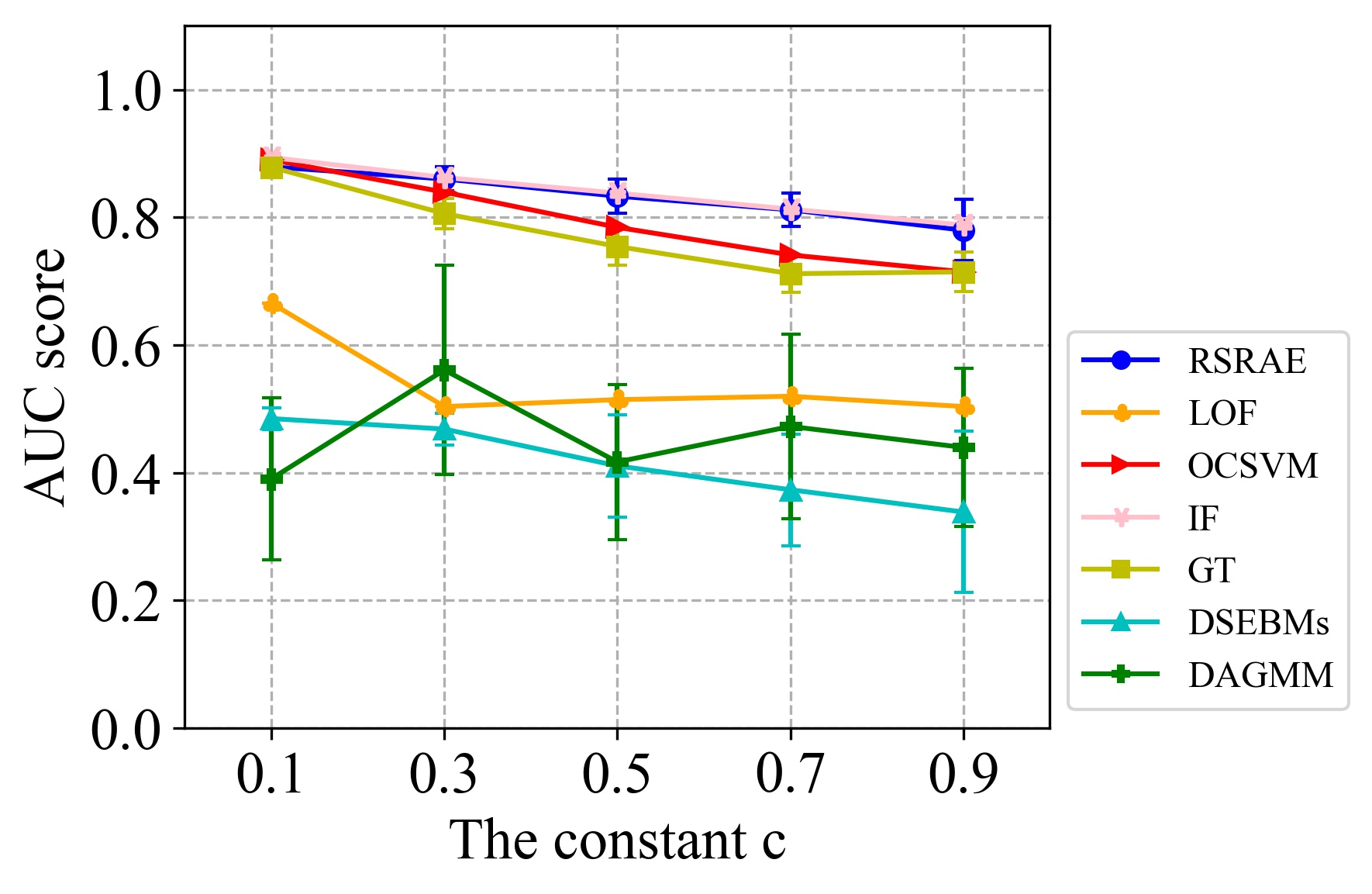}
 
\end{minipage}
\begin{minipage}[t]{0.48\textwidth}
\centering
\includegraphics[width=6cm]{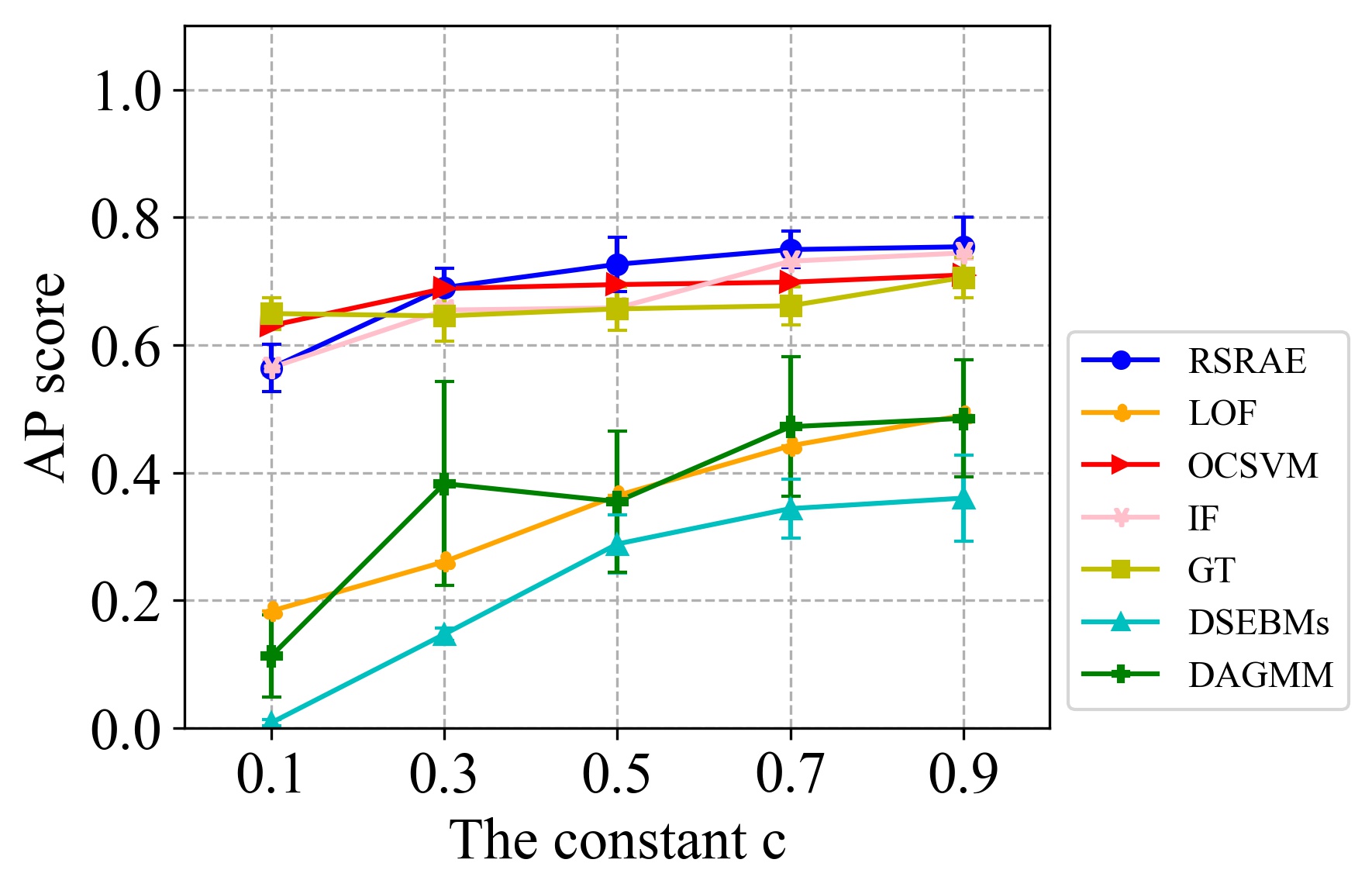}
 
\end{minipage}

\caption{AUC and AP scores for RSRAE using Caltech 101 and Fashion MNIST.}
\label{fig:aucapall}

\end{figure}

\begin{figure}[htb]%\ContinuedFloat

\centering
\begin{minipage}[t]{0.48\textwidth}
\rotatebox{90}{\null \qquad Tiny Imagenet}
\centering
\includegraphics[width=6cm]{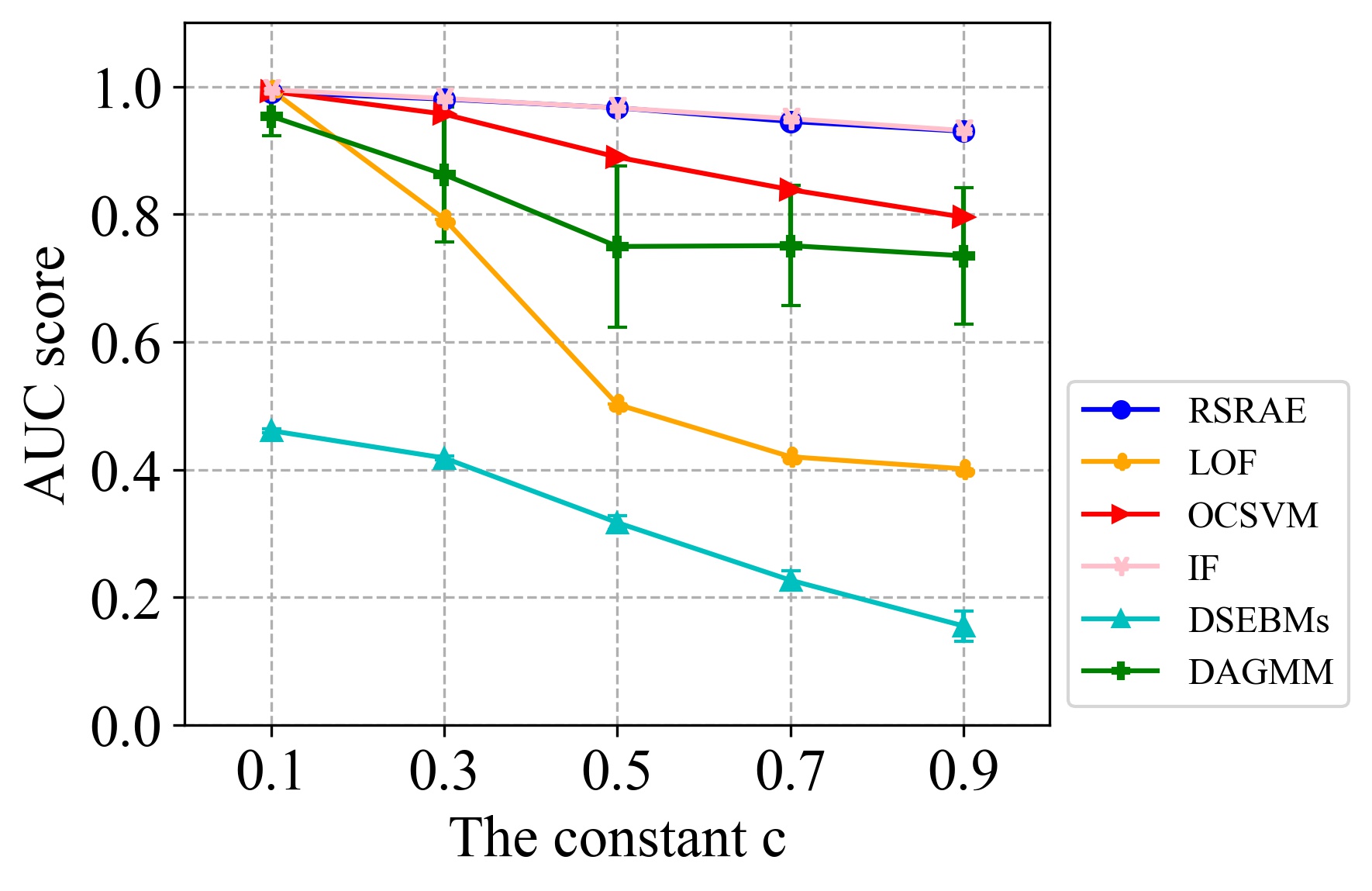}
\end{minipage}
\begin{minipage}[t]{0.48\textwidth}
\centering
\includegraphics[width=6cm]{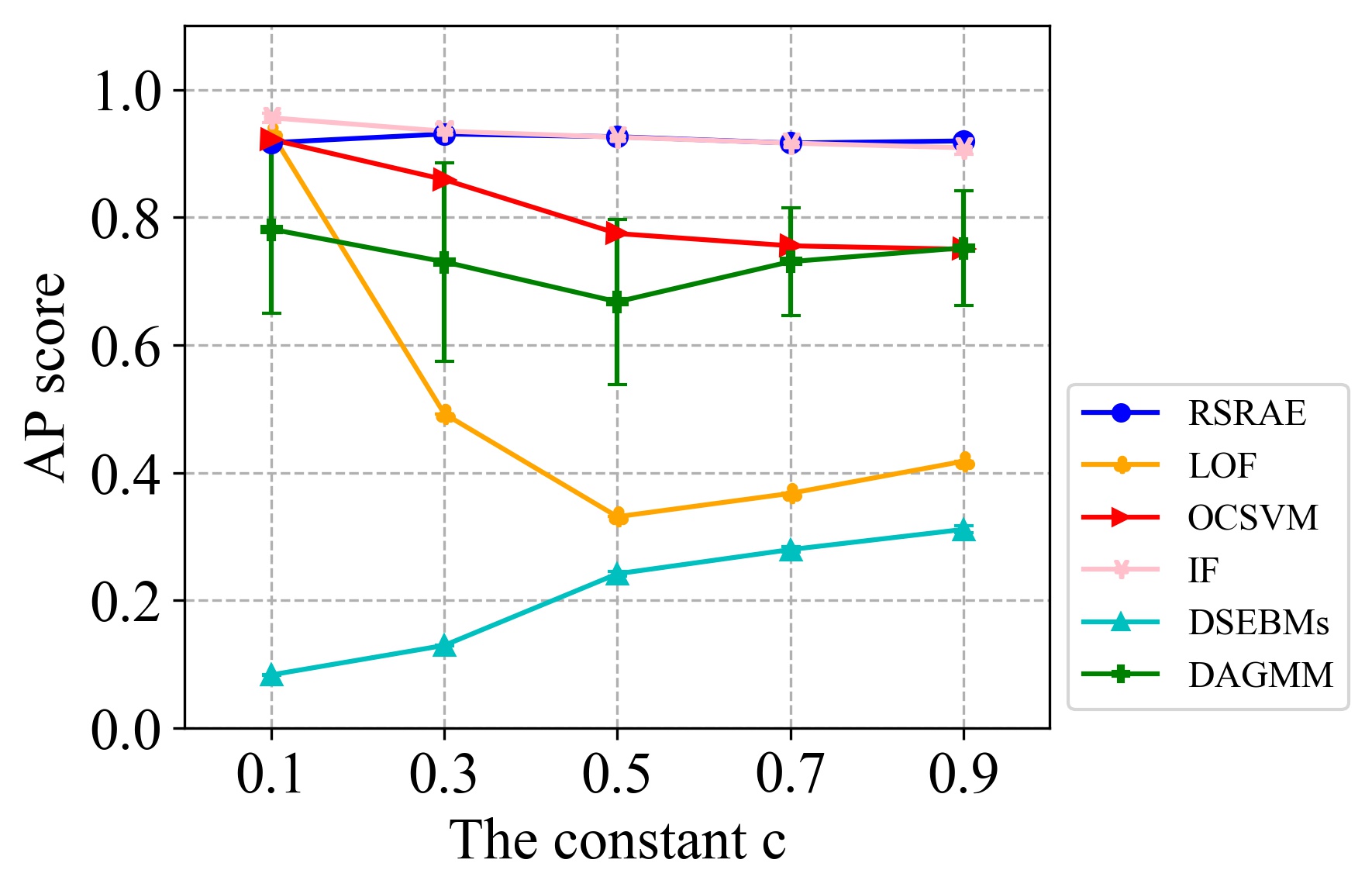}
\end{minipage}

\centering
\begin{minipage}[t]{0.48\textwidth}
\rotatebox{90}{\null \qquad Reuters-21578}
\centering
\includegraphics[width=6cm]{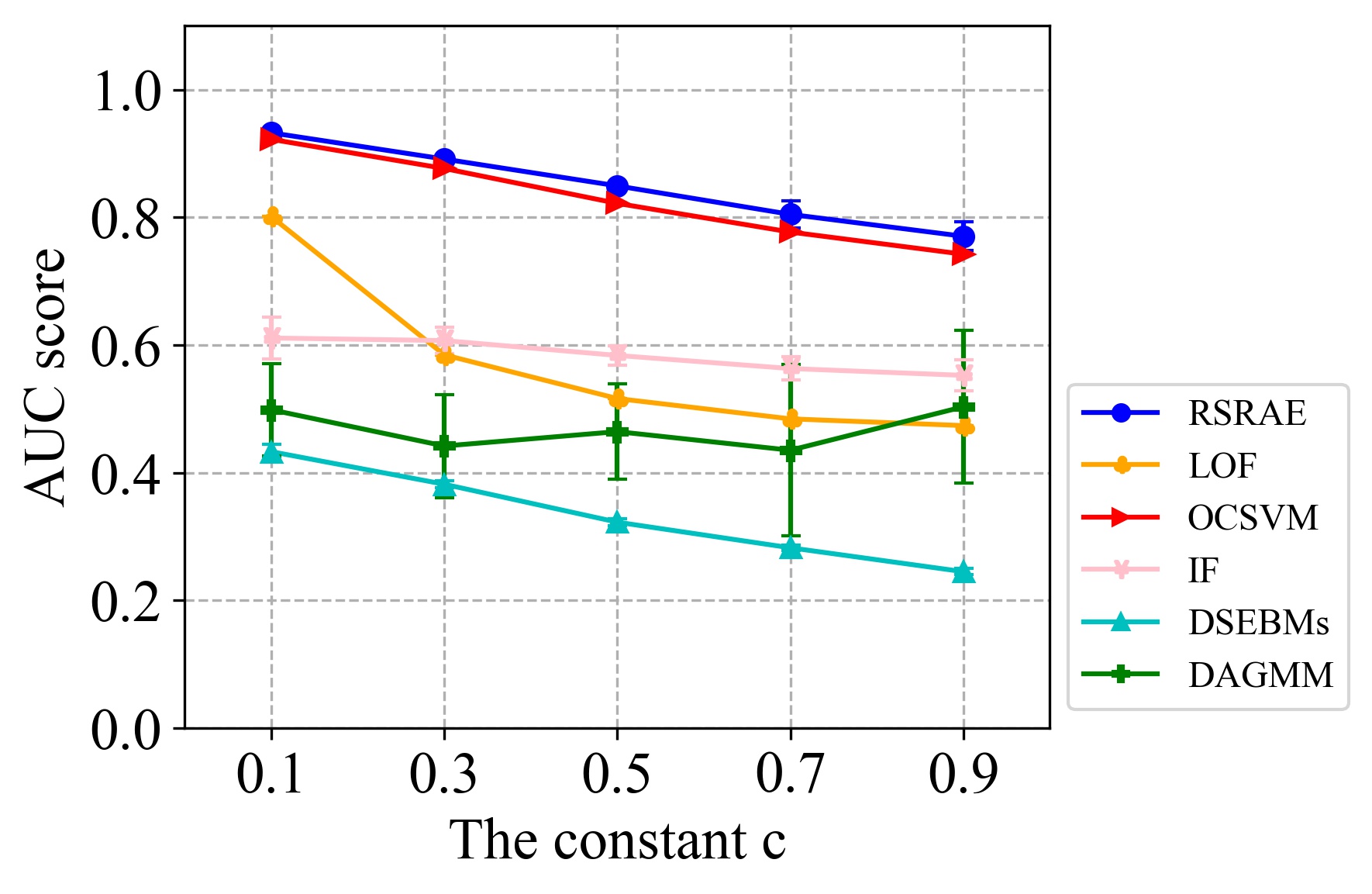}
\end{minipage}
\begin{minipage}[t]{0.48\textwidth}
\centering
\includegraphics[width=6cm]{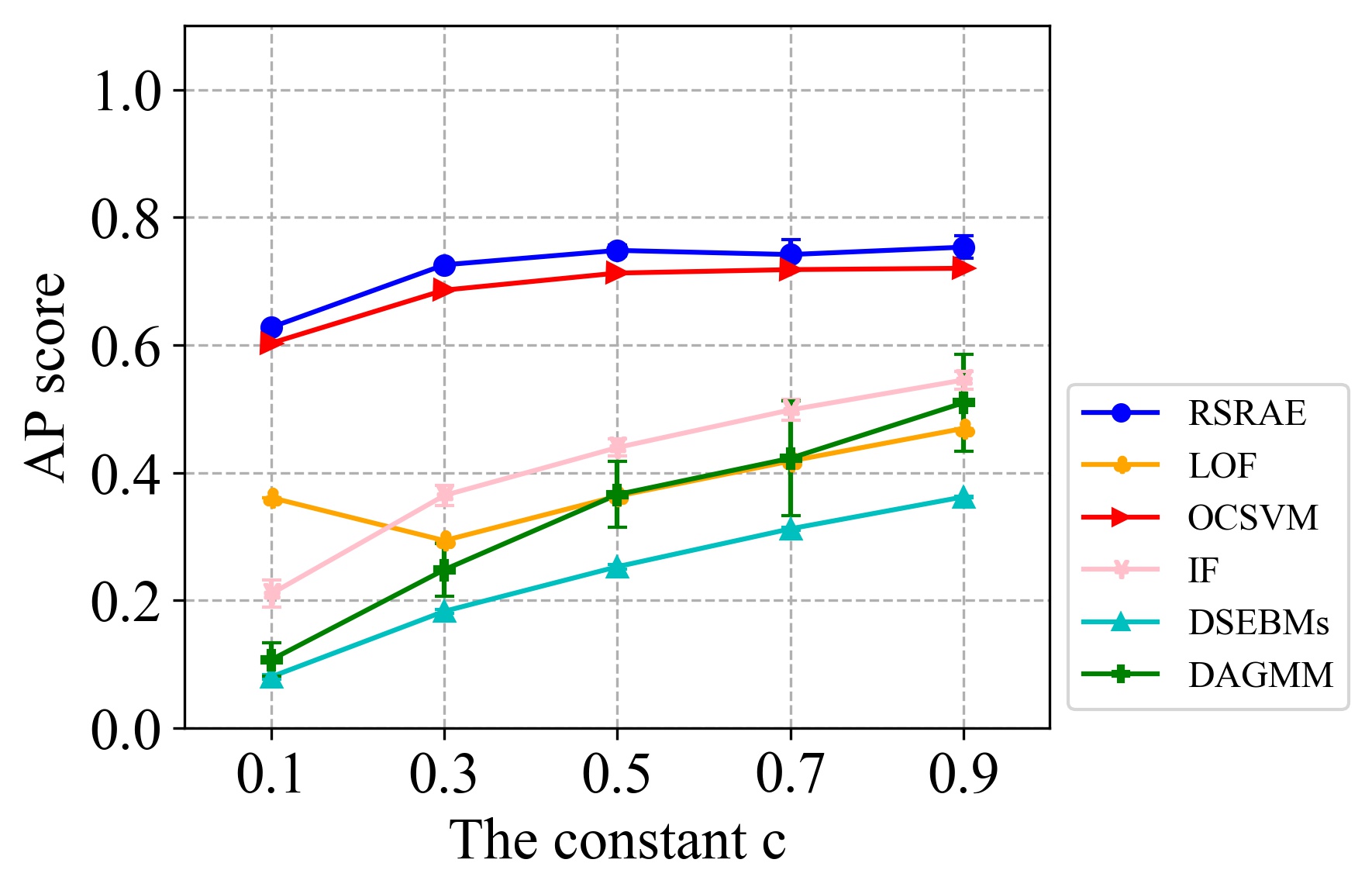}
\end{minipage}
 
\centering
\begin{minipage}[t]{0.48\textwidth}
\rotatebox{90}{\null \qquad 20 Newsgroups}
\centering
\includegraphics[width=6cm]{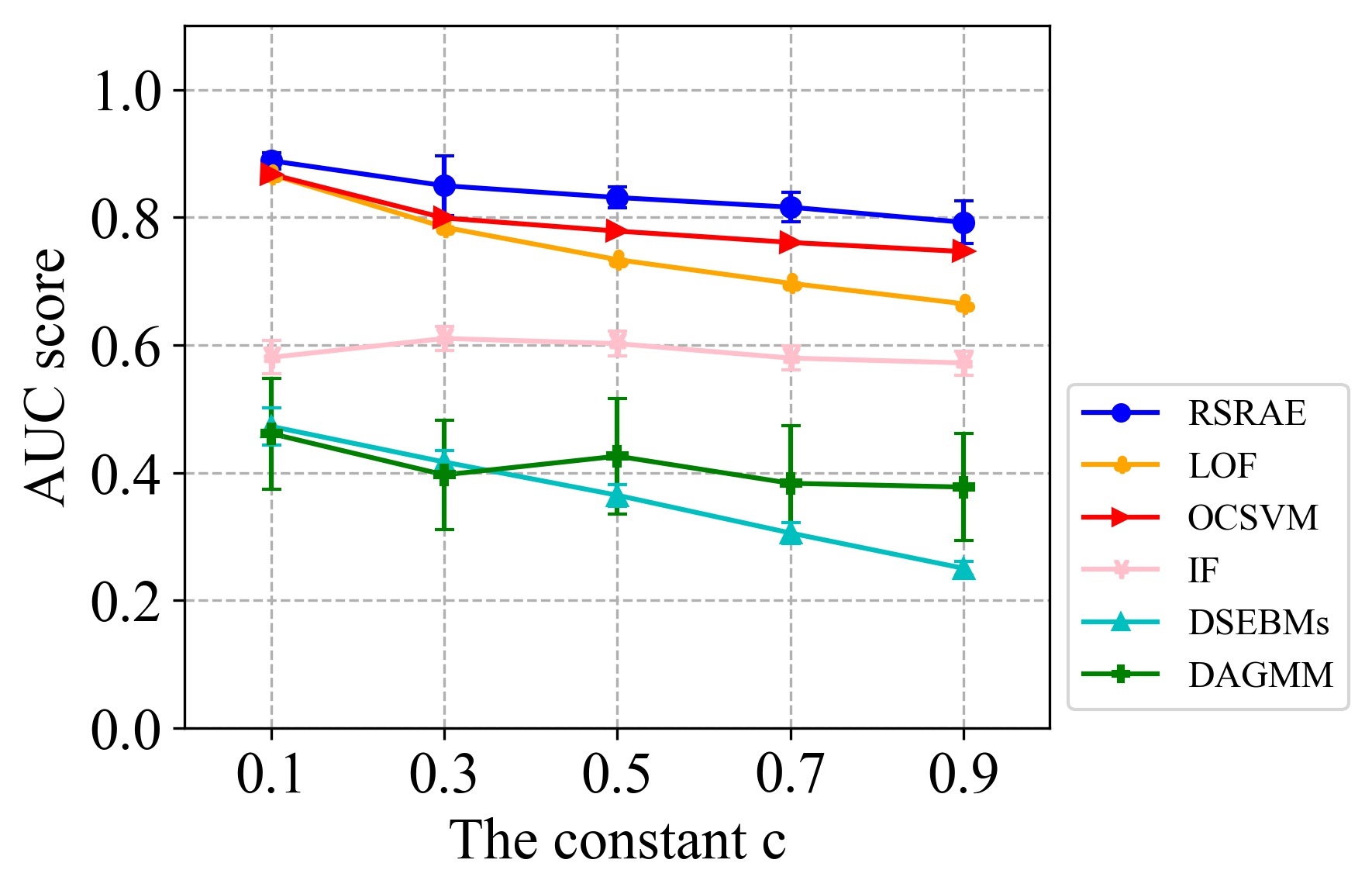}
\end{minipage}
\begin{minipage}[t]{0.48\textwidth}
\centering
\includegraphics[width=6cm]{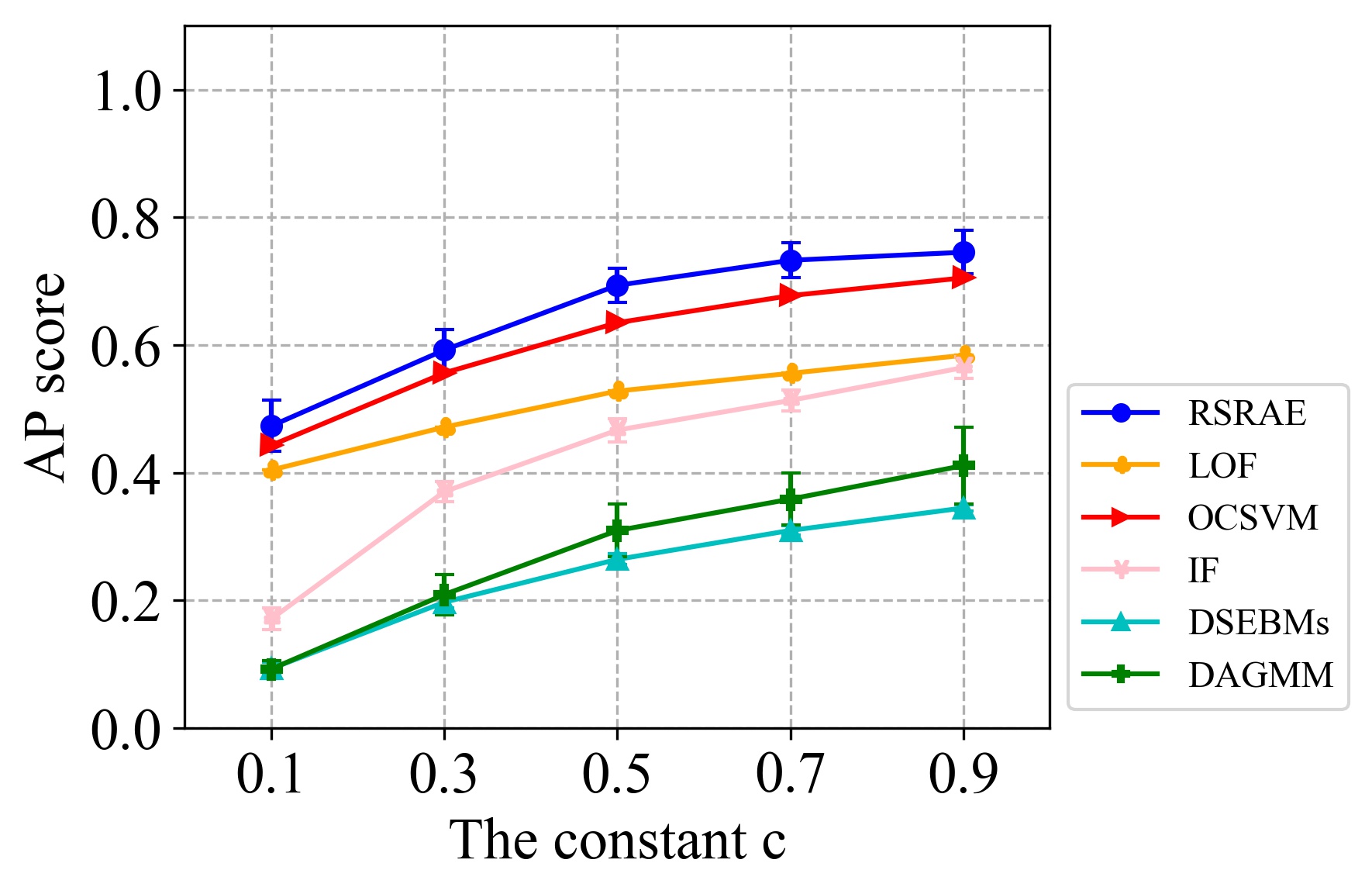}
\end{minipage}

\caption{AUC and AP scores for RSRAE using Tiny Imagenet with deep features, Reuters-21578 and 20 Newsgroups.}
\label{fig:aucapall2}
\end{figure}

\subsection{Comparison with Variations of RSRAE}
\label{subsec:cprnorm}

We use one image dataset (Caltech 101) and one document dataset (Reuters-21578) and compare between RSRAE and three variations of it. The first one is RSRAE+ (see \Secref{sec:alg_describe}) with $\lambda_1 = \lambda_2 = 0.1$ in \eqref{eq:lossRSR} (these parameters were optimized on 20 Newsgroup, though results with other choices of parameters are later demonstrated in \Secref{subsec:hyperparameters}). 
The next two are simpler autoencoders without RSR layers: {AE-1}  minimizes $L_{\rm{AE}}^1$, the $\ell_{2,1}$ reconstruction loss;
and {AE} minimizes $L_{\rm{AE}}^2$, the $\ell_{2,2}$ reconstruction loss (it is a regular autoencoder for anomaly detection). We maintain the same architecture as that of RSRAE, including the matrix $\rmA$, but use different loss functions. 

Fig.~\ref{fig:cprall} reports the AUC and AP scores. We see that for the two datasets RSRAE+ with the prespecified $\lambda_1$ and $\lambda_2$ does not perform as well as RSRAE, but its performance is still better than AE and AE-1. This is expected since we chose $\lambda_1$ and $\lambda_2$ after few trials with a different dataset, whereas RSRAE is independent of these parameters.
The performance of AE and AE-1 is clearly worse, and they are also not as good as some methods compared with in \Secref{subsec:res}.
At last, AE is generally comparable with AE-1. Similar results are noticed for the other datasets in Appendix~\ref{subsec:cprnormnotshown}.

\begin{figure}[ht]
\centering
\begin{minipage}[t]{0.48\textwidth}
\rotatebox{90}{\null \qquad Caltech 101}
\centering
\includegraphics[width=6cm]{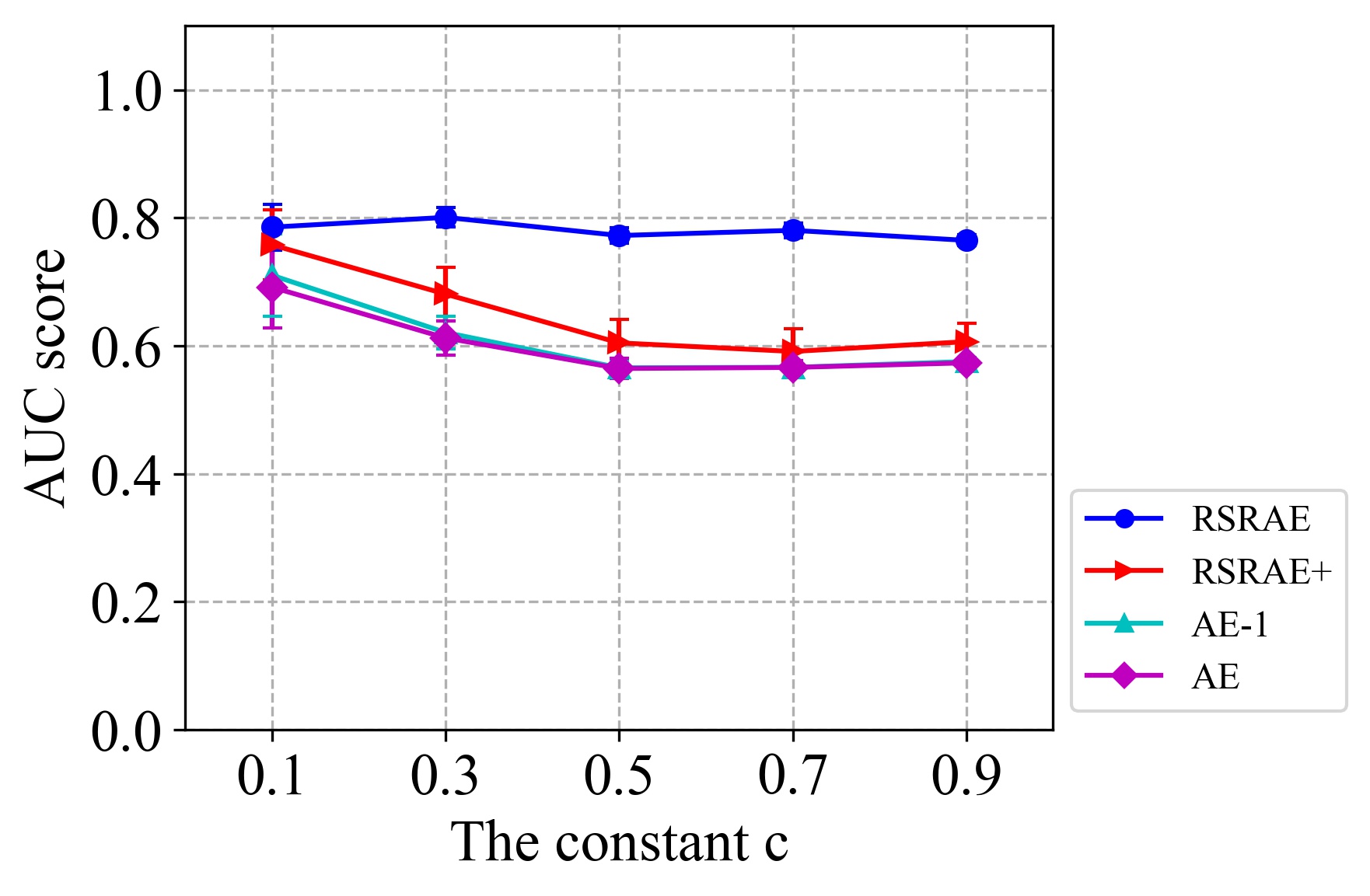}
\end{minipage}
\begin{minipage}[t]{0.48\textwidth}
\centering
\includegraphics[width=6cm]{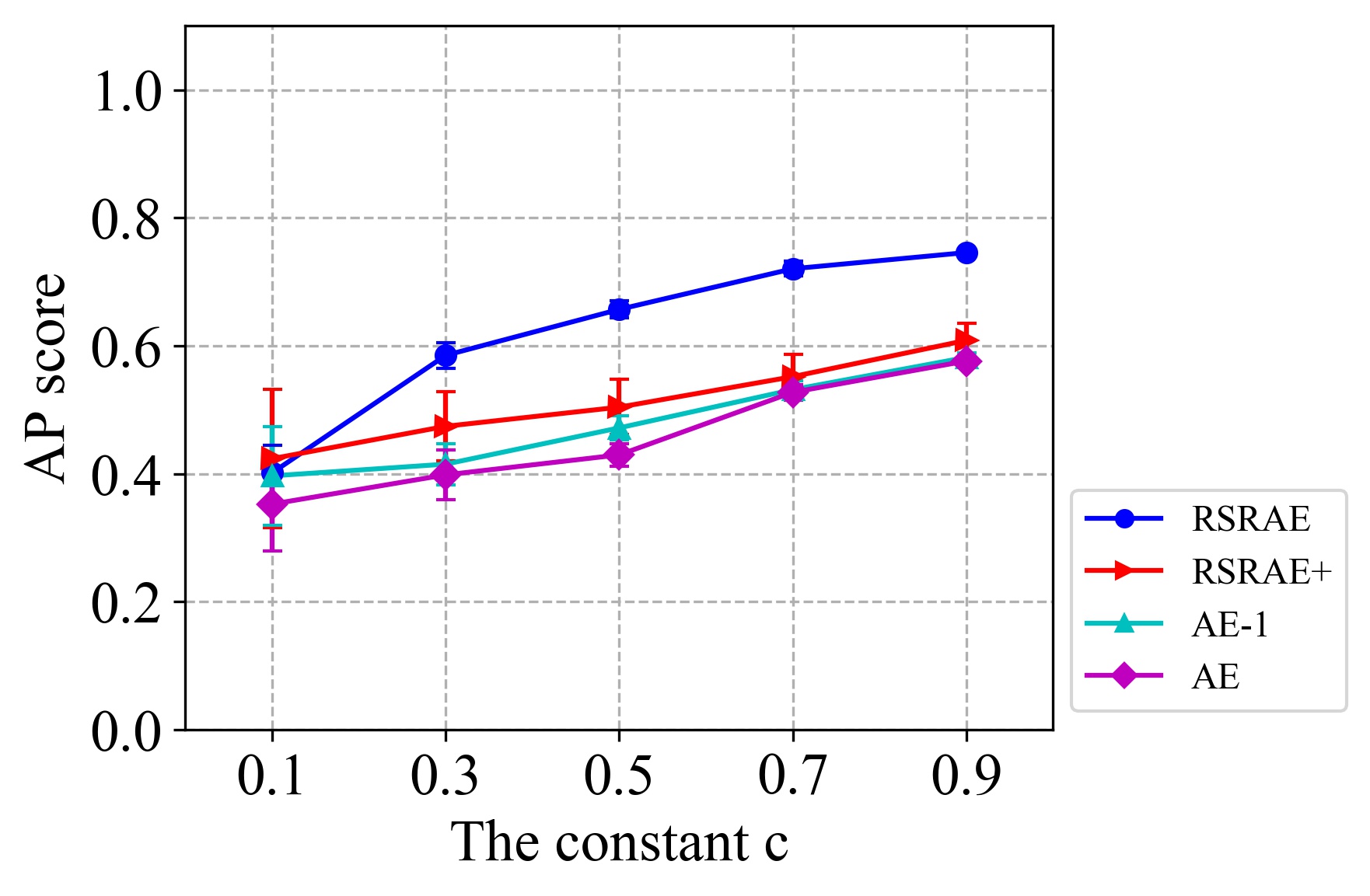}
\end{minipage}

\centering
\begin{minipage}[t]{0.48\textwidth}
\rotatebox{90}{\null \qquad Reuters-21578}
\centering
\includegraphics[width=6cm]{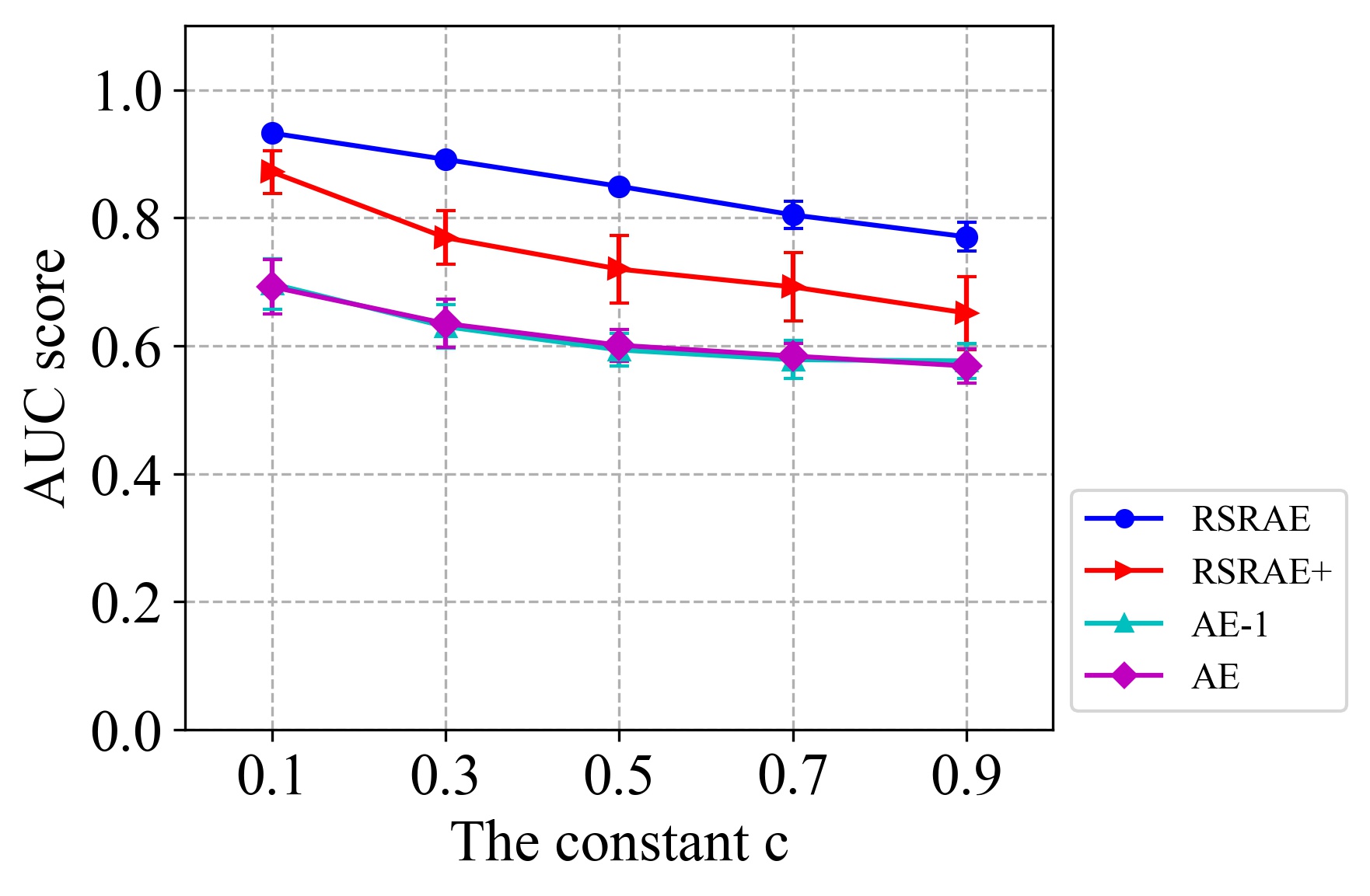}
\end{minipage}
\begin{minipage}[t]{0.48\textwidth}
\centering
\includegraphics[width=6cm]{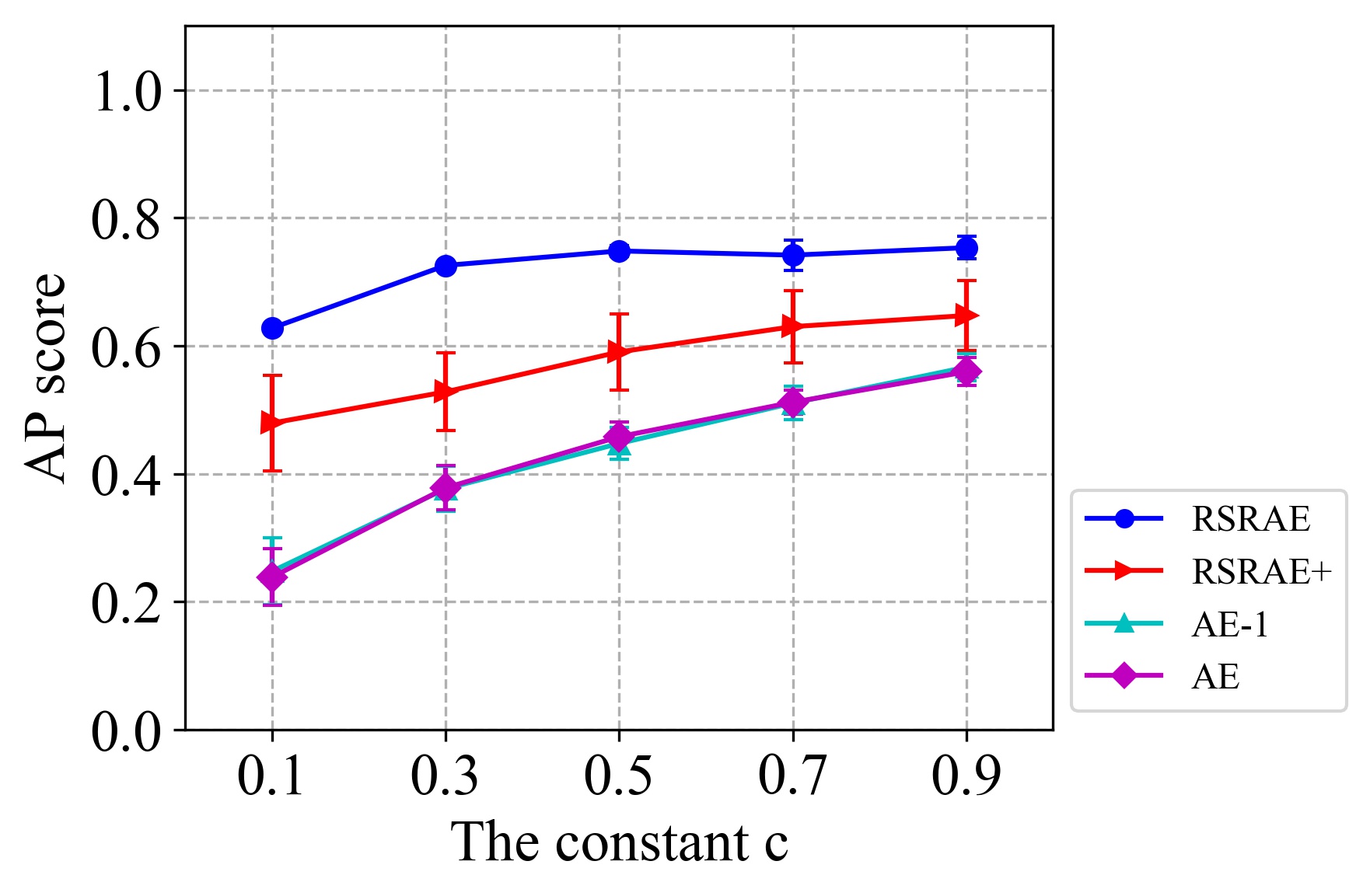}
\end{minipage}
% \caption{Norms comparison for Reuters-21578}
% \label{fig:cprreuters}
\caption{AUC and AP scores for RSRAE and alternative formulations using Caltech 101 and Reuters-21578.}
\label{fig:cprall}
\end{figure}

% \subsection{Discussion}
% \label{subsec:discussion}

% \newpage
\section{Related theory for the RSR penalty}
\label{sec:relatedtheory}

We explain here why we find it natural to incorporate RSR within a neural network. In Section \ref{sec:heuristic} we first review the mathematical idea of an autoencoder and discuss the robustness of a linear autoencoder with an $\ell_{2,1}$ loss (i.e., RSR loss). We then explain why a general autoencoder with an  $\ell_{2,1}$ loss is not expected to be robust to outliers and why an RSR layer can improve its robustness. Section \ref{sec:generative} is a first step of extending this view to a generative network. It establishes some robustness of WGAN with a linear generator, but the extension of an RSR layer to WGAN is left as an open problem. 
\subsection{Robustness and related properties of autoencoders} 
\label{sec:heuristic}
Mathematically, an autoencoder for a dataset $\{\rvx^{(t)}\}_{t=1}^N \subset \R^D$ and a latent dimension $d <D$ is composed of an encoder $\mathscr{E}: \R^D \rightarrow \R^d$ and a decoder $\mathscr{D}: \R^d \rightarrow \R^D$ that minimize the following energy function with $p=2$:
\begin{equation}\label{eq:lossAE}
    %L_{\rm{PCA}} (\mathscr{E}, \mathscr{D}) = 
    \sum_{t=1}^N \norm{\rvx^{(t)} - \mathscr{D} \circ \mathscr{E} ({\rvx}^{(t)})}_2^p ~,
\end{equation}
where $\circ$ denotes function decomposition.
It is a natural nonlinear generalization of PCA \citep{goodfellow2016deep}. Indeed, in the case of a linear autoencoder, $\mathscr{E}$ and $\mathscr{D}$ are linear maps represented by  matrices  $\rmE \in \R^{d \times D}$ and $\rmD \in \R^{D \times d}$, respectively, that need to minimize (among such matrices) the following loss function with $p=2$ 
\begin{equation}\label{eq:lossAEPCA}
    %L_{\rm{PCA}} (\rmE, \rmD) = 
    \sum_{t=1}^N \norm{\rvx^{(t)} - \rmD \rmE {\rvx}^{(t)}}_2^p ~.
\end{equation}

We explain in Appendix \ref{subsec:linearaeforrsr} that if  $(\rmD^{\star},\rmE^{\star})$ is a minimizer of \eqref{eq:lossAEPCA} with $p=2$ (among  $\rmE \in \R^{d \times D}$ and $\rmD \in \R^{D \times d}$), then $\rmD^{\star} \rmE^{\star}$ is the orthoprojector on the $d$-dimensional PCA subspace. This means, that the latent code
$\{\rmE^{\star} \rvx^{(t)}\}_{t=1}^N$ parametrizes the PCA subspace and an additional application of
$\rmD^{\star}$ to $\{\rmE^{\star} \rvx^{(t)}\}_{t=1}^N$ results in the projections of the data points $\{\rvx^{(t)}\}_{t=1}^N$ onto the PCA subspace. The recovery error for data points on this subspace is zero (as $\rmD^{\star} \rmE^{\star}$ is the identity on this subspace), and in general, this error is the Euclidean distance to the PCA subspace, $\norm{\rvx^{(t)} - \rmD^{\star} \rmE^{\star} {\rvx}^{(t)}}_2$.

Intuitively, the idea of a general autoencoder is the same. It aims to fit a nice structure, such as a  manifold,  to the data, where ideally $\mathscr{D} \circ \mathscr{E}$ is a projection onto this nice structure. This idea can only be made rigorous for data approximated by simple geometric structure, e.g., by a graph of a sufficiently smooth function.

In order to extend these methods to anomaly detection, one needs to incorporate robust strategies, so that the methods can still recover the underlying structure of the inliers, and consequently assign lower recovery errors for the inliers and higher recovery errors for the outliers.
For example, in the linear case, one may assume a set of inliers lying on and around a subspace and an arbitrary set of outliers (with some restriction on their fraction). PCA, and equivalently, the linear autoencoder that minimizes \eqref{eq:lossAEPCA} with $p=2$, is not robust to general outliers. Thus it is not expected to distinguish well between inliers and outliers in this setting. As explained in Appendix \ref{subsec:linearaeforrsr}, minimizing \eqref{eq:lossAEPCA} with $p=1$ gives rise to the least absolute deviations subspace. This subspace can be robust to outliers under some conditions, but these conditions are restrictive (see examples in \citet{lp_recovery_part1_11}). In order to deal with more adversarial outliers, it is advised to first normalize the data to the sphere (after appropriate centering) and then estimate the least absolute deviations subspace. This procedure was theoretically justified for a general setting of adversarial outliers in \citet{maunu2019robust}.

As in the linear case, an autoencoder that uses the loss function in \eqref{eq:lossAE} with $p=1$ may not be robust to adversarial outliers. Unlike the linear case, there are no simple normalizations for this case. Indeed, the normalization to the sphere can completely distort the structure of an underlying manifold and it is also hard to center in this case. Furthermore, there are some obstacles of establishing robustness for the nonlinear case even under special assumptions. 

Our basic idea for a robust autoencoder is to search for a latent low-dimensional code for the inliers within a larger embedding space. The additional RSR loss focuses on parametrizing the low-dimensional subspace of the encoded inliers, while being robust to outliers. 
Following the above discussion, we enhance such robustness by applying a normalization similar to the one discussed above, but adapted better to the structure of the network (see \Secref{subsec:benchmark}). 
The emphasis of the RSR layer is on appropriately encoding the inliers, where the encoding of the outliers does not matter.  It is okay for the encoded outliers to lie within the subspace of the encoded inliers, as this will result in large recovery errors for the outliers. However, in general, most encoded outliers lie away from this subspace, and this is why such a mechanism is needed (otherwise, a regular autoencoder may obtain a good embedding).

\subsection{Relationship of the RSR loss with linearly generated WGAN}
\label{sec:generative}

An open problem is whether RSR can be used within other neural network structures for unsupervised learning, such as variational autoencoders (VAEs) \citep{kingma2013auto} and generative adversarial networks (GANs) \citep{goodfellow2014generative}. The latter two models are used in anomaly detection with a score function similar to the reconstruction error \citep{an2015variational,vasilev2018q,zenati2018efficient,kliger2018novelty}.

While we do not solve this problem, we establish a natural relationship between RSR and Wasserstein-GAN (WGAN) \citep{arjovsky2017wasserstein, gulrajani2017improved} with a linear generator, which is analogous to the example of a linear autoencoder mentioned above. 

Let $W_p$ denote the $p$-Wasserstein distance in $\R^D$ ($p \geq 1$). That is, for two probability distributions $\mu, \nu$ on $\R^D$,
\begin{equation}\label{eq:defwasserstein}
W_p(\mu, \nu) = \left( \inf_{\pi \in \Pi(\mu, \nu)} \E_{(\rvx,\rvy) \sim \pi} \norm{\rvx-\rvy}_2^p \right)^{1/p} ~,
\end{equation}
where $\Pi(\mu, \nu)$ is the set of joint distributions with $\mu$, $\nu$ as marginals. We formulate the following proposition (while prove it later in Appendix~\ref{sec:proofsubspace}) and then interpret it.

\begin{proposition}\label{prop:subspace}
Let $p\geq 1$ and 
$\rvmu$ be a Gaussian distribution on $\R^D$ with mean $\rvm_X \in \R^D$ and full-rank covariance matrix $\rmSigma_X \in \R^{D \times D}$ (that is, $\rvmu$ is $\mathcal{N} (\rvm_X, \rmSigma_X)$). 
Then \begin{equation}\label{eq:optprob}
\begin{split}
\min_{\rvnu ~\rm{is}~ \mathcal{N}(\rvm_Y, \rmSigma_Y)} \quad & W_p(\rvmu, \rvnu) \\ 
\quad  
\rm{s.t.} \quad \quad \quad &  \rvm_Y \in \R^D \\
\quad & 
\rm{rank} (\rmSigma_Y) = d 
\end{split}
\end{equation}
is achieved when $\rvm_Y=\rvm_X$ and 
$\rmSigma_Y = \rmP_\mathscr{L} \rmSigma_X \rmP_\mathscr{L}$, 
where for $X  \sim \rvmu$
\begin{equation}
\label{eq:rsrganconclusion}
\mathscr{L} = \argminB_{\rm{dim} \mathscr{L} = d} \E \norm{X - \rmP_\mathscr{L} X}_2^p ~.
\end{equation}
\end{proposition}

The setting of this proposition implicitly assumes a linear generator of WGAN. Indeed, the linear mapping, which can be represented by a $d \times D$ matrix, maps a distribution in $\mathcal{N} (\rvm_X, \rmSigma_X)$ into a distribution in 
$\mathcal{N}(\rvm_Y, \rmSigma_Y)$ and reduces the rank of the covariance matrix from $D$ to $d$. The proposition states that in this setting the underlying minimization is closely related to minimizing the loss function \eqref{eq:lossrsrorig}. Note that here $p \geq 1$, however, if one further corrupts the sample, then $p=1$ is the suitable choice \citep{lerman2018overview}. This choice is also more appropriate for WGAN, since there is no $p$-WGAN for $p \neq 1$. 

Nevertheless, training a WGAN is not exactly the same as minimizing the $W_1$ distance \citep{gulrajani2017improved}, since it is difficult to impose the Lipschitz constraint for a neural network. 
Furthermore, in practice, the WGAN generator, which is a neural network, is nonlinear, and thus its output is typically non-Gaussian.
The robustness of WGAN with a linear autoencoder, which we established here, does not extend to a general WGAN (this is similar to our earlier observation that the robustness of a linear autoencoder with an RSR loss does not generalize to a nonlinear autoencoder). We believe that a similar structure like the RSR layer has to be imposed for enhancing the robustness of WGAN, and possibly also other generative networks, but we leave its effective implementation as an open problem.

\section{Conclusion and future work} 
\label{sec:conclude}
We constructed a simple but effective RSR layer within the autoencoder structure for anomaly detection. It is easy to use and adapt. 
We have demonstrated competitive results for image and document data and believe that it can be useful in many other applications.

There are several directions for further exploration of the RSR loss in unsupervised deep learning models for anomaly detection. First, we are interested in theoretical guarantees for RSRAE. A more direct subproblem is understanding the geometric structure of the ``manifold'' learned by RSRAE. 
Second, it is possible that there are better geometric methods to robustly embed the manifold of inliers. For example, one may consider a multiscale incorporation of RSR layers, which we expand on in Appendix~\ref{subsec:betas}.
Third, one may try to incorporate an RSR layer in other neural networks for anomaly detection that use nonlinear dimension reduction. We hope that some of these methods may be easier to directly analyze than our proposed method. For example, we are curious about successful incorporation of robust metrics for GANs or WGANs. In particular, we wonder about extensions of the theory proposed here for WGAN when considering a more general setting.

\subsubsection*{Acknowledgments}
\label{subsubsec: acknowledgement}
This research has been supported by NSF award DMS18-30418. Part of this work was pursued when Dongmian Zou was a postdoctoral associate at the Institute for Mathematics and its Applications at the University of Minnesota. We thank Teng Zhang for his help with proving
Proposition~\ref{prop:subspace} (we discussed a related but different proposition with similar ideas of proofs). We thank Madeline
Handschy for commenting on an earlier version of this paper. 

%%%%%%%%%%%%%%%%%%%%%%%%%%%%%%%%%%%%%%%%%%%%%%%%%%%%%%%%%%%%

\bibliography{iclr2020_conference}
\bibliographystyle{iclr2020_conference}

%%%%%%%%%%%%%%%%%%%%%%%%%%%%%%%%%%%%%%%%%%%%%%%%%%%%%%%%%%%%

\newpage

\appendix

\section{Details of RSRAE and RSRAE+}
\label{sec:alg}

The implementations of both RSRAE and RSRAE+ are simple. For completeness we provide here their details in algorithm boxes. The codes will be later posted in a supplementary webpage.  Algorithm~\ref{alg:the_alg} describes RSRAE, which minimizes \eqref{eq:combined} by alternating minimization. It denotes the vectors of parameters of the encoder and decoder by $\vtheta$ and $\vvarphi$, respectively.

 \begin{algorithm}[h] 
 \caption{RSRAE 
 }
 \label{alg:the_alg}
 \begin{algorithmic}[1]
 \renewcommand{\algorithmicrequire}{\textbf{Input:}}
 \renewcommand{\algorithmicensure}{\textbf{Output:}}
 \REQUIRE Data $\{\rvx^{(t)}\}_{t=1}^N$; thresholds $\rvepsilon_{\rm{AE}}$, $\rvepsilon_{\rm{RSR_{1}}}$, $\rvepsilon_{\rm{RSR_{2}}}$, $\rvepsilon_{\rm{T}}$;  architecture and initial parameters of $\mathscr{E}$, $\mathscr{D}$, $\rmA$ (including number of columns of $\rmA$); number of epochs $\&$ batches; learning rate for backpropagation; similarity measure
 \ENSURE  Labels of data points as normal or anomalous \\
  \FOR {each epoch}
  \STATE { Divide input data into batches }
  \FOR {each batch}
  \IF {$L_{\rm{AE}}^1(\vtheta, \rmA, \vvarphi)>\rvepsilon_{\rm{AE}}$}
  \STATE Backpropagate $L_{\rm{AE}}^1(\vtheta, \rmA, \vvarphi)$ w.r.t.~$\vtheta, \rmA, \vvarphi$ $\&$ update $\vtheta, \rmA, \vvarphi$
  \ENDIF
  \IF{$L_{\rm{RSR_{1}}}^1(\rmA)>\rvepsilon_{\rm{RSR_{1}}}$}
  \STATE Backpropagate $L_{\rm{RSR_{1}}}^1(\rmA)$ w.r.t.~$\rmA$
  $\&$ update $\rmA$
  \ENDIF
  \IF{$L_{\rm{RSR_{2}}}^1(\rmA)>\rvepsilon_{\rm{RSR_{2}}}$}
  \STATE Backpropagate $L_{\rm{RSR_{2}}}^1(\rmA)$ w.r.t.~$\rmA$ $\&$
  update $\rmA$
  \ENDIF
  \ENDFOR
  \ENDFOR
  \FOR{$t = 1, \ldots, N$}
  \STATE Calculate similarity between
  $\rvx^{(t)}$ and $\tilde{\rvx}^{(t)}$
  \IF {similarity $\geq$  $\rvepsilon_{\rm{T}}$}
  \STATE $\rvx^{(t)}$ is normal
  \ELSE
 \STATE $\rvx^{(t)}$ is anomalous
 \ENDIF
 \ENDFOR
 \RETURN Normality labels for $t = 1, \ldots, N$
 \end{algorithmic} 
 \end{algorithm}

We clarify some guidelines for choosing default parameters, which we follow in all reported experiments. We set $\eps_{\rm{AE}}$, $\eps_{\rm{RSR_1}}$ and $\eps_{\rm{RSR_2}}$ to be zero. In general, we use networks with dense layers but for image data we use convolutional layers. We prefer using $\tanh$ as the activation function due to its smoothness. However, for a dataset that does not lie in the unit cube, we use either a ReLU function if all of its coordinates are positive, or a leaky ReLU function otherwise. The network parameters and the elements of $\rmA$ are initialized to be i.i.d.~standard normal. In all numerical experiments, we set the number of columns of $\rmA$ to be 10, that is, $d=10$. The learning rate is chosen so that there is a sufficient improvement of the loss values after each epoch. Instead of fixing $\eps_{\rm{T}}$, we report the AUC and AP scores for different values of $\eps_{\rm{T}}$.

Algorithm \ref{alg:the_alg_plus} describes RSRAE+, which minimizes \eqref{eq:combined} with fixed $\lambda_1$ and $\lambda_2$ by auto-differentiation.

 \begin{algorithm}[ht] 
 \caption{RSRAE+
}
 \label{alg:the_alg_plus}
 \begin{algorithmic}[1]
 \renewcommand{\algorithmicrequire}{\textbf{Input:}}
 \renewcommand{\algorithmicensure}{\textbf{Output:}}
 \REQUIRE Data $\{\rvx^{(t)}\}_{t=1}^N$; thresholds $\rvepsilon_{\rm{AE}}$, $\rvepsilon_{\rm{T}}$;  architecture and initial parameters of $\mathscr{E}$, $\mathscr{D}$, $\rmA$ (including number of columns of $\rmA$); parameters of the the energy function $\lambda_1$, $\lambda_2$; number of epochs $\&$ batches; learning rate for backpropagation; similarity measure
 \ENSURE  Labels of data points as normal or anomalous \\

  \FOR {each epoch}
  \STATE { Divide input data into batches }
  \FOR {each batch}
  \IF {$L_{\rm{AE}}^1(\vtheta, \rmA, \vvarphi)>\rvepsilon_{\rm{AE}}$}
  \STATE Backpropagate $L_{\rm{AE}}^1(\vtheta, \rmA, \vvarphi) + \lambda_1 L_{\rm{RSR_{1}}}^1(\rmA) + \lambda_2 L_{\rm{RSR_{2}}}^1(\rmA)$ w.r.t.~$\vtheta, \rmA, \vvarphi$ $\&$ update $\vtheta, \rmA, \vvarphi$
  \ENDIF
  \ENDFOR
  \ENDFOR
  \FOR{$t = 1, \ldots, N$}
  \STATE Calculate similarity between
  $\rvx^{(t)}$ and $\tilde{\rvx}^{(t)}$
  \IF {similarity $\geq$  $\rvepsilon_{\rm{T}}$}
  \STATE $\rvx^{(t)}$ is normal
  \ELSE
 \STATE $\rvx^{(t)}$ is anomalous
 \ENDIF
 \ENDFOR
 \RETURN Normality labels for $t = 1, \ldots, N$
 \end{algorithmic} 
 \end{algorithm}

\newpage

\section{Demonstration of RSRAE for artificial data}
\label{sec:artificial}

For illustrating the performance of RSRAE, in comparison with a regular autoencoder, we consider a simple artificial geometric example.
We assume corrupted data whose normal part is embedded in a ``Swiss roll manifold''\footnote{\url{https://scikit-learn.org/stable/modules/generated/sklearn.datasets.make\_swiss\_roll.html}}, which is a two-dimensional manifold in $\R^3$. More precisely, the normal part is obtained by mapping 1,000 points uniformly sampled from the rectangle $[3\pi/2, 9\pi/2] \times [0, 21]$ into $\R^3$ by the function
\begin{equation}
    (s,t) \mapsto (t\cos(t), s, t\sin(t)).
\end{equation} 
The anomalous part is obtained by i.i.d.~sampling of 500 points from an isotropic Gaussian distribution in $\R^3$ with zero mean and standard deviation 2 in any direction. Fig.~\ref{fig:RSRAE_x} illustrates such a sample, where the inliers are in black and the outliers are in blue. We remark that Fig~\ref{fig:AE_x} is identical.

We construct the RSRAE with the following structure. The encoder is composed of fully-connected layers of sizes (32, 64, 128). The decoder is composed of fully connected layers of sizes (128, 64, 32, 3). Each fully connected layer is activated by the leaky ReLU function with $\alpha=0.2$. The intrinsic dimension for the RSR layer, that, is the number of columns of $\rmA$, is $d=2$.

For comparison, we construct the regular autoencoder AE (see \Secref{subsec:cprnorm}). Recall that both of them have the same architecture (including the linear map $\rmA$), but AE minimizes the $\ell_2$ loss function in \eqref{eq:lossAE} (with $p$ = 2) without an additional RSR loss. We optimize both models with 10,000 epochs and a batch gradient descent using Adam \citep{kingma2014adam} with a learning rate of 0.01. 

The reconstructed data ($\tilde{\rmX}$) using  
RSRAE and AE are plotted in Figs.~\ref{fig:RSRAE_xtilde} and \ref{fig:AE_xtilde}, respectively. 
We further demonstrate the output obtained by the encoder and the RSR layer. The output of the encoder, $\rmZ = \mathscr{E}(\rmX)$, lies in $\R^{128}$. For visualization purposes we project it onto a $\R^3$ as follows. We first find two vectors that span the image of $\rmA$  and we add to it the ``principal direction'' of $\rmZ$ orthogonal to the span of $\rmA$. We project $\rmZ$ onto the span of these 3 vectors.  Figs.~\ref{fig:RSRAE_y} and \ref{fig:AE_y} show these projections for RSRAE and AE, respectively. Figs.~\ref{fig:RSRAE_yrsr} and \ref{fig:AE_yrsr} demonstrate the respective mappings of $\rmZ$ by $\rmA$ during the RSR layer.  

Figs.~\ref{fig:RSRAE_xtilde} and \ref{fig:AE_xtilde} imply that the set of reconstructed normal points in RSRAE seem to lie on the original manifold, whereas the reconstructed normal points by AE seem to only lie near, but often not on the Swiss roll manifold. More importantly, the anomalous points reconstructed by RSRAE seem to be sufficiently far from the set of original anomalous points, unlike the reconstructed points by AE. Therefore, RSRAE can better distinguish anomalies using the distance between the original and reconstructed points, where small values are obtained for normal points and large ones for anomalous ones. 
Fig.~\ref{fig:histograms} demonstrates this claim. They plot the histograms of the distance 
between the original and reconstructed points when applying RSRAE and AE, where distances for normal and anomalous points are distinguished by color. Clearly, RSRAE distinguishes normal and anomalous data better than AE.

\begin{figure}[ht]
\centering
\begin{subfigure}{0.3\textwidth}
\centering
    \includegraphics[height=6em]{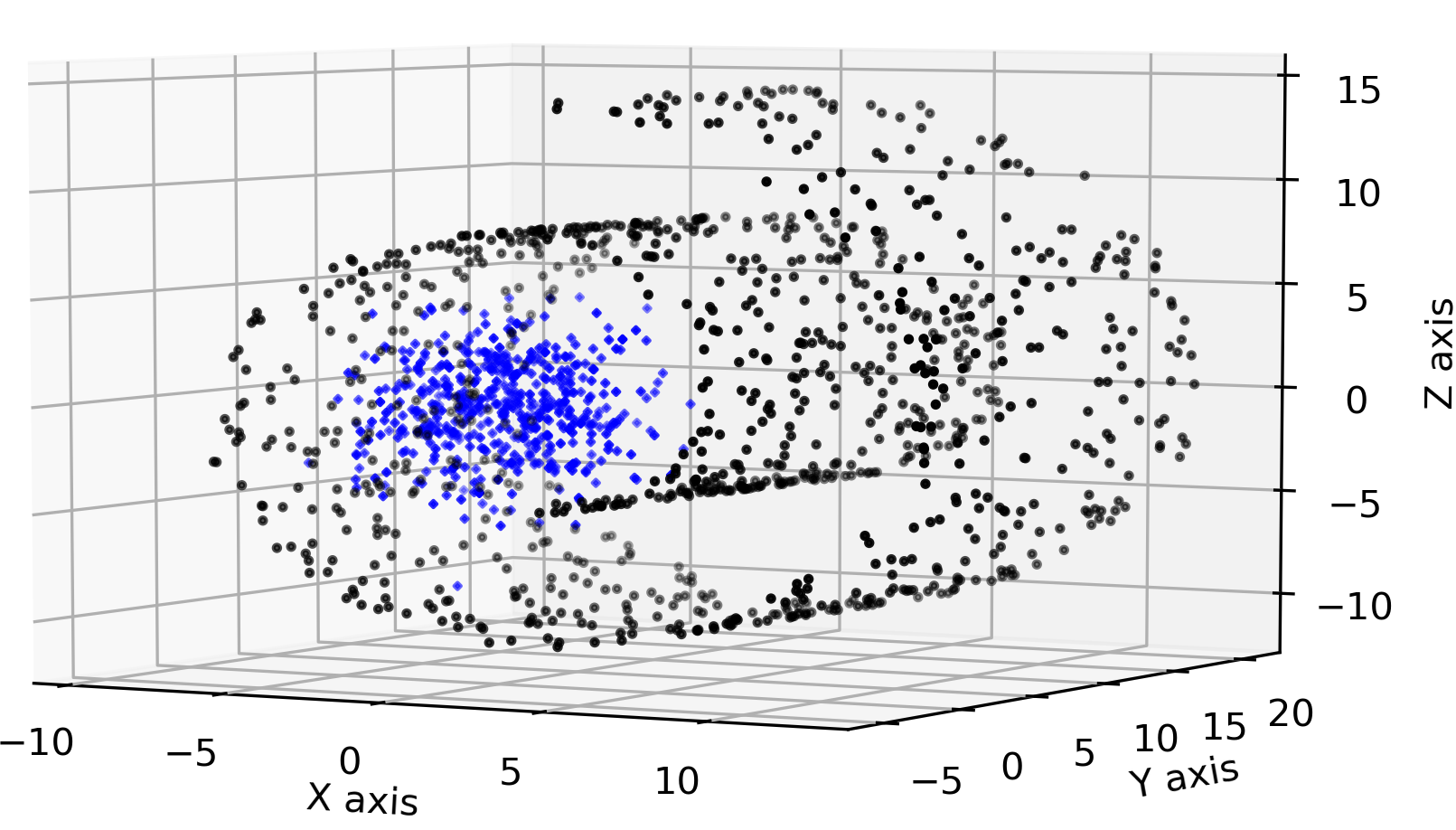}
    \caption{Input data $\rmX$}
    %%%%%%%%%%%%%%%%%%%%%%%%%
    \label{fig:RSRAE_x}
    %%%%%%%%%%%%%%%%%%%%%%%%%
\end{subfigure}%
{\large$\xrightarrow[\begin{subarray}{c} \mathscr{E}:\R^3 \rightarrow \R^{128}  \end{subarray}]{ \rm{Encoder}  } $}%
\begin{subfigure}{0.3\textwidth}
\centering
    \includegraphics[height=6em]{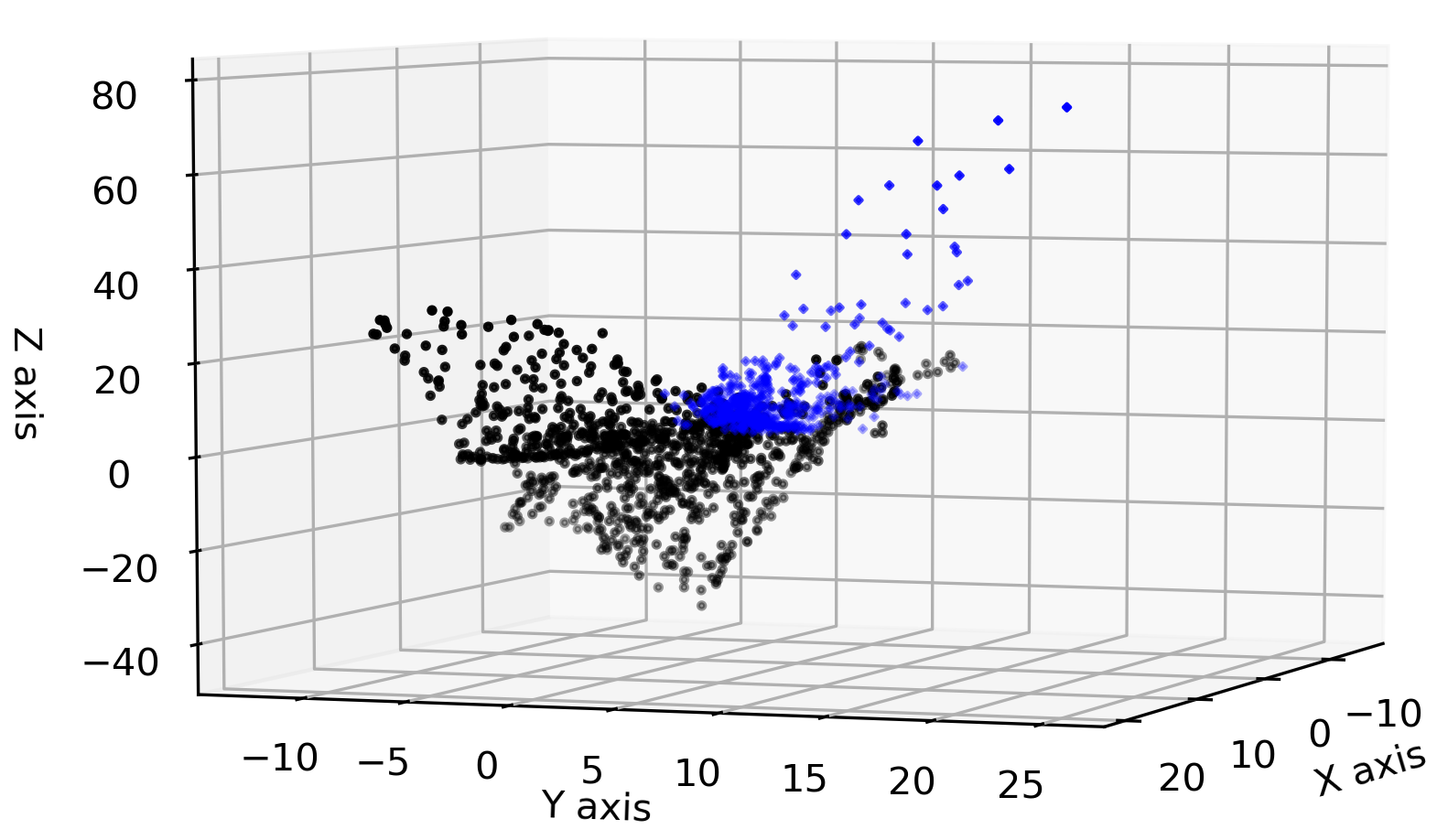}
    \caption{$\rmZ=\mathscr{E}(\rmX)$ projected onto 3D}
    %%%%%%%%%%%%%%%%%%%%%%%%%
    \label{fig:RSRAE_y}
    %%%%%%%%%%%%%%%%%%%%%%%%%
\end{subfigure}
{\large$\xrightarrow[ \begin{subarray}{c} \rm{linear \,mapping} \\ \rmA:\R^{128} \rightarrow \R^{2} \end{subarray}   ]{\rm{~RSR~}}$}%
\begin{subfigure}{0.3\textwidth}
\centering
    \includegraphics[height=6em]{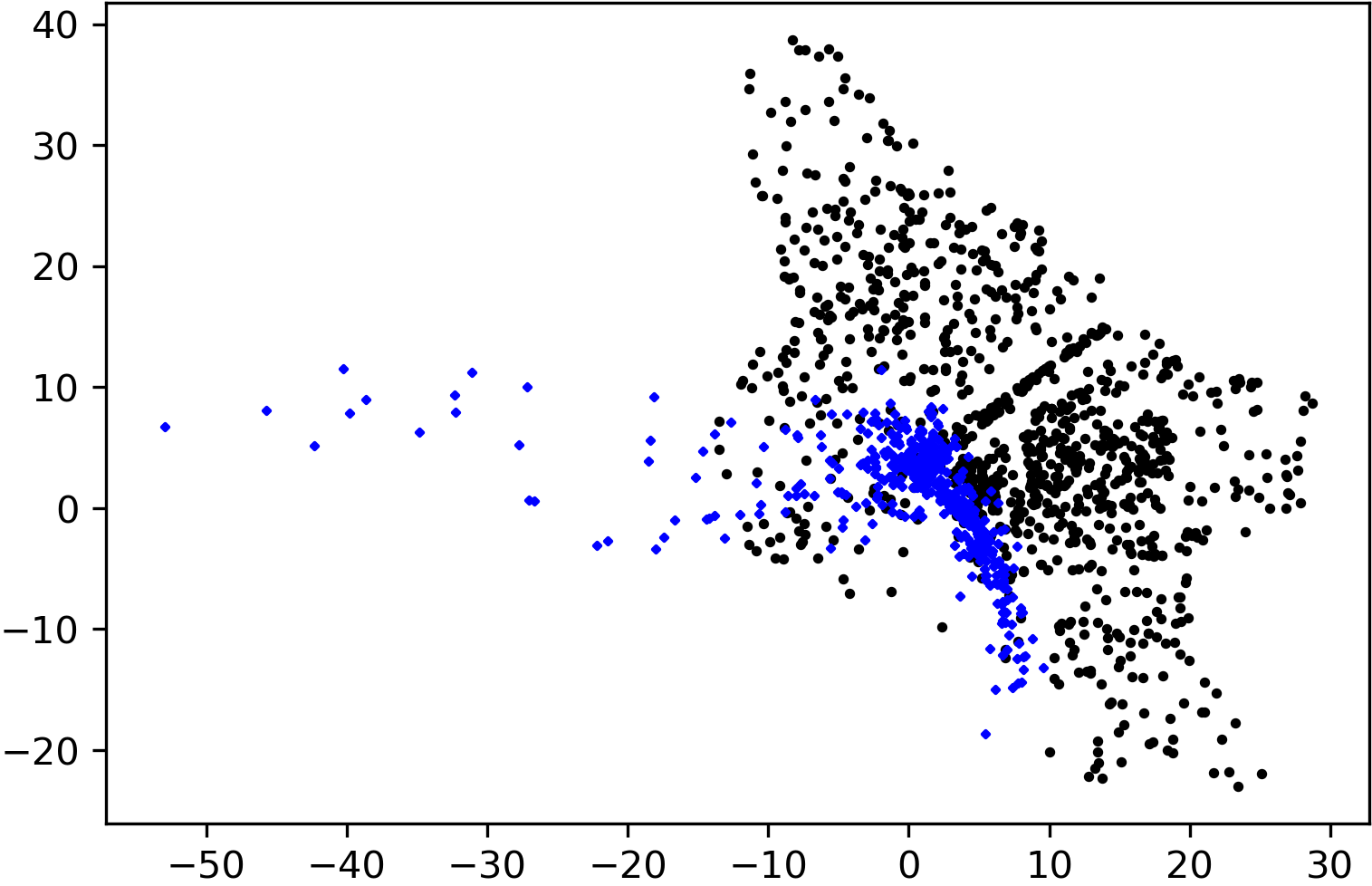}
    \caption{$\tilde{\rmZ} = \rmA \rmZ$ }
    %%%%%%%%%%%%%%%%%%%%%%%%%
    \label{fig:RSRAE_yrsr}
    %%%%%%%%%%%%%%%%%%%%%%%%%
\end{subfigure}
{\large$\xrightarrow[\begin{subarray}{c} \mathscr{D}:\R^2 \rightarrow \R^{3}  \end{subarray}]{ \rm{Decoder}  } $}%
\begin{subfigure}{0.3\textwidth}
\centering
    \includegraphics[height=6em]{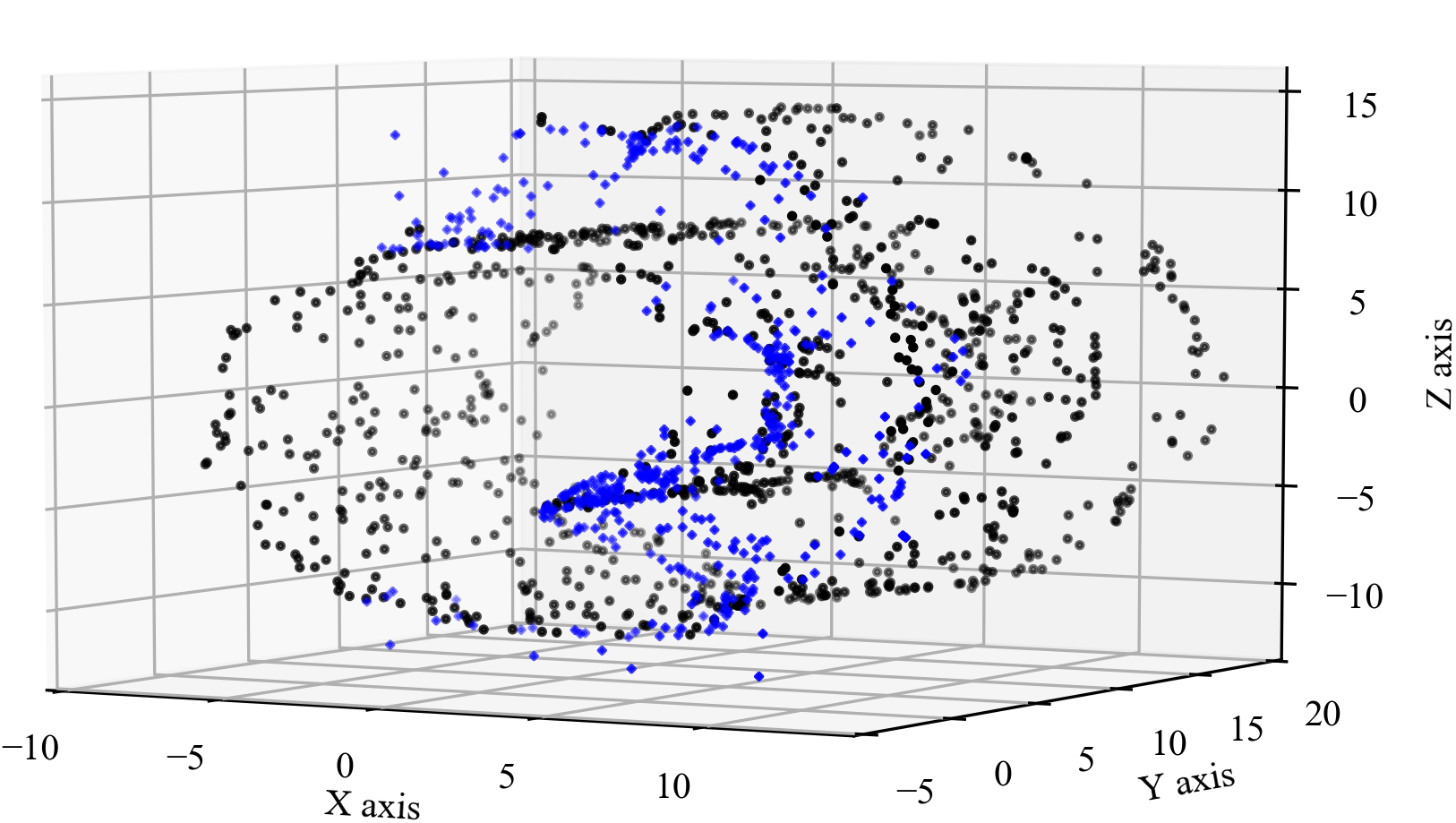}
    \caption{Output of RSRAE \\ $\tilde{\rmX}=\mathscr{D}(\tilde{\rmZ})$}
    %%%%%%%%%%%%%%%%%%%%%%%%%
    \label{fig:RSRAE_xtilde}
    %%%%%%%%%%%%%%%%%%%%%%%%%
\end{subfigure}%

\caption{{Demonstration of the output of the encoder, RSR layer and decoder of RSRAE on a corrupted Swiss roll dataset.}}
\label{fig:RSRAE_structure}

\end{figure}

%%%%%%%%%%%%%%%%%%%%%%%%%%%%%%%%%%%%%%%%%%%%%%%%%%%%%%%%%%%
\begin{figure}[ht]
\centering
\begin{subfigure}{0.3\textwidth}
\centering
    \includegraphics[height=6em]{photos/appendix/hi_x.png}
    \caption{Input data $\rmX$}
    %%%%%%%%%%%%%%%%%%%%%%%%%
    \label{fig:AE_x}
    %%%%%%%%%%%%%%%%%%%%%%%%%
\end{subfigure}
{\large$\xrightarrow[\begin{subarray}{c} \mathscr{E}:\R^3 \rightarrow \R^{128}\end{subarray}]{ \rm{Encoder}  } $}%
\begin{subfigure}{0.3\textwidth}
\centering
    \includegraphics[height=6em]{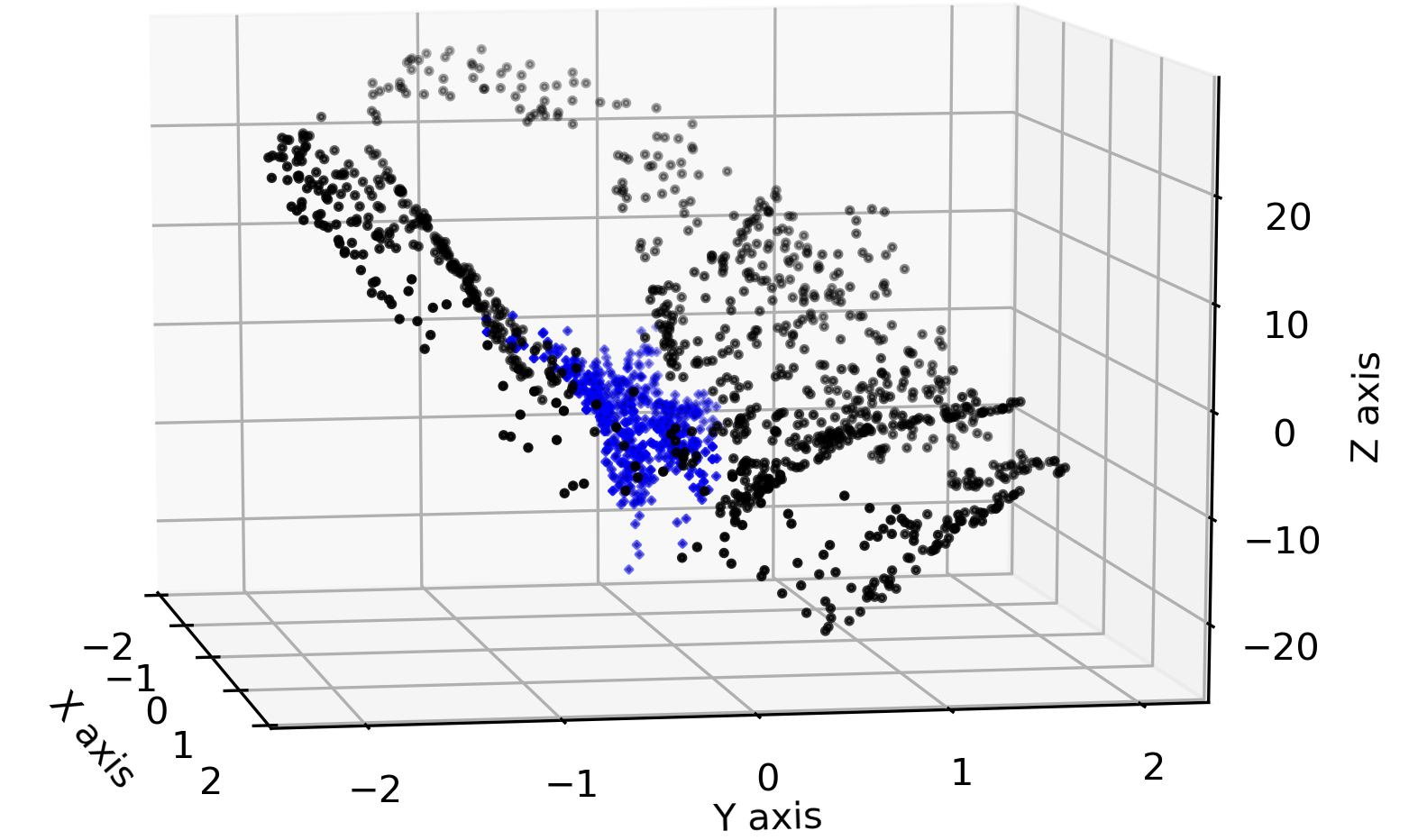}
    \caption{$\rmZ=\mathscr{E}(\rmX)$ projected onto 3D}
    %%%%%%%%%%%%%%%%%%%%%%%%%
    \label{fig:AE_y}
    %%%%%%%%%%%%%%%%%%%%%%%%%
\end{subfigure}
{\large$\xrightarrow[ \begin{subarray}{c}  \\ \rmA:\R^{128} \rightarrow \R^{2} \end{subarray}   ]{\rm{linear \,mapping}}$}%
\begin{subfigure}{0.3\textwidth}
\centering
    \includegraphics[height=6em]{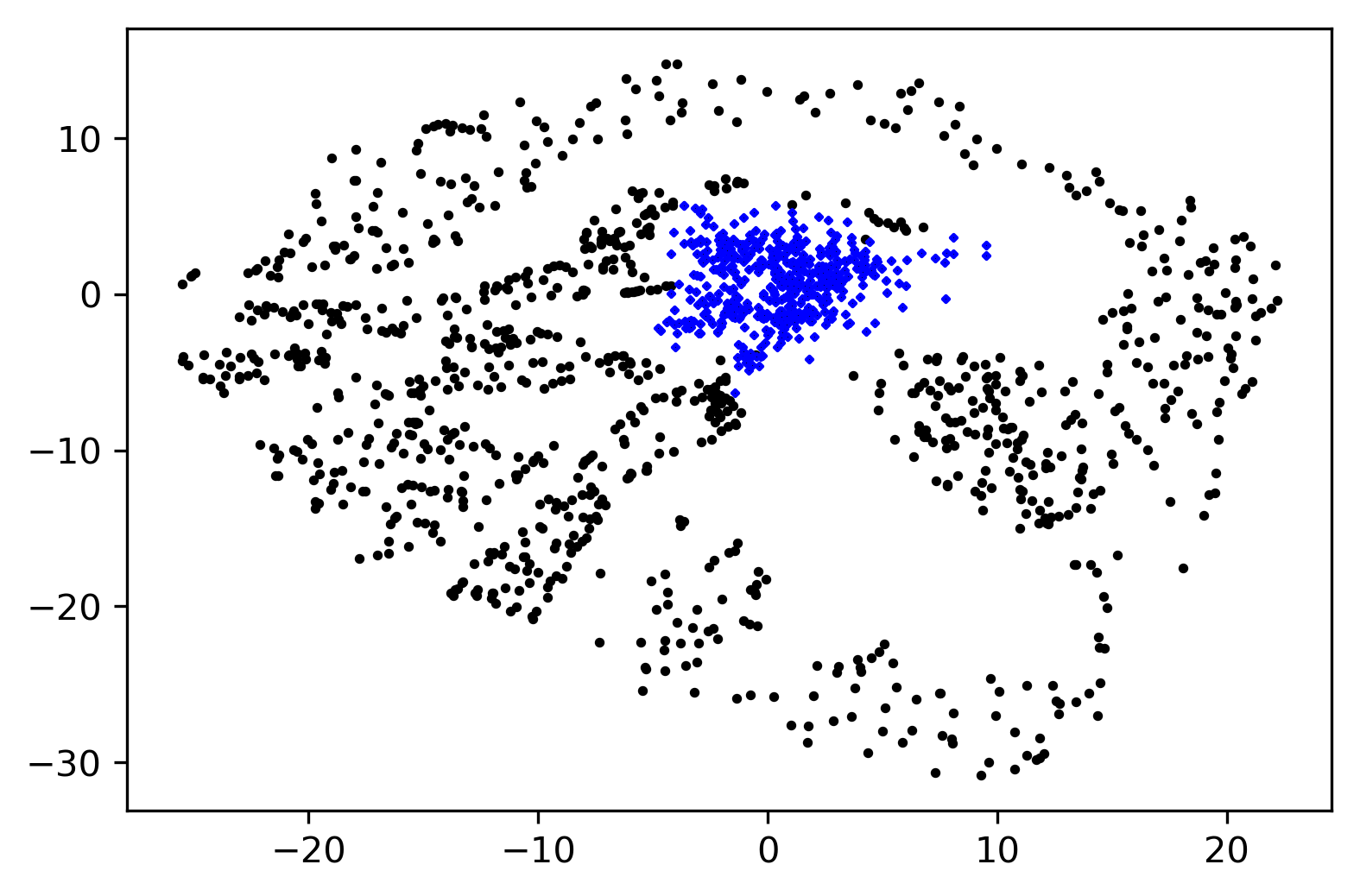}
    \caption{$\tilde{\rmZ} = \rmA \rmZ$}
    %%%%%%%%%%%%%%%%%%%%%%%%%
    \label{fig:AE_yrsr}
    %%%%%%%%%%%%%%%%%%%%%%%%%
\end{subfigure}
{\large$\xrightarrow[\begin{subarray}{c} \mathscr{D}:\R^2 \rightarrow \R^{3}  \end{subarray}]{ \rm{Decoder}  } $}%
\begin{subfigure}{0.3\textwidth}
\centering
    \includegraphics[height=6em]{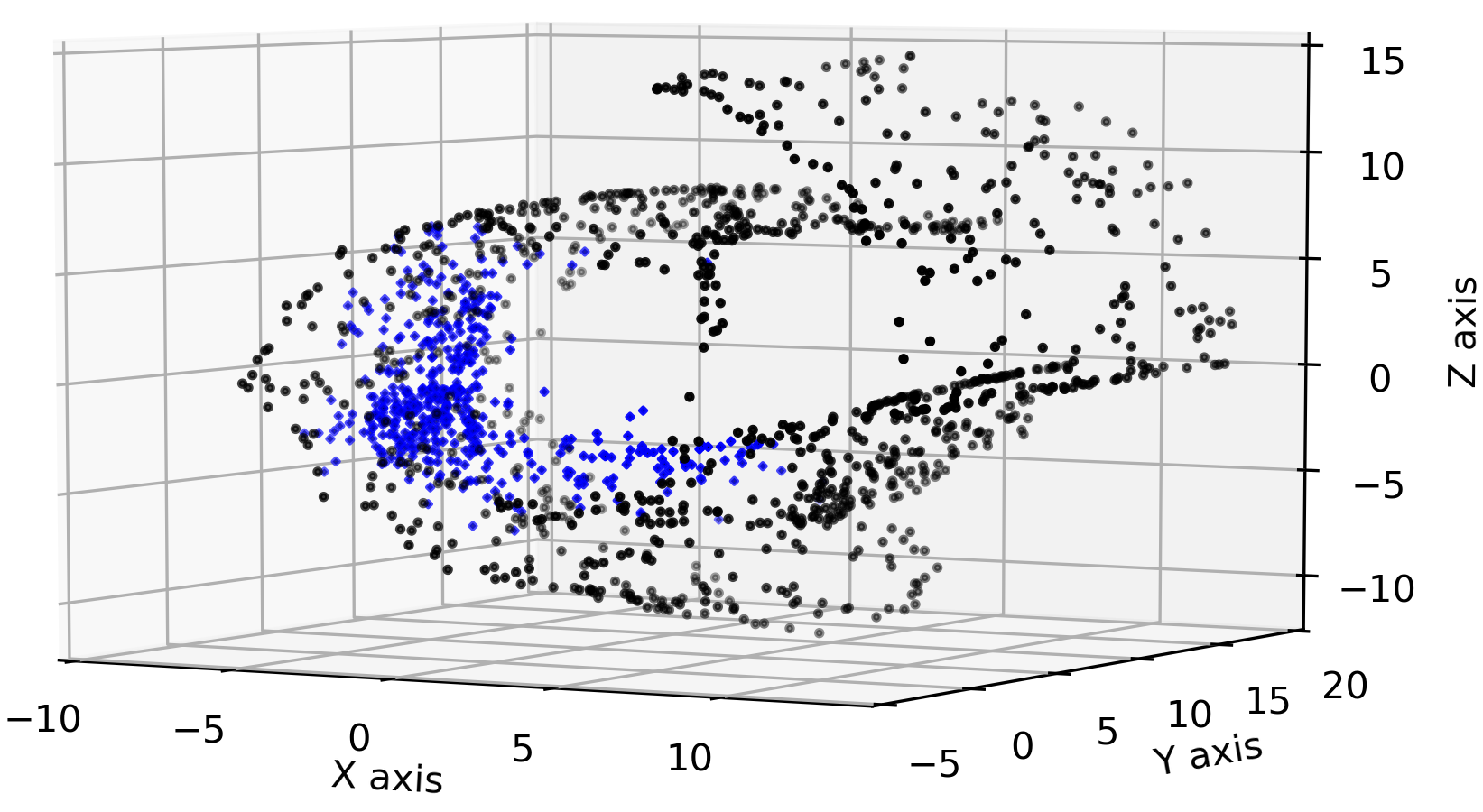}
    \caption{Output of AE \\ $\tilde{\rmX}=\mathscr{D}(\tilde{\rmZ})$}
    %%%%%%%%%%%%%%%%%%%%%%%%%
    \label{fig:AE_xtilde}
    %%%%%%%%%%%%%%%%%%%%%%%%%
\end{subfigure}
\caption{{Demonstration of the output of the encoder, mapping by $\rmA$, and decoder of AE on a corrupted Swiss roll dataset.}}
\label{fig:AE_structure}
\end{figure}

%%%%%%%%%%%%%%%%%%%%%%%%%%%%%%%%%%%%%%%%%%%%%%%%%%%%%%%%%%%%
\begin{figure}[htbp]
\centering
\begin{minipage}[t]{0.48\textwidth}
\centering
\includegraphics[width=6cm]{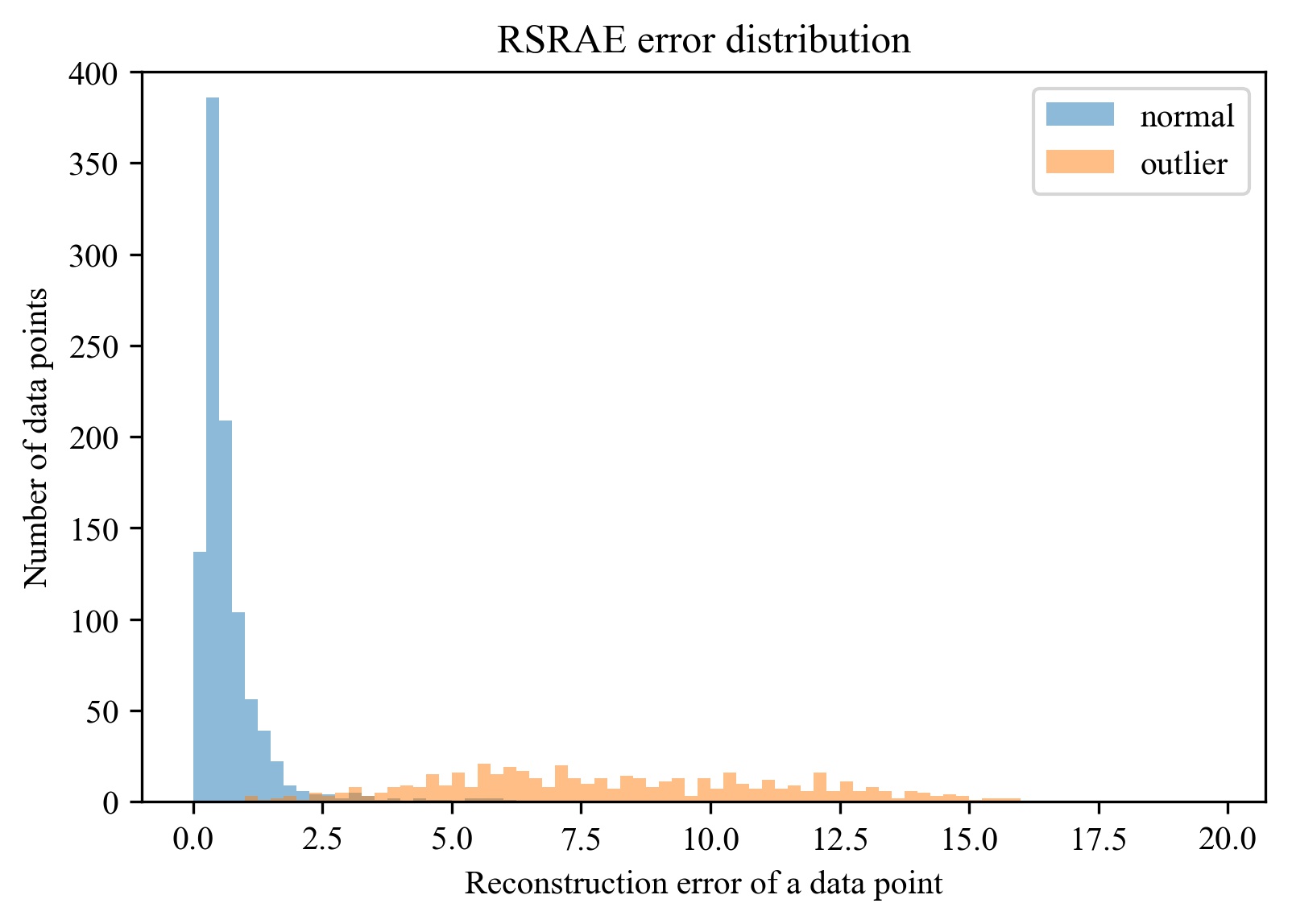}
\subcaption{Error distribution for RSRAE.}
\label{fig:RSRAE_hist}

\end{minipage}
\begin{minipage}[t]{0.48\textwidth}
\centering
\includegraphics[width=6cm]{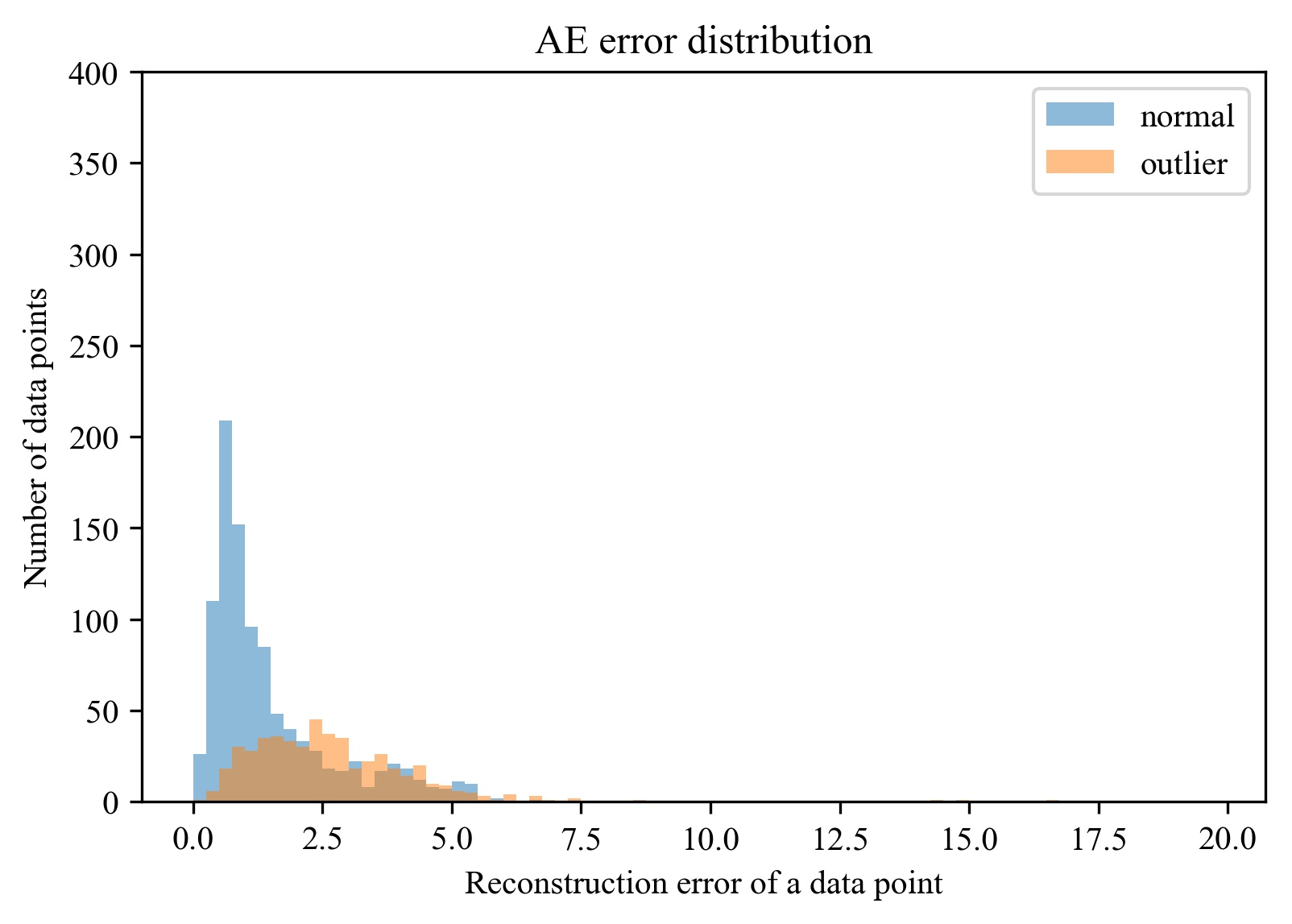}
\subcaption{Error distribution for AE.}
\label{fig:AE_hist}
\end{minipage}
\caption{{Demonstration of the reconstruction error distribution for RSRAE and AE.}}
\label{fig:histograms}
\end{figure}

%%%%%%%%%%%%%%%%%%%%%%%%%%%%%%%%%%%%%%%%%%%%%%%%%%%%%%%%%%%%%%%%%%

\newpage

\section{Further discussion of the RSR term}
\label{sec:more_RSR_term}

The RSR energy in \eqref{eq:lossRSR} includes two different terms. 
The proposition below indicates that the second term of \eqref{eq:lossRSR} is zero when plugging into it the solution of the minimization of the first term of \eqref{eq:lossRSR} with the additional requirement that $\rmA$ has full rank. That is, in theory, one may only minimize the first term of \eqref{eq:lossRSR} over the set of matrices $\rmA \in \R^{d \times D}$ with full rank. We then discuss computational issues of this different minimization.

\begin{proposition}\label{prop:noneedboth}
Assume that $\{ \rvz^{(t)} \}_{t=1}^N \subset \R^D$ spans $\R^D$, $d \leqslant D$ and let
\begin{equation}\label{op2}
    \rmA^{\star} =  \argminB_{\substack{\rmA \in \R^{d \times D} \\ \rm{rank}(\rmA) = d}} \sum_{t=1} ^{N} \norm{ \rvz^{(t)} - \rmA^{\rm{T}}\rmA \rvz^{(t)} }_2 .
\end{equation}
Then $\rmA^{\star} {\rmA^{\star}}^{\rm{T}} = \rmI_d.$
\end{proposition}

\begin{proof}

Let $\rmA^{\star}$ be an optimizer of \eqref{op2} and $\rmP^{\star}$ denote the orthogonal projection onto the range of $\rmA^{\star \rm{T}} \rmA^{\star}$. Note that $\rmP^{\star}$ can be written as $\tilde{\rmA}^{\rm{T}} \tilde{\rmA}$, where $\tilde{\rmA}$ is a $d \times D$ matrix composed of an orthonormal basis of the range of $\rmP^{\star}$. Therefore, being an optimum of \eqref{op2}, $\rmA^{\star}$ satisfies
\begin{equation}\label{eq:projgeqaaopt}
    \norm{\rvz^{(t)} - \rmP^{\star} \rvz^{(t)}}_2 \geq \norm{\rvz^{(t)} - \rmA^{\star \rm{T}} \rmA^{\star} \rvz^{(t)}}_2 ~, \quad t = 1, \cdots, N ~.
\end{equation}
On the other hand, the definition of orthogonal projection implies that
\begin{equation}\label{eq:projleqaa}
    \norm{\rvz^{(t)} - \rmP^{\star} \rvz^{(t)}}_2 \leq \norm{\rvz^{(t)} - \rmA^{\star \rm{T}} \rmA^{\star} \rvz^{(t)}}_2 ~, \quad t = 1, \cdots, N ~.
\end{equation}
That is, equality is obtained in \eqref{eq:projgeqaaopt} and \eqref{eq:projleqaa}. This equality and the fact that $\rmP^{\star}$ is a projection on the range of $\rmA^{\star \rm{T}} \rmA^{\star}$ imply that  
\begin{equation}\label{eq:allequal}
    \rmP^{\star} \rvz^{(t)} = \rmA^{\star \rm{T}} \rmA^{\star} \rvz^{(t)} ~, \quad t = 1, \cdots, N ~.
\end{equation}
Since $\{\rvz^{(t)}\}_{t=1}^N$ spans $\R^D$, \eqref{eq:allequal} results in
\begin{equation}
    \rmP^{\star} = \rmA^{\star \rm{T}} \rmA^{\star} ~,
\end{equation}
which further implies that 
\begin{equation}
\rmA^{\star} \rmA^{\star \rm{T}} \rmA^{\star} = \rmA^{\star} \rmP^{\star} = \rmA^{\star} ~.
\end{equation}
Combining this observation ($\rmA^{\star} \rmA^{\star \rm{T}} \rmA^{\star} = \rmA^{\star}$) with the constraint that $\rmA^{\star}$ has a full rank, we conclude that $\rmA^{\star} {\rmA^{\star}}^{\rm{T}} = \rmI_d$.

\end{proof}

The minimization in \eqref{op2} is nonconvex and intractable. Nevertheless, \citet{lerman2017fast} propose a heuristic to solve it with some weak guarantees 
and \citet{maunu2017well} propose an algorithm with guarantees under some conditions.
However, such a minimization is even more difficult when applied to the combined energy in \eqref{eq:combined}, instead of \eqref{eq:lossRSR}. Therefore, we find it necessary to include the second term in \eqref{eq:lossRSR} that imposes the nearness of $\rmA^{\rm{T}} \rmA$ to an orthogonal projection (equivalently, of $ \rmA \rmA^{\rm{T}}$ to the identity). 
%%%%%%%%%%%%%%%%%%%

\section{More on related theory for the RSR penalty}
\label{sec:GAN}

In \Secref{subsec:linearaeforrsr} we characterize the solution of \eqref{eq:lossAEPCA} via a subspace problem. Special case solutions to this problem include both the PCA subspace and the least absolute deviations subspace.
In \Secref{sec:proofsubspace} we prove Proposition \ref{prop:subspace}. In \Secref{subsec:betas} we review some pure mathematical work that we find relevant to this discussion. 
\subsection{Property of linear autoencoders} \label{subsec:linearaeforrsr}
The following proposition expresses the solution of \eqref{eq:lossAEPCA} in terms of another minimization problem.
After proving it, we clarify that the other minimization problem is related to both PCA and RSR.

\begin{proposition}
\label{prop:deforpcarsr}
Let $p \geq 1$, $d<D$, and $\{\rvx^{(t)}\}_{t=1}^N \subset \R^D$ be a dataset with rank at least $d$. If $(\rmD^{\star}, \rmE^{\star}) \in \R^{D \times d} \times \R^{d \times D}$ is a minimizer of \eqref{eq:lossAEPCA}, then
\begin{equation}
\label{eq:deforpcarsr}
\rmD^{\star} \rmE^{\star} = \rmP^{\star} ~,
\end{equation}
where $\rmP^{\star} \in \R^{D \times D}$ is a minimizer of 
\begin{equation}
\label{eq:pstarargminxpx}
\sum_{t=1}^N \norm{\rvx^{(t)} - \rmP \rvx^{(t)}}_2^p ~,
\end{equation}
among all orthoprojectors $\rmP$ (that is, $\rmP = \rmP^T$ and $\rmP^2 = \rmP$) of rank $d$. 
\end{proposition}

\begin{proof}
Let $\rmP^{\diamond}$ be a minimizer of \eqref{eq:pstarargminxpx} and 
$(\rmD^{\star}, \rmE^{\star})$ be a minimizer of \eqref{eq:lossAEPCA}. Since $\rmP^{\diamond}$ is an orthoprojector of rank $d$ it can be written as $\rmP^{\diamond} = \rmU^{\diamond} {\rmU^{\diamond}}^{\rm{T}}$, where $\rmU^{\diamond} \in \R^{D \times d}$, and thus
\begin{equation}
\label{eq:deuutp}
    \sum_{t=1}^N \norm{\rvx^{(t)} - \rmD^{\star} \rmE^{\star} \rvx^{(t)}}_2^p \leq \sum_{t=1}^N \norm{\rvx^{(t)} - \rmU^{\diamond} {\rmU^{\diamond}}^{\rm{T}} \rvx^{(t)}}_2^p = \sum_{t=1}^N \norm{\rvx^{(t)} - \rmP^{\diamond} \rvx^{(t)}}_2^p ~.
\end{equation}

Let $\LL$ denote the column space of $\rmD^{\star} \rmE^{\star}$. Then by the property of orthoprojection
\begin{equation}
\label{eq:deplp0}
    \norm{\rvx^{(t)} - \rmD^{\star} \rmE^{\star} \rvx^{(t)}}_2 
    \geq 
    \norm{\rvx^{(t)} - \rmP_{\LL} \rvx^{(t)}}_2 \ \text{ for } 1 \leq t \leq N
\end{equation}
and consequently
\begin{equation}
\label{eq:deplp}
    \sum_{t=1}^N \norm{\rvx^{(t)} - \rmD^{\star} \rmE^{\star} \rvx^{(t)}}_2^p \geq \sum_{t=1}^N \norm{\rvx^{(t)} - \rmP_{\LL} \rvx^{(t)}}_2^p \geq \sum_{t=1}^N \norm{\rvx^{(t)} - \rmP^{\diamond} \rvx^{(t)}}_2^p ~.
\end{equation}

The combination of \eqref{eq:deuutp} and \eqref{eq:deplp} yields the following two equalities
\begin{equation}
\label{eq:equalityprop0}
    \sum_{t=1}^N \norm{\rvx^{(t)} - \rmP_{\LL} \rvx^{(t)}}_2^p = \sum_{t=1}^N \norm{\rvx^{(t)} - \rmP^{\diamond} \rvx^{(t)}}_2^p ~,
\end{equation}
\begin{equation}
\label{eq:equalityprop}
    \sum_{t=1}^N \norm{\rvx^{(t)} - \rmD^{\star} \rmE^{\star} \rvx^{(t)}}_2^p = \sum_{t=1}^N \norm{\rvx^{(t)} - \rmP_{\LL} \rvx^{(t)}}_2^p ~.
\end{equation}
We note that \eqref{eq:equalityprop0} implies  that $\rmP_{\LL}$ is a minimizer of \eqref{eq:pstarargminxpx} (among all rank $d$ orthoprojectors).
We further note that \eqref{eq:deplp0} and \eqref{eq:equalityprop} yield that for all 
$1 \leq t \leq N$
\begin{equation}
\label{eq:deplp1}
    \norm{\rvx^{(t)} - \rmD^{\star} \rmE^{\star} \rvx^{(t)}}_2 
    =
    \norm{\rvx^{(t)} - \rmP_{\LL} \rvx^{(t)}}_2 ~. 
\end{equation}
Since $\rmD^{\star} \rmE^{\star} \rvx^{(t)} \in \LL$ and $\rmP_{\LL}$ is an orthoprojector we conclude from \eqref{eq:deplp1} that 
\begin{equation}
\label{eq:deplp2}
\rmD^{\star} \rmE^{\star} \rvx^{(t)} = \rmP_{\LL} \rvx^{(t)} \ \text{ for }1 \leq t \leq N.    
\end{equation}
We note that the definition of $(\rmD^{\star},\rmE^{\star})$ implies that $\LL$ (which is the column space of  
$\rmD^{\star}\rmE^{\star}$) is contained in the span of  $\{\rvx^{(t)}\}_{t=1}^N$. We also recall that the dimension of the span of  $\{\rvx^{(t)}\}_{t=1}^N$ is at least the dimension of $\LL$, that is, $d$. Combining the latter facts with \eqref{eq:deplp2} we obtain that $\rmD^{\star} \rmE^{\star} = \rmP_{\LL}$. 
This and the fact that $\rmP_{\LL}$ is a minimizer of \eqref{eq:pstarargminxpx} (which was derived from \eqref{eq:equalityprop0}) 
concludes \eqref{eq:deforpcarsr}.

\end{proof}

Note that when $p=2$, the energy function in \eqref{eq:pstarargminxpx} corresponds to 
PCA. More precisely, a minimizer $\rmP^{\star}$
of \eqref{eq:pstarargminxpx} (among rank $d$ orthoprojectors) is an orthoprojector on a $d$-dimensional PCA subspace, equivalently, a subspace spanned by top $d$ eigenvectors of the sample covariance (we assume for simplicity linear, and not affine, autoencoder, so the PCA subspace is linear and thus when $p=2$ the data is centered at the origin).  This minimizer is unique if and only if the $d$-th eigenvalue of the sample covariance is larger than the $(d+1)$-st eigenvalue. These elementary facts are reviewed in Section II-A
of \citet{lerman2018overview}.

When $p=1$, the minimizer $\rmP^{\star}$
of \eqref{eq:pstarargminxpx} (among rank $d$ orthoprojectors) is an orthoprojector on the $d$-dimensional least absolute deviations subspace.
This subspace is reviewed in Section II-D of   \citet{lerman2018overview} as a common approach for RSR. 
The minimizer is often not unique, where sufficient and necessary conditions for local minima of \eqref{eq:pstarargminxpx} are studied in \citet{lp_recovery_part1_11}.

\subsection{Proof of Proposition \ref{prop:subspace}}
\label{sec:proofsubspace}

\begin{proof}%[Proof of Proposition \ref{prop:subspace}]
We denote the subspace $\LL$ in the left hand side of \eqref{eq:rsrganconclusion} by $\LL^{\star}$ in order to distinguish it from the generic notation $\LL$ for subspaces. 
Consider the random variable $X \sim \mu$, Where $\mu$ is $\mathcal{N}(\rvm_X, \rmSigma_X)$. Fix $\pi \in \Pi(\rvmu, \rvnu)$. We note that 
\begin{align}
& \E_{(X,Y) \sim \pi} \norm{X-Y}_2^p \nonumber \\
~=~ & \int_{\R^D} \int_{\R^D} {\norm{\rvx-\rvy}_2^p \pi(\rvx,\rvy) } \rm{d}\rvx ~ \rm{d}\rvy \nonumber \\
~\geq~ & \min_{\rm{dim} \LL = d} \int_{\R^D} \rm{dist}(\rvx,\LL)^p \int_{\R^D} {\pi(\rvx,\rvy)} \rm{d}\rvy ~ \rm{d}\rvx \label{eq:wassersteintodist} \\
~=~ & \min_{\rm{dim} \LL = d} \int_{\R^D} \rm{dist}(\rvx,\LL)^p \mu(\rvx) ~ \rm{d}\rvx \nonumber
\\
~=~ & \min_{\rm{dim} \LL = d} \E \norm{X - \rmP_\LL X}_2^p
~. \nonumber
\end{align}
The inequality in \eqref{eq:wassersteintodist} holds since $X$ is fixed and $Y$ satisfies $(X,Y) \sim \pi$, so the distribution of $Y$ is $\mathcal{N}(\rvm_Y, \rmSigma_Y)$. Therefore, almost surely, $Y$ takes values in the $d$-dimensional affine subspace $\{\rvy \in \R^D: \rvy - \rvm_Y \in \rm{range}(\rmSigma_Y)\}$.
% since $Y$ that satisfies $(X,Y) \sim \pi$ takes values in a $d$ dimensional affine subspace almost surely. 
Furthermore, we note that equality in \eqref{eq:wassersteintodist} is achieved when $Y = \rmP_{\LL^{\star}} X$. 

We conclude the proof by showing that 
\begin{equation}
\label{eq:center}
\rvm_X \in \LL^{\star}.
\end{equation}
Indeed, \eqref{eq:center} implies that the orthogonal projection of $X \sim \mathcal{N}(\rvm_X, \rmSigma_X)$ onto $\LL^{\star}$
results in a random variable with distribution $\rvnu$ which is $\mathcal{N} (\rvm_X, \rmP_{\LL^{\star}} \rmSigma_X \rmP_{\LL^{\star}})$. By the above observation about the optimality of $Y = \rmP_{\LL^{\star}} X$, the density of this distribution is the optimal solution of \eqref{eq:optprob}.

To prove \eqref{eq:center}, we assume without loss of generality that $\rvm_X = \mathbf{0}$. Denote the orthogonal projection of the origin onto the affine subspace $\LL^{\star}$ by $\rvm_{\LL^{\star}}$
% displacement of $\LL$ from the origin to be $\rvm_{\LL}$ 
and let $\LL_0 = \LL^{\star} - \rvm_{\LL^{\star}}$. We need to show that 
$\LL^{\star} = \LL_0$, or equivalently, $\rvm_{\LL^{\star}} = \mathbf{0}$. We note $\LL_0$ is a linear subspace, $\rvm_{\LL^{\star}}$ is orthogonal to  $\LL_0$ and thus there exists a rotation matrix $\rmO$ such that 
\begin{equation}\label{eq:wlogl}
    \rmO \LL_0 = \left\{ (0, \cdots, 0, z_{D-d+1}, \cdots, z_D):~ z_{D-d+1}, \cdots z_D \in \R \right\} ~,
\end{equation}
and
\begin{equation}\label{eq:wlogm}
    \rmO \rvm_{\LL^{\star}} = (m_1, \cdots, m_{D-d}, 0, \cdots, 0) ~.
\end{equation}

For any $\rvx \in \R^D$ we note that $\mu(\rvx) = \mu(-\rvx)$ since $\mu$ is Gaussian. Using this observation, other basic observations and the notation $\rmO \rvx = (x'_1, \cdots, x'_D)$ we obtain that
\allowdisplaybreaks
\begin{align}
    & \rm{dist}(\rvx,\LL^{\star})^p \mu(\rvx) + \rm{dist}(-\rvx, \LL^{\star})^p \mu(-\rvx) \nonumber \\
~=~ & \left( \rm{dist}(\rvx,\LL^{\star})^p + \rm{dist}(-\rvx, \LL^{\star})^p \right) \mu(\rvx) \nonumber \\
~=~ & \left( \rm{dist}(\rmO \rvx, \rmO \LL^{\star})^p + \rm{dist}(- \rmO \rvx, \rmO \LL^{\star})^p \right) \mu(\rvx) \nonumber \\
~=~ & \left( \left( \sum_{i=1}^{D-d} (x'_i - m_i)^2 \right)^{p/2} + \left( \sum_{i=1}^{D-d} ( - x'_i - m_i)^2 \right)^{p/2} \right) \mu(\rvx) \nonumber \\
~=~ & \left( \left( \sum_{i=1}^{D-d} (x'_i - m_i)^2 \right)^{p/2} + \left( \sum_{i=1}^{D-d} ( x'_i + m_i)^2 \right)^{p/2} \right) \mu(\rvx) \nonumber \\
~\geq~ & 2 \left( \sum_{i=1}^{D-d} {x'_i}^2 \right)^{p/2} \mu(\rvx) \label{eq:convex}\\
~=~ & 2 ~ \rm{dist}(\rmO \rvx, \rmO \LL_0)^p \mu(\rvx) \nonumber \\
~=~ & 2 ~ \rm{dist}(\rvx, \LL_0)^p \mu(\rvx) \nonumber \\
~=~ & \left( \rm{dist}(\rvx, \LL_0)^p + \rm{dist}(-\rvx, \LL_0)^p \right) \mu(\rvx) \nonumber \\
~=~ & \rm{dist}(\rvx,\LL_0)^p \mu(\rvx) + \rm{dist}(-\rvx, \LL_0)^p \mu(-\rvx) ~. \nonumber
\end{align}

The inequality in \eqref{eq:convex} follows from the fact that for  $p \geq 1$, the function $\norm{\cdot}_2^p$ is convex as it is a composition of the convex function $\norm{\cdot}_2: \R^d \rightarrow \R_{+}$ and the increasing convex function $(\cdot)^p: \R_{+} \rightarrow \R_{+}$. Equality is achieved in \eqref{eq:convex} if $m_i = 0$ for $i=1, \cdots, D-d$, that is, $\LL^{\star} = \LL_0$.

Integrating the left and right hand sides of \eqref{eq:convex} over $\R^D$ results in 
\begin{equation}\label{eq:intdistcpr}
    \int_{\R^D} \rm{dist}(\rvx, \LL^{\star})^p \mu(\rvx) d\rvx \geq \int_{\R^D} \rm{dist}(\rvx, \LL_0)^p \mu(\rvx) d\rvx ~.
\end{equation}
Since $\LL^{\star}$ is a minimizer among all affine subspaces of rank $d$ of  
$\int_{\R^D} \rm{dist}(\rvx,\LL)^p \mu(\rvx) ~ \rm{d}\rvx = \E \norm{X - \rmP_\LL X}_2^p$, equality is obtained in \eqref{eq:intdistcpr}. Consequently, equality is obtained, almost everywhere, in \eqref{eq:convex}. Therefore, $\LL^{\star} = \LL_0$ and the claim is proved.
\end{proof}

\subsection{Relevant Mathematical Theory}
\label{subsec:betas}
We note that a complex network can represent a large class of functions. 
Consequently, for a sufficiently complex network, minimizing the loss function in \eqref{eq:lossAE} results in minimum value zero. In this case the minimizing ``manifold'' contains the original data, including the outliers. On the other hand, the RSR loss term imposes fitting a subspace that robustly fits only part of the data and thus cannot result in minimum value zero. Nevertheless, imposing a subspace constraint might be too restrictive, even in the latent space.  A seminal work by \citet{Jones90} studies optimal types of curves that contain general sets. This work relates the construction and optimal properties of these curves with multiscale approximation of the underlying set by lines. It was generalized to higher dimensions in \citep{DS93} and to a setting relevant to outliers in \citep{Lerman03}. These works suggest loss functions that incorporate several linear RSR layers from different scales.
Nevertheless, their pure setting does not directly apply to our setting. We have also noticed various technical difficulties when trying to directly implement these ideas to our setting.

\section{Brief description of the baselines and metrics}
\label{sec:describe_baselines}
We first clarify the methods used as baselines in \Secref{sec:real}. 

\textbf{Local Outlier Factor (LOF)} measures the local deviation of a given data point with respect to its neighbors. If the LOF of a data point is too large then the point is determined to be an outlier. 

\textbf{One-Class SVM (OCSVM)} learns a margin for a class of data. Since outliers contribute less than the normal class, it also applies to the unsupervised setting \citep{goldstein2016comparative}. It is usually applied with a non-linear kernel.

\textbf{Isolation Forest (IF)} determines outliers by looking at the number of splittings needed for isolating a sample. It constructs random decision trees. A short path length for separating a data point implies a higher probability that the point is an outlier.

\textbf{Geometric Transformations (GT)} applies a variety of geometric transforms to input images and consequently creates a self-labeled dataset, where the labels are the types of transformations. Its anomaly detection is based on Dirichlet Normality score according to the softmax output from a classification network for the labels. 

\textbf{Deep Structured Energy-Based Models (DSEBMs)} outputs
an energy function which is the negative log probability that a sample follows the data distribution. The energy based model is connected to an autoencoder to avoid the need of complex sampling methods.

\textbf{Deep Autoencoding Gaussian Mixture Model (DAGMM)} is also a deep autoencoder model. It optimizes an end-to-end structure that contains both an autoencoder and an estimator for Gaussian Mixture Model. The anomaly detection is done after modeling the density function of the Gaussian Mixture Model.

Next, we review the definitions of the two metrics that we used: the AUC and AP scores  \citep{davis2006relationship}. In computing these metrics we identify the outliers as ``positive''.

\textbf{AUC (area-under-curve)} is the area under the Receiver Operating Characteristic (ROC) curve. Recall that the True Positive Rate (TPR), or Recall, is the number of samples correctly labeled as positive divided by the total number of actual positive samples. The False Positive Rate (FPR), on the other hand, is the number of negative samples incorrectly labeled as positive divided by the total number of actual negative samples. The ROC curve is a graph of TPR as a function of FPR. It is drawn by recording values of FPR and TPR for different choices of $\rvepsilon_{\rm{T}}$ in Algorithm~\ref{alg:the_alg}.

\textbf{AP (average-precision)} is the area under the Precision-Recall Curve. While Recall is the TPR, Precision is the number of samples correctly labeled as positive divided by the total number of predicted positives. The  Precision-Recall curve is the graph of  Precision as a function of Recall.
It is drawn by recording values of Precision and Recall for different choices of $\rvepsilon_{\rm{T}}$ in Algorithm~\ref{alg:the_alg}.

Both AUC and AP can be computed using the corresponding functions in the scikit-learn package \citep{scikit-learn}.

\section{Comparison with RSR and RCAE}
\label{sec:rsrandrpca}

We demonstrate basic properties of our framework by comparing it to two different frameworks. The first framework is direct RSR, which  tries to model the inliers by a low-dimensional subspace, as opposed to the nonlinear model discussed in here. Based on careful comparison of RSR methods in \citet{lerman2018overview}, we use the Fast Median Subspace (FMS) algorithm \citep{lerman2017fast} and its normalized version, the Spherical FMS (SFMS). The other framework can be viewed a nonlinear version of RPCA, instead of RSR. It assumes sparse elementwise corruption of the data matrix, instead of corruption of whole data points, or equivalently, of some columns of the data matrix. For this purpose we use the Robust Convolutional Autoencoder (RCAE) algorithm of \citet{chalapathy2017robust}, who advocate it as ``extension of robust PCA to allow for a nonlinear manifold that explains most of the data''.  We adopt the same network structures as in \Secref{subsec:benchmark}.

Fig.~\ref{fig:RSRandRCAE} reports comparisons of RSRAE, FMS, SFMS and RCAE on the datasets used in \Secref{subsec:res}. 
We first note that both FMS and SFMS are not effective for the datasets we have been using.  That is, the inliers in these datasets are not well-approximated by a linear model. It is also interesting to notice that without normalization to the sphere, FMS can be much worse than SFMS. That is, SFMS is often way more robust to outliers than FMS. This observation and the fact that there are no obvious normalization procedures a general autoencoder (see \Secref{sec:relatedtheory}) clarifies why the mere use of the $L_{\rm{AE}}^1$ loss for an autoencoder is not expected to be robust enough to outliers.

Comparing with RSRAE, we note that RCAE is not a competitive method for these datasets. This is not surprising since the model of RCAE, which assumes sparse elementwise corruption, does not fit well to the problem of anomaly detection, but to other problems, such as background detection.

\begin{figure}[htbp]
\centering
\begin{minipage}[t]{0.48\textwidth}
\rotatebox{90}{\null \qquad Caltech 101}
\centering
\includegraphics[width=6cm]{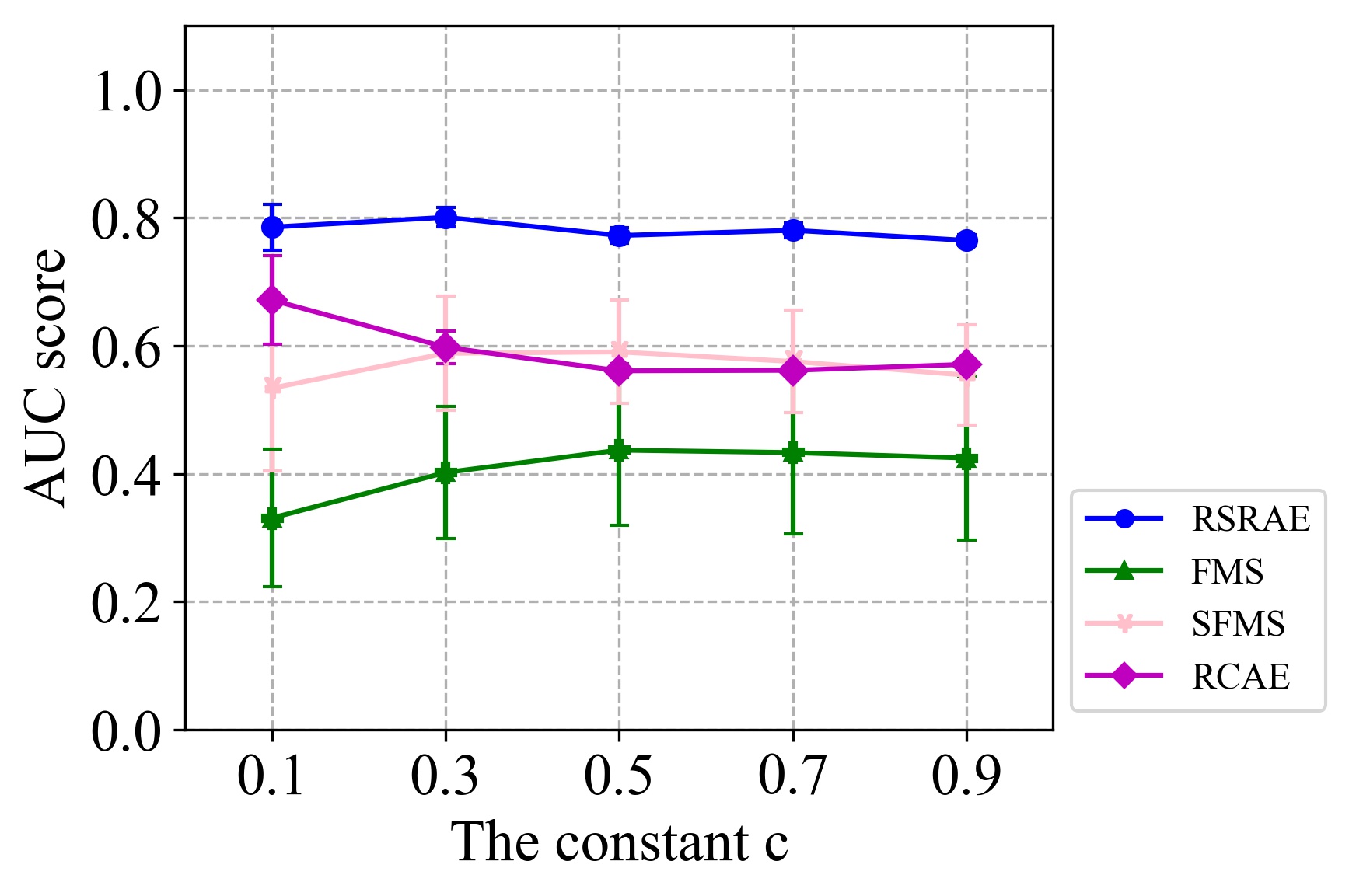}
\end{minipage}
\begin{minipage}[t]{0.48\textwidth}
\centering
\includegraphics[width=6cm]{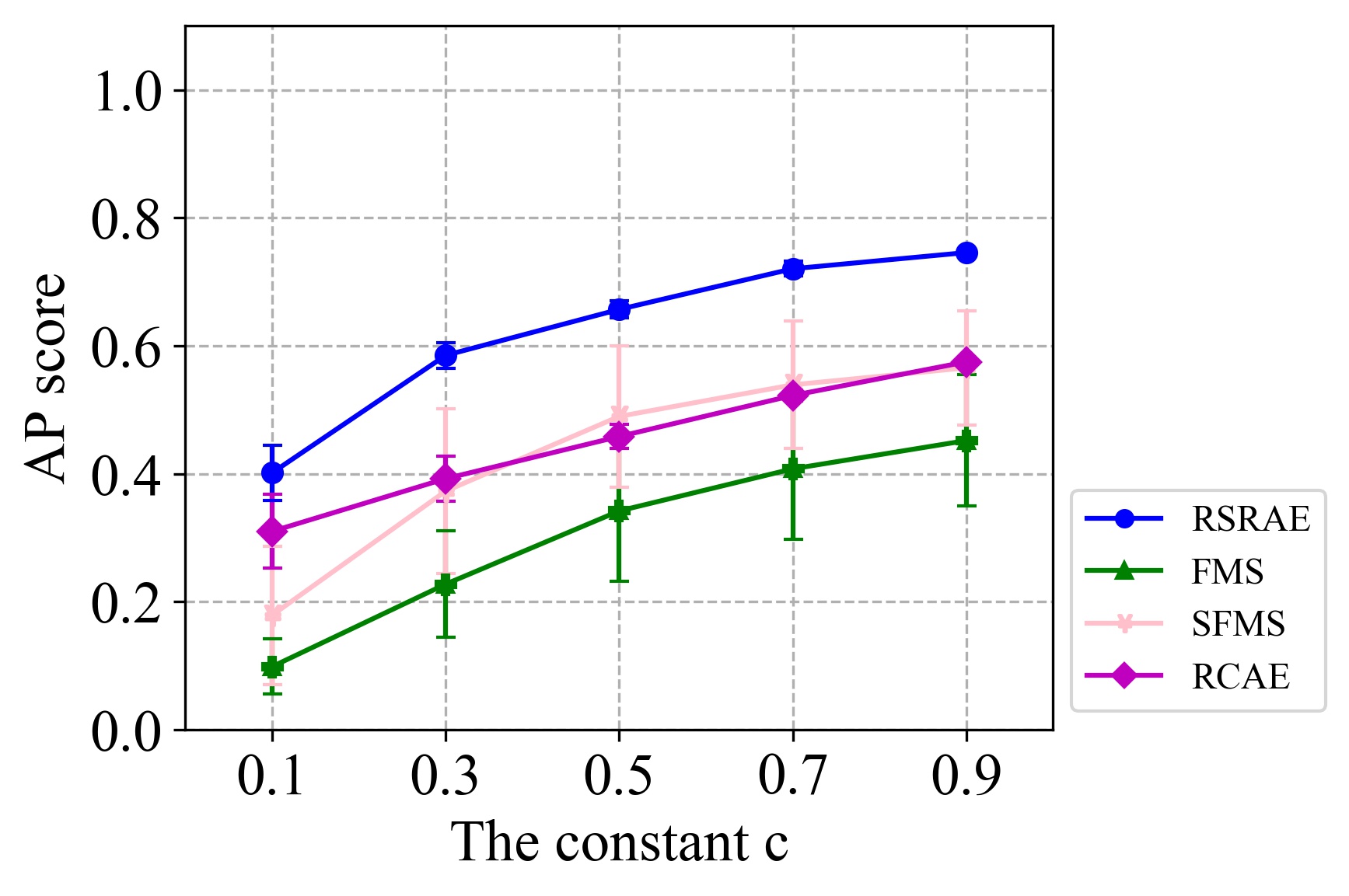}
\end{minipage}

\centering
\begin{minipage}[t]{0.48\textwidth}
\rotatebox{90}{\null \qquad Fashion MNIST}
\centering
\includegraphics[width=6cm]{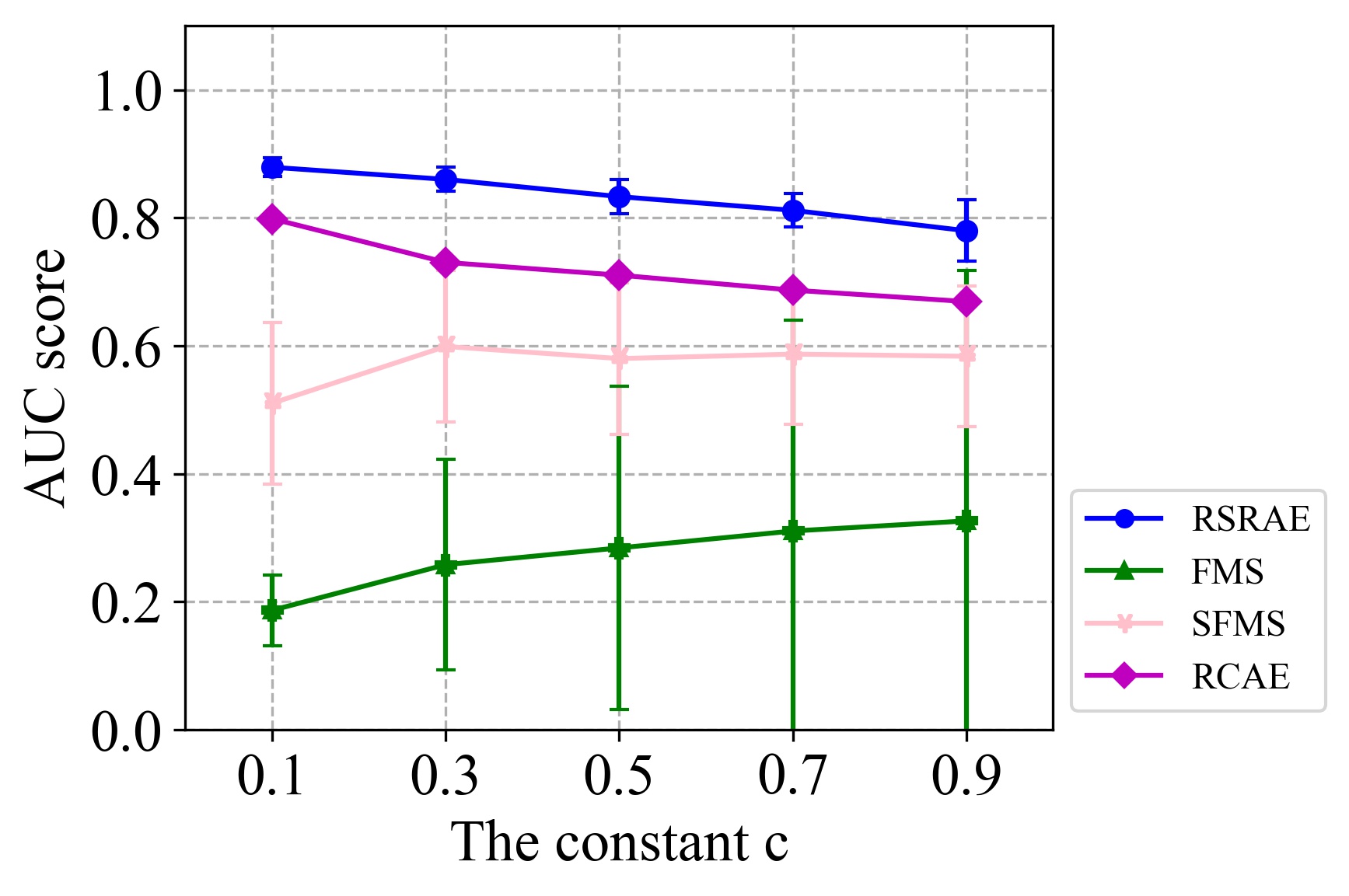}
\end{minipage}
\begin{minipage}[t]{0.48\textwidth}
\centering
\includegraphics[width=6cm]{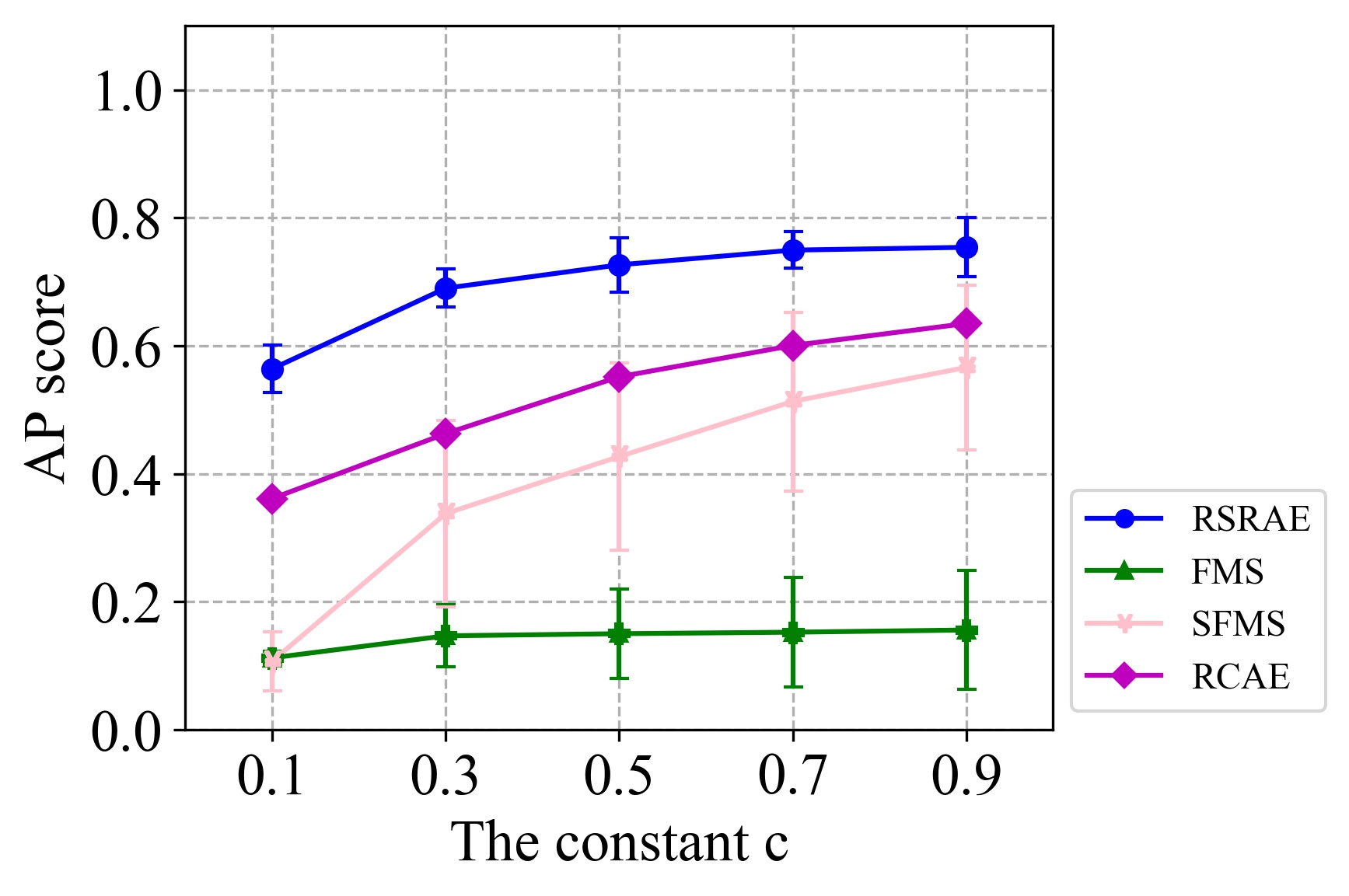}
\end{minipage}

\centering
\begin{minipage}[t]{0.48\textwidth}
\rotatebox{90}{\null \qquad Tiny Imagenet}
\centering
\includegraphics[width=6cm]{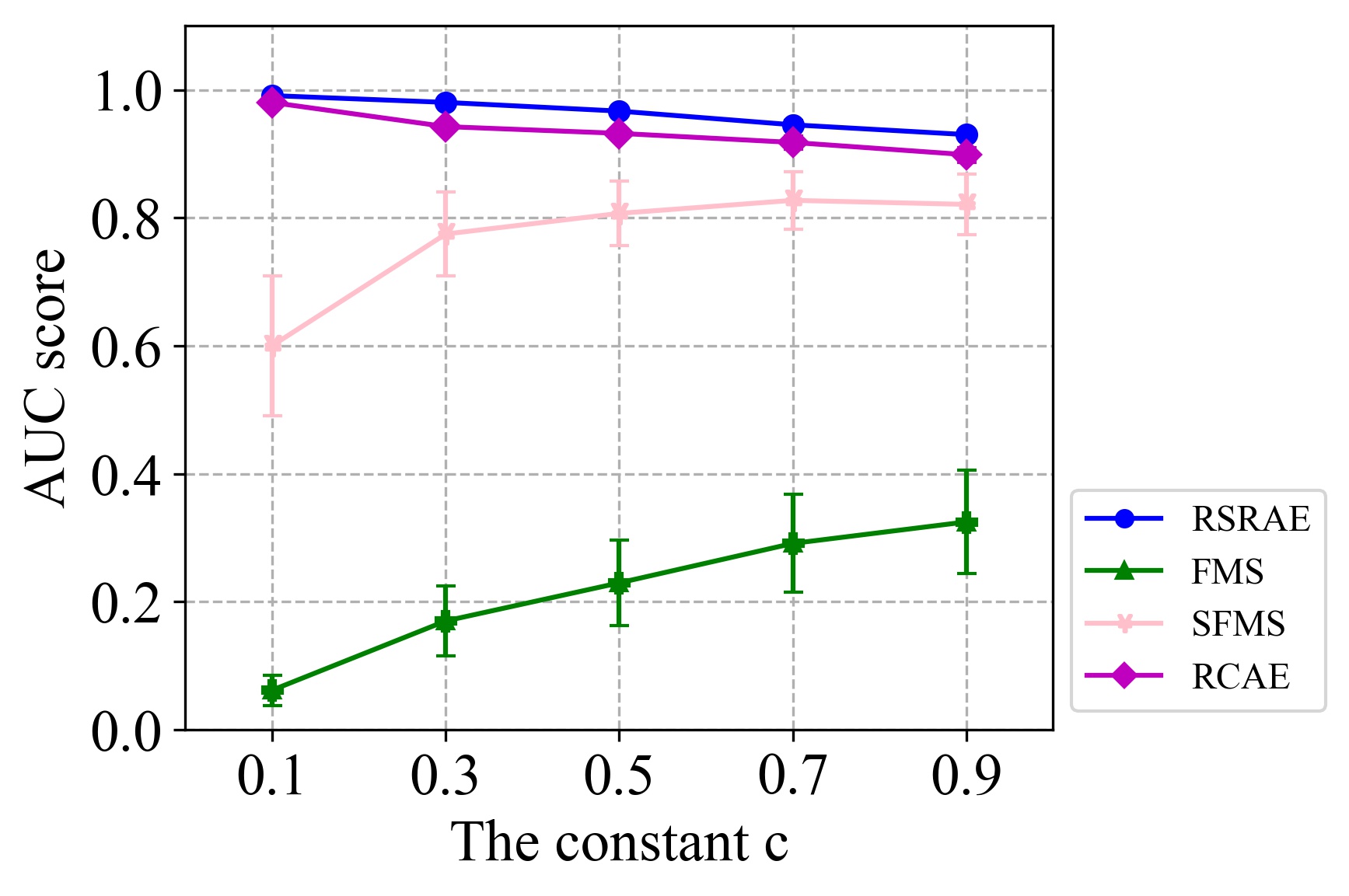}
\end{minipage}
\begin{minipage}[t]{0.48\textwidth}
\centering
\includegraphics[width=6cm]{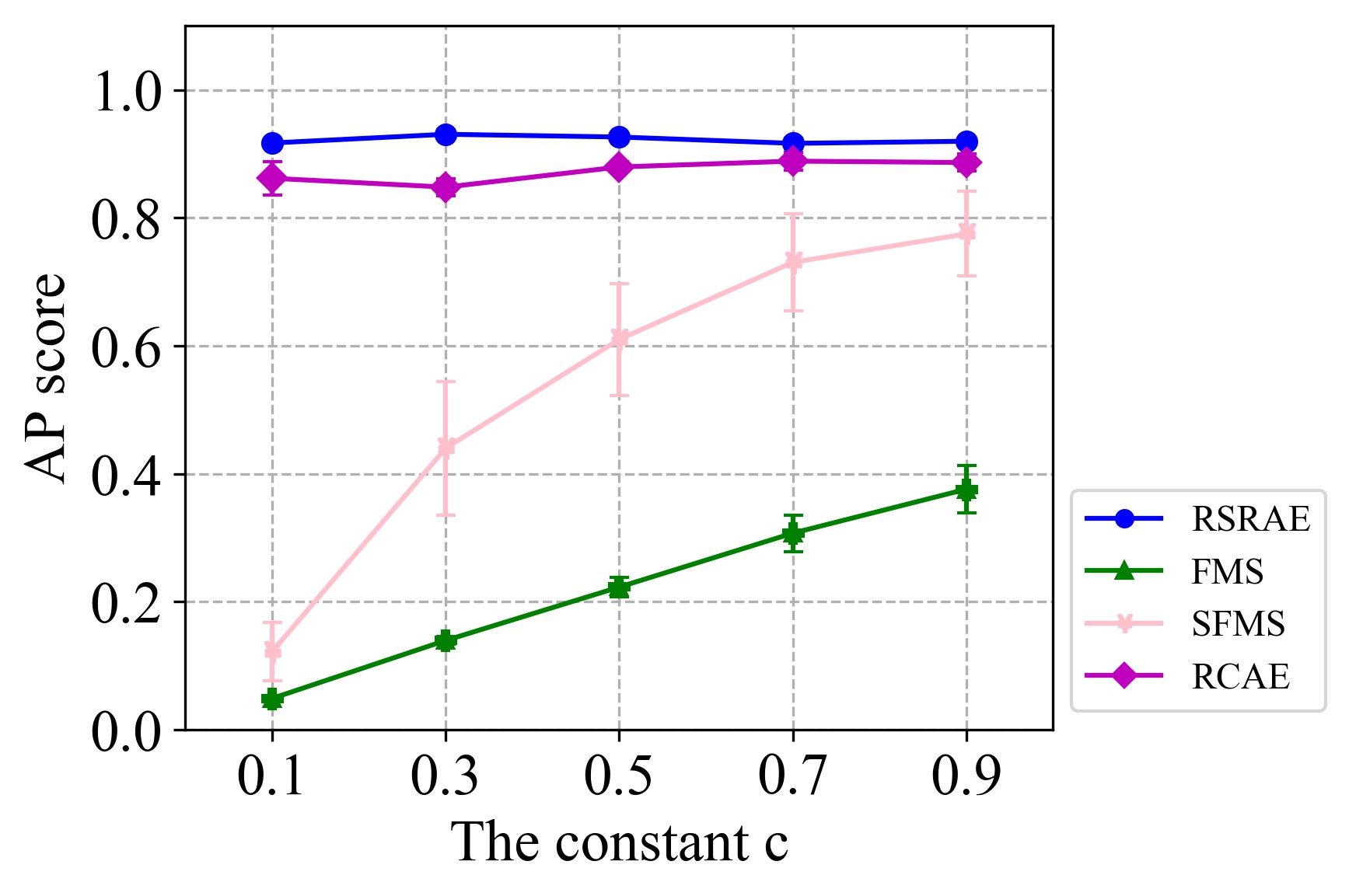}
\end{minipage}

\centering
\begin{minipage}[t]{0.48\textwidth}
\rotatebox{90}{\null \qquad Reuters-21578}
\centering
\includegraphics[width=6cm]{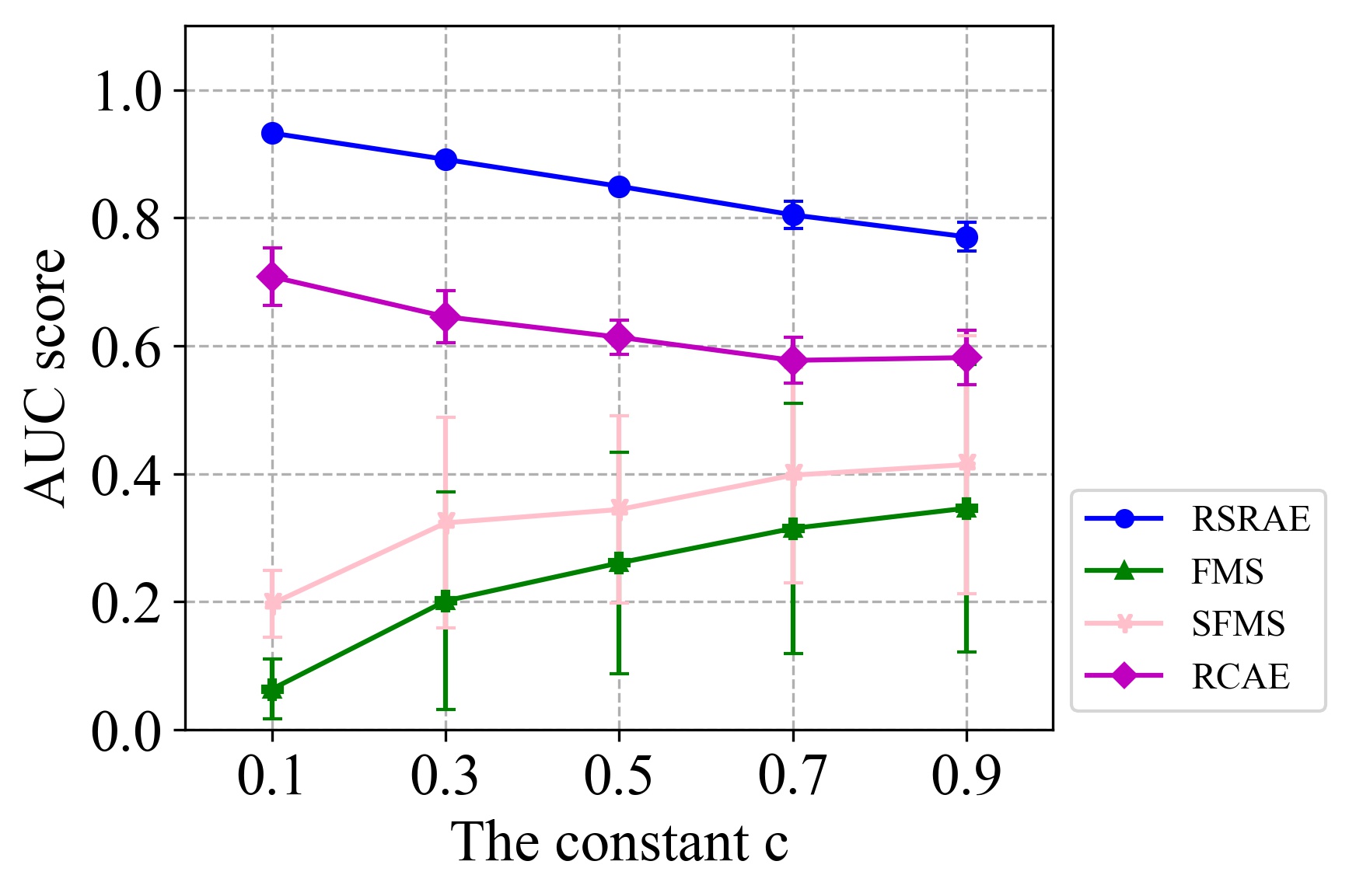}
\end{minipage}
\begin{minipage}[t]{0.48\textwidth}
\centering
\includegraphics[width=6cm]{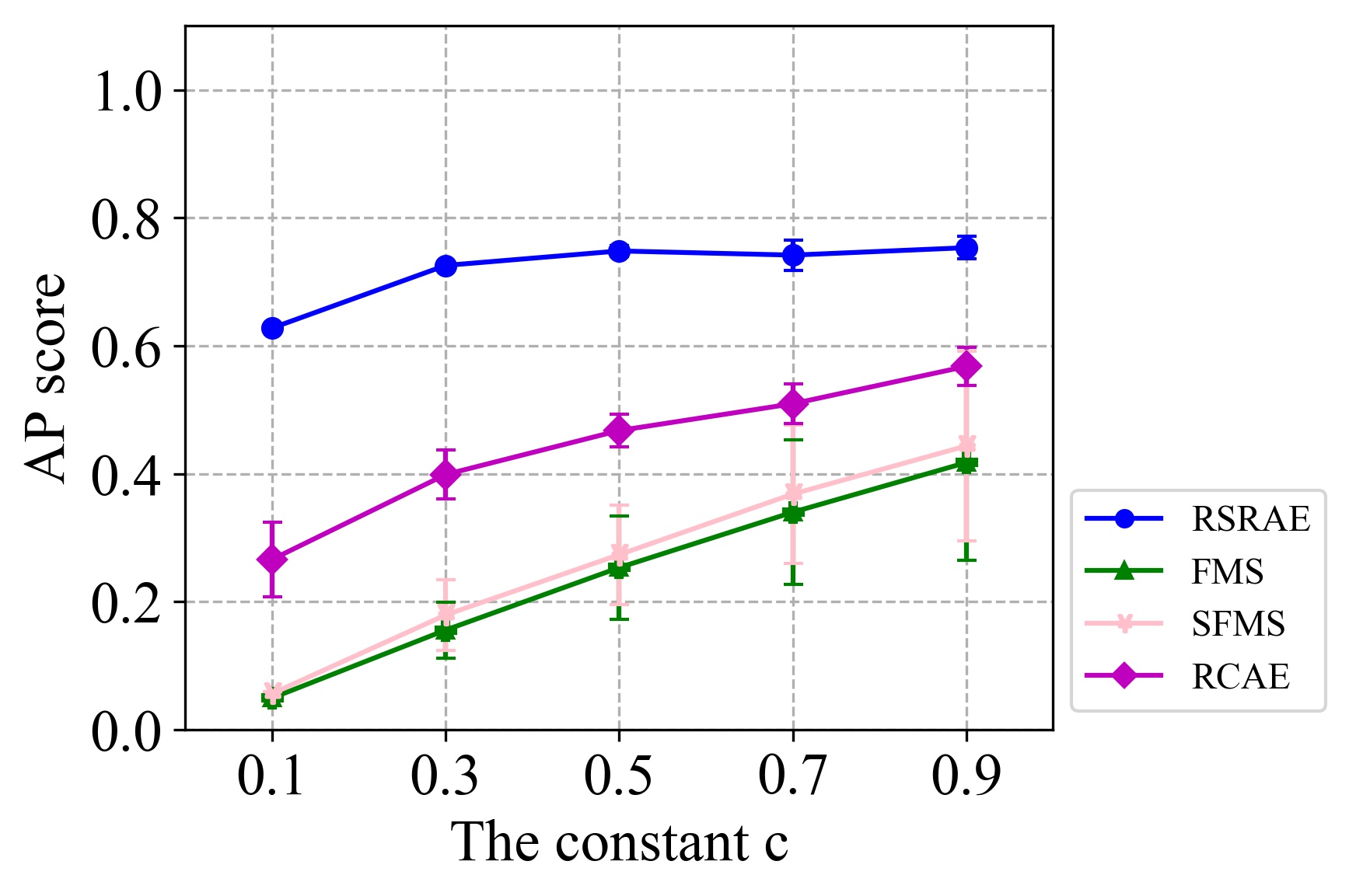}
\end{minipage}

\centering
\begin{minipage}[t]{0.48\textwidth}
\rotatebox{90}{\null \qquad 20 Newsgroups}
\centering
\includegraphics[width=6cm]{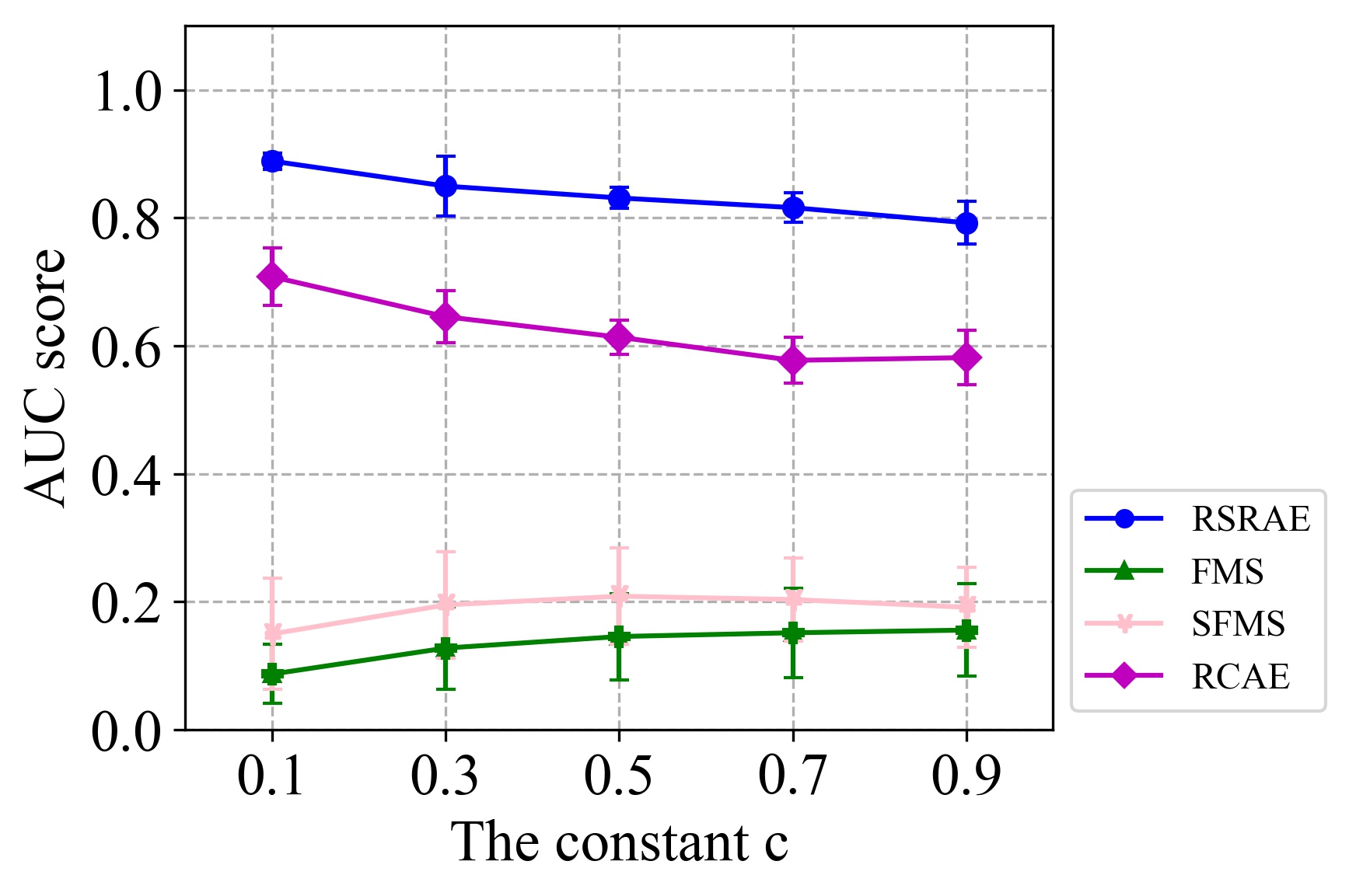}
\end{minipage}
\begin{minipage}[t]{0.48\textwidth}
\centering
\includegraphics[width=6cm]{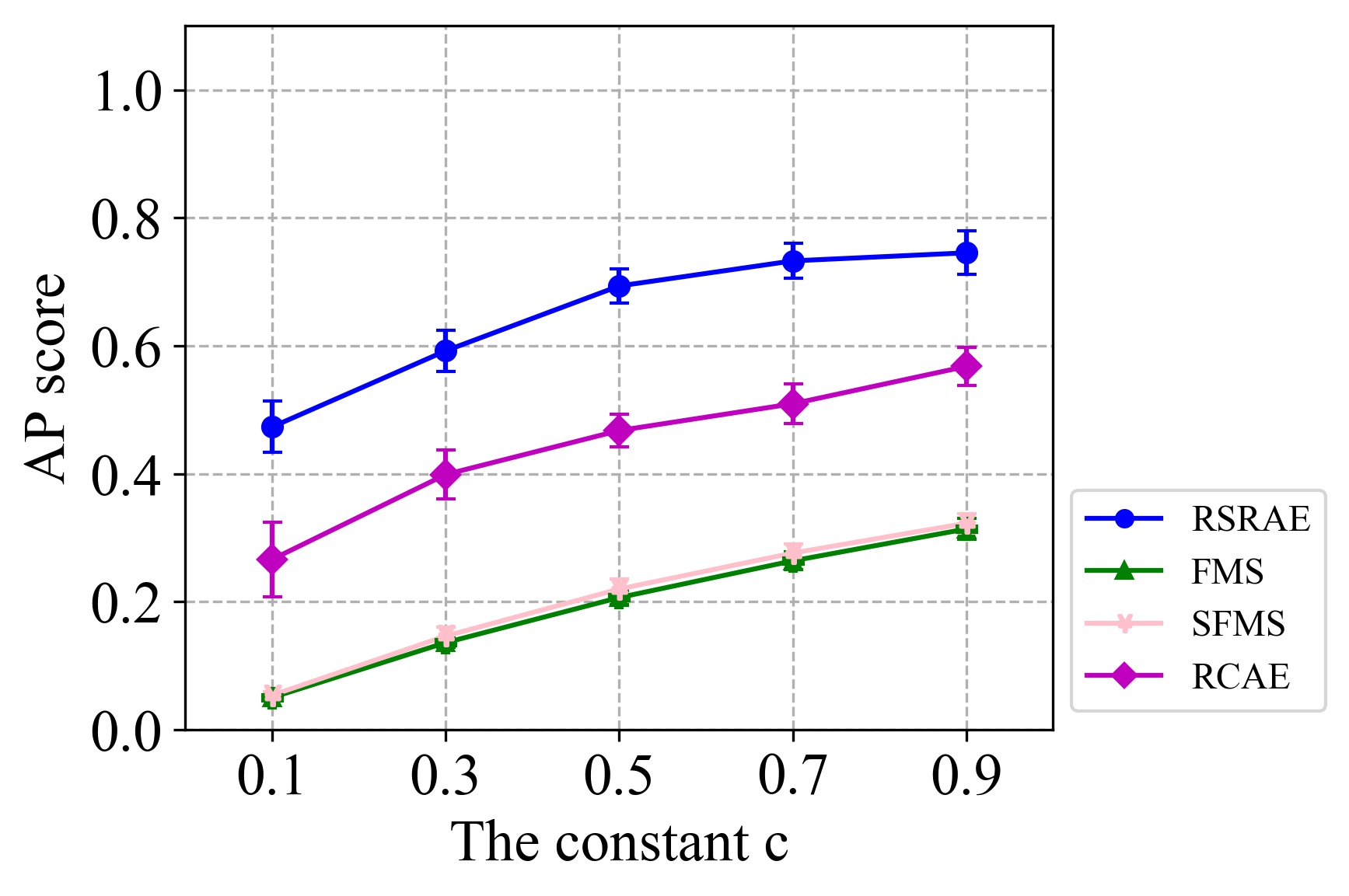}
\end{minipage}

\caption{AUC and AP scores for RSRAE, FMS, SFMS and RCAE. From top to bottom are the results using Caltech 101, Fashion MNIST, Tiny Imagenet with deep features, Reuters-21578 and 20 Newsgroups.}
\label{fig:RSRandRCAE}

\end{figure}

\newpage
\section{Sensitivity to hyperparameters}
\label{sec:sensitivity}

We examine the sensitivity of some of the reported results to changes in the hyperparameters. Section \ref{subsec:intrinsicdim} tests the sensitivity of RSRAE to changes in the intrinsic dimension $d$.
Section \ref{subsec:learningrate} tests the sensitivity of RSRAE to changes in the learning rate. 
Section \ref{subsec:hyperparameters} tests the sensitivity of RSRAE+ to changes in $\lambda_1$ and $\lambda_2$.

\subsection{Sensitivity to the intrinsic dimension}
\label{subsec:intrinsicdim}

In the experiments reported in \Secref{sec:real} we fixed $d=10$. Here we check the sensitivity of the reported results to changes in $d$.
We use the same datasets of \Secref{subsec:res} with an outlier ratio of $c=0.5$ and test the following values of $d$: $1,2,5,8,10,12,15,20,30,40,50$.  Fig.~\ref{fig:intrinsicdim} reports the AUC and AP scores for these choice of $d$ and for these datasets with $c=0.5$. 
We note that, in general, our results are not  sensitive to choices of $d \leq 30$.

We believe that the structure of these datasets is complex, and is not represented by a smooth manifold of a fixed dimension. Therefore, low-dimensional encoding of the inliers is beneficial with various choices of low dimensions. 

When $d$ gets closer to $D$ the performance deteriorates. Such a decrease in accuracy is noticeable for Reuters-21578 and 20 Newsgroups, where for both datasets $D=128$. For the image data sets (without deep features) $D = 1152$ and thus only relatively small values of $d$ were tested. As an example of large $d$ for an image dataset, we consider the case of $d=D=1152$ in Caltech101 with $c=0.5$. In this case, AUC = 0.619 and AP = 0.512, which are very low scores.

We conclude that in our experiments (with $c=0.5$), RSRAE was stable in $d$ around our choice of $d=10$.

\begin{figure}[htbp]
\centering
\begin{minipage}[t]{0.48\textwidth}
\rotatebox{90}{\null \qquad Caltech 101}
\centering
\includegraphics[width=6cm]{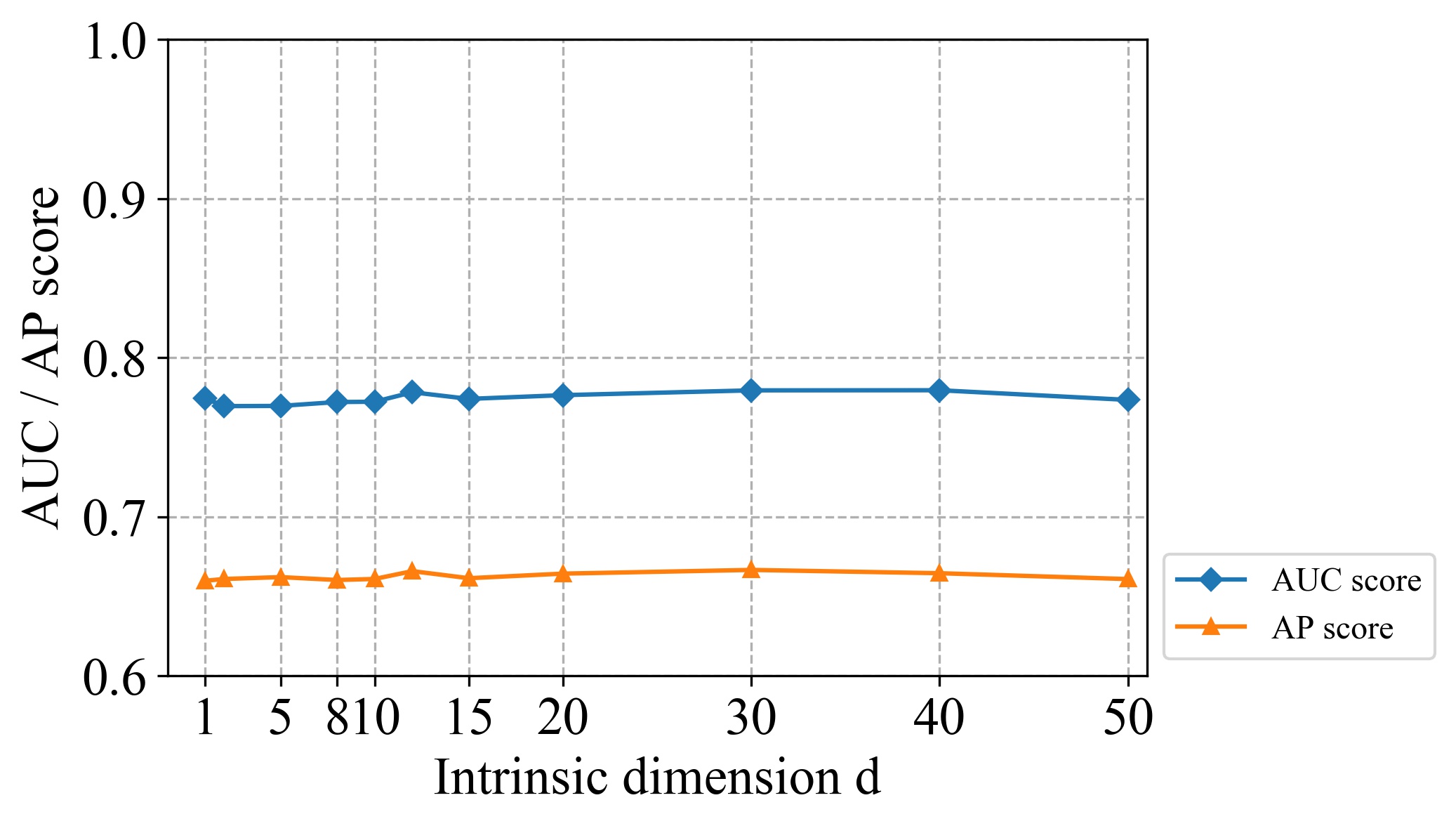}
\end{minipage}
\begin{minipage}[t]{0.48\textwidth}
\rotatebox{90}{\null \qquad Fashion MNIST}
\centering
\includegraphics[width=6cm]{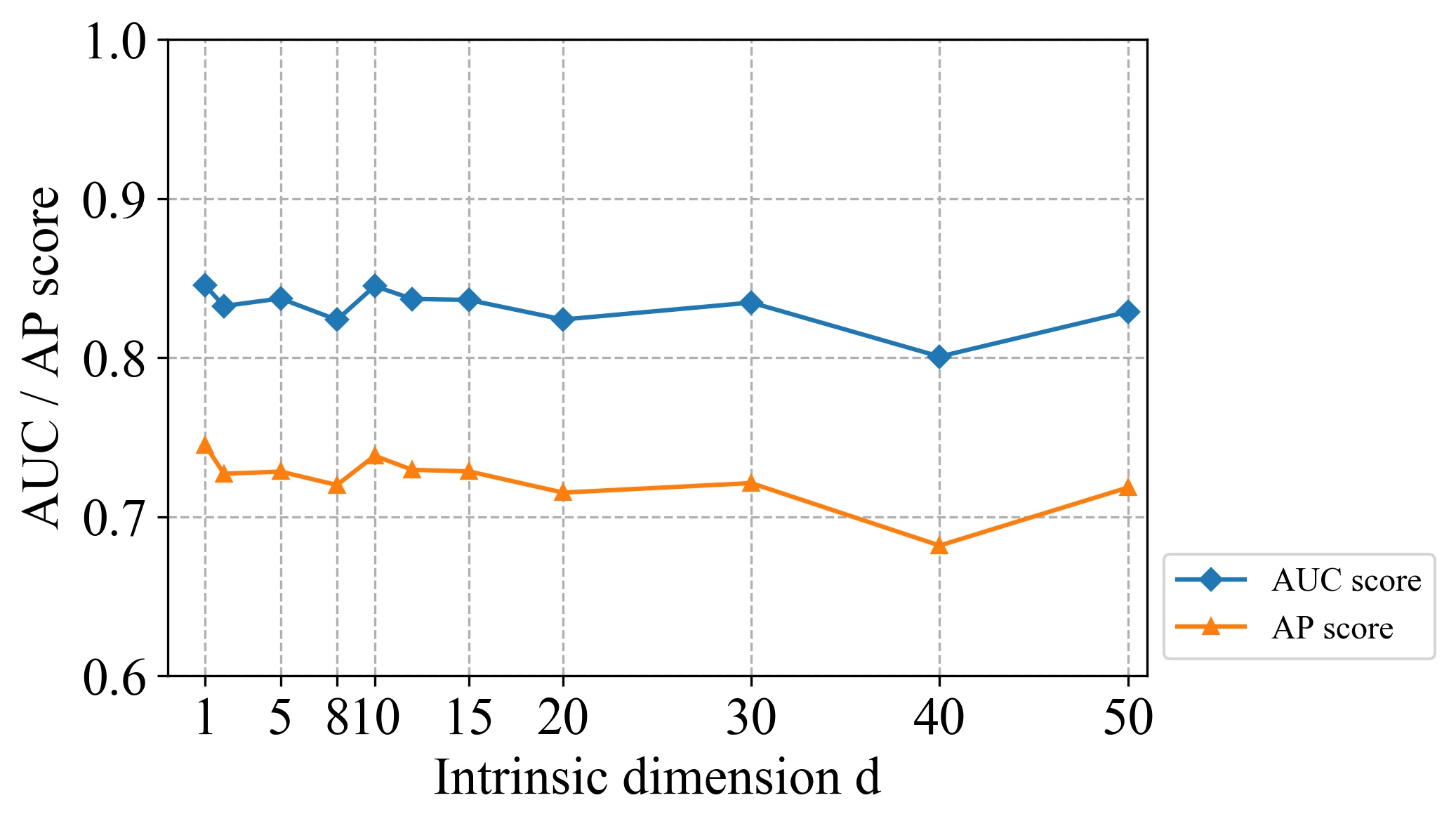}
\end{minipage}

\centering
\begin{minipage}[t]{0.48\textwidth}
\rotatebox{90}{\null \qquad Tiny Imagenet}
\centering
\includegraphics[width=6cm]{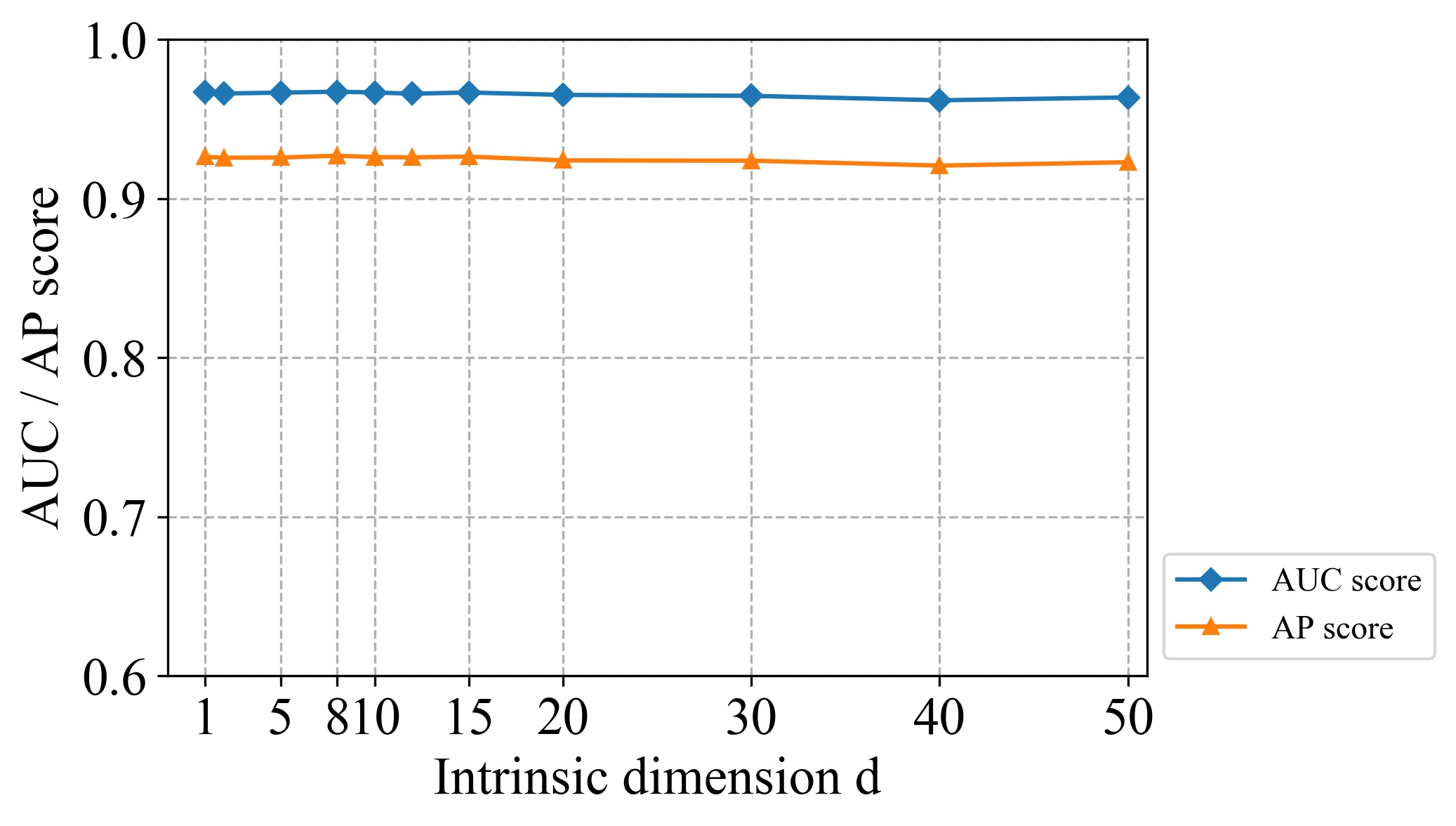}

\end{minipage}
\begin{minipage}[t]{0.48\textwidth}
\rotatebox{90}{\null \qquad Reuters-21578}
\centering
\includegraphics[width=6cm]{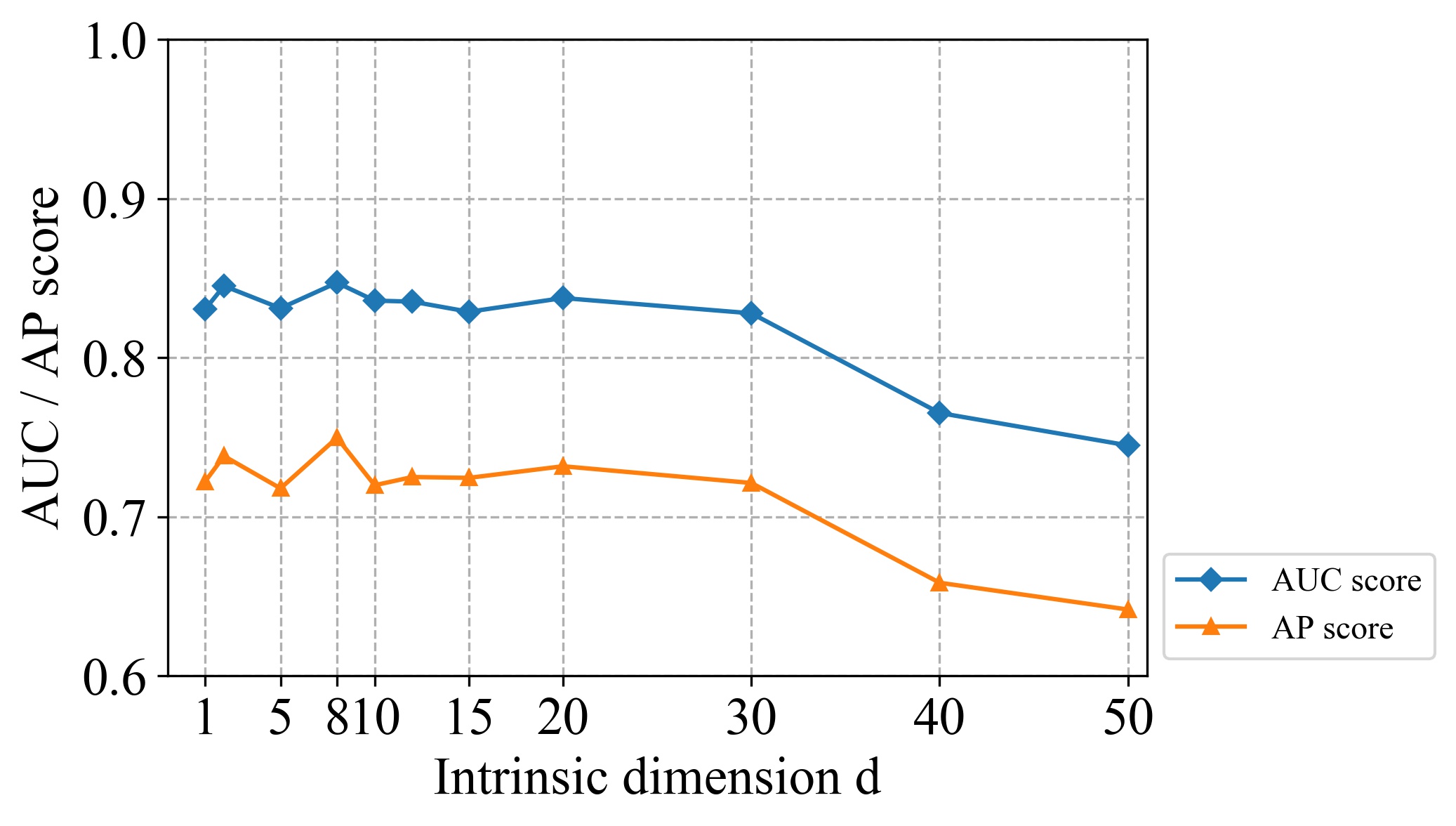}
\end{minipage}

\centering
\begin{minipage}[t]{0.48\textwidth}
\rotatebox{90}{\null \qquad 20 Newsgroups}
\centering
\includegraphics[width=6cm]{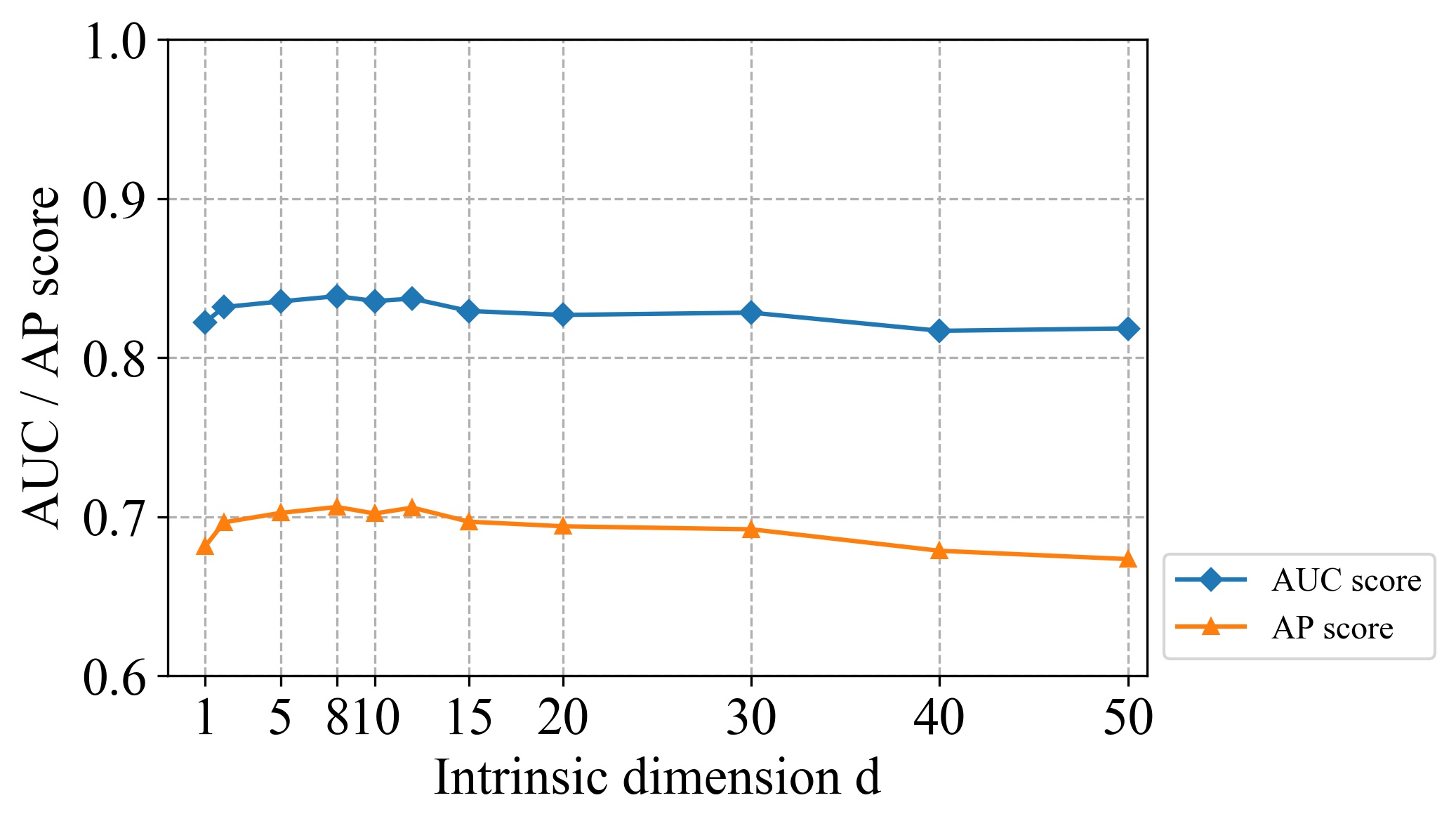}

\end{minipage}

\caption{AUC and AP scores for different choices of $d$. The datasets are the same as those in \Secref{subsec:res}, where the outlier ratio is $c=0.5$.
\label{fig:intrinsicdim}
}
\end{figure}

\newpage
\subsection{Sensitivity to the learning rate}
\label{subsec:learningrate}

In the experiments reported in \Secref{sec:real} we fixed the learning rate for RSRAE to be $0.00025$. Here we check the sensitivity of the reported results to changes in the learning rate.
We use the same datasets of \Secref{subsec:res} with an outlier ratio of $c=0.5$ and test the following values of the learning rate: $0.0001, 0.00025, 0.0005, 0.001, 0.0025, 0.005, 0.01, 0.025, 0.05, 0.1$.  Fig.~\ref{fig:learningrate} reports the AUC and AP scores for these values and for these datasets (with $c=0.5$). 
We note that the performance is stable for learning rates not exceeding 0.01.

\begin{figure}[htbp]
\centering
\begin{minipage}[t]{0.48\textwidth}
\rotatebox{90}{\null \qquad Caltech 101}
\centering
\includegraphics[width=6cm]{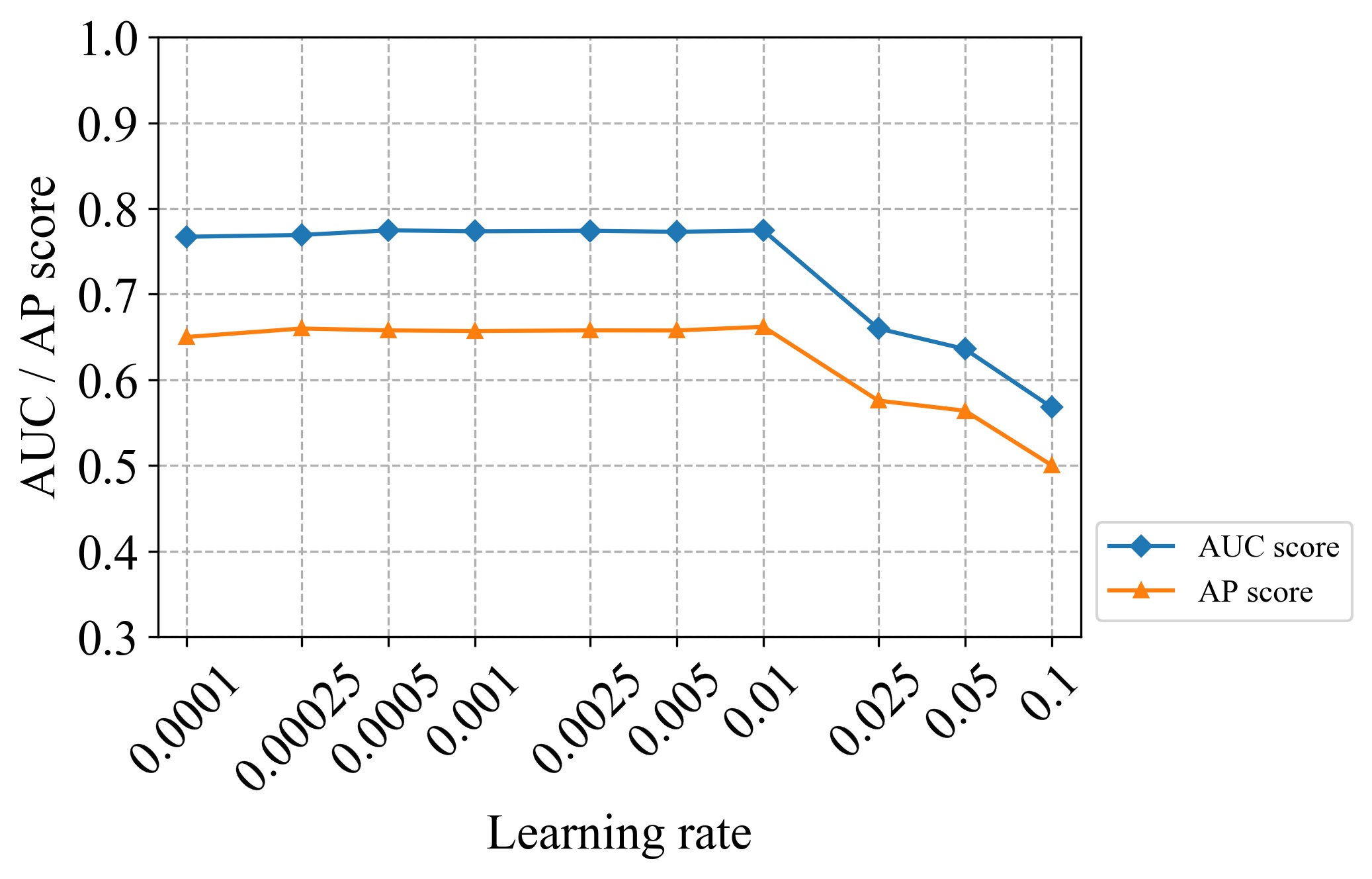}
\end{minipage}
\begin{minipage}[t]{0.48\textwidth}
\rotatebox{90}{\null \qquad Fashion MNIST}
\centering
\includegraphics[width=6cm]{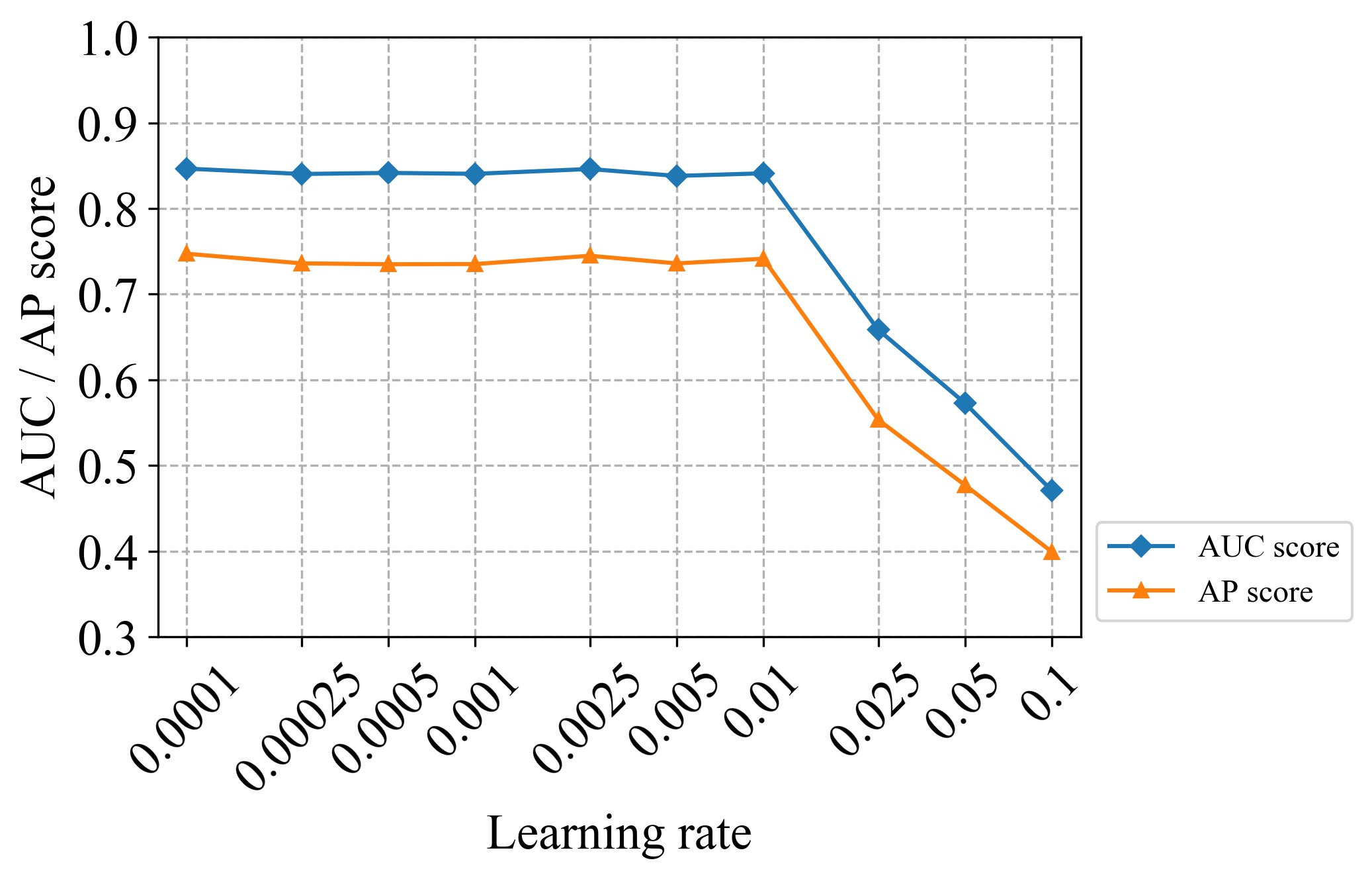}
\end{minipage}

\centering
\begin{minipage}[t]{0.48\textwidth}

\rotatebox{90}{\null \qquad Tiny Imagenet}
\centering
\includegraphics[width=6cm]{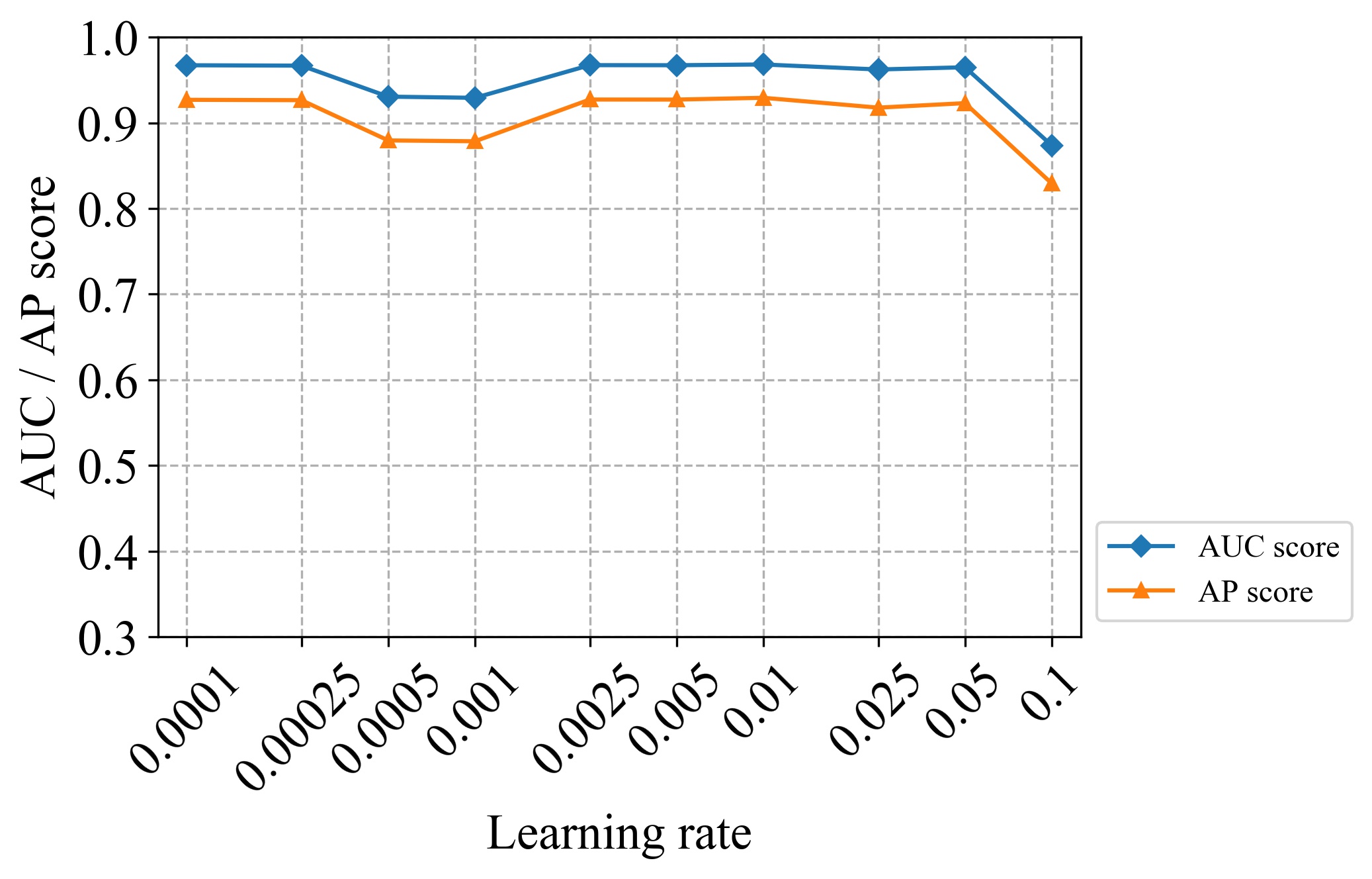}
\end{minipage}
\begin{minipage}[t]{0.48\textwidth}
\rotatebox{90}{\null \qquad Reuters-21578}
\centering
\includegraphics[width=6cm]{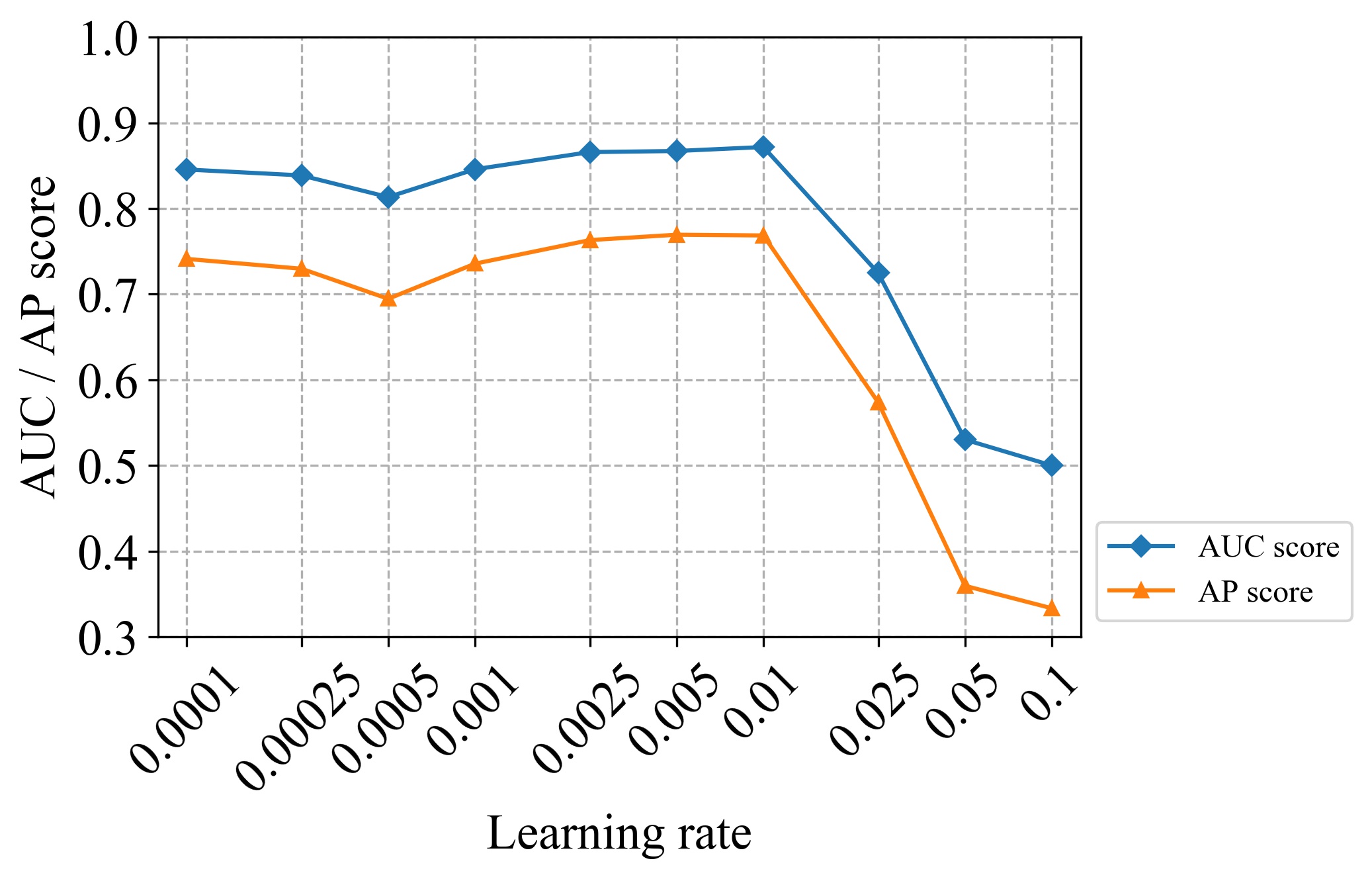}
\end{minipage}

\centering
\begin{minipage}[t]{0.48\textwidth}
\rotatebox{90}{\null \qquad 20 Newsgroups}
\centering
\includegraphics[width=6cm]{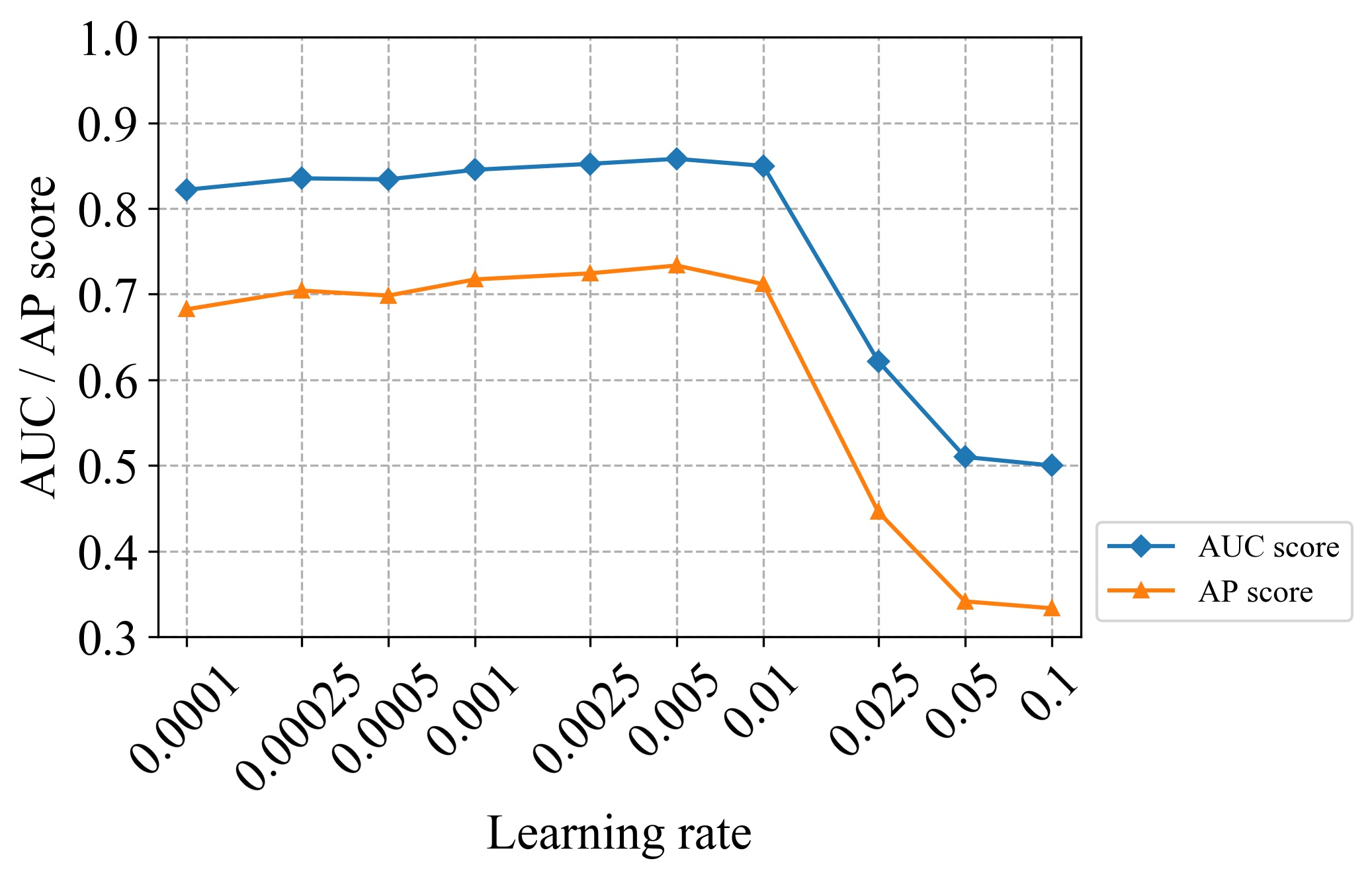}
\end{minipage}

\caption{AUC and AP scores for various learning rates.  The datasets are the same as those in \Secref{subsec:res}, where the outlier ratio is $c=0.5$.}
\label{fig:learningrate}

\end{figure}

\subsection{Sensitivity of RSRAE+ to $\lambda_1$ and $\lambda_2$}
\label{subsec:hyperparameters}

We study the sensitivity of RSRAE+ to different choices of $\lambda_1$ and $\lambda_2$. We recall that RSRAE does not require these parameters. It is still interesting to check such sensitivity and find out whether careful tuning of these parameters in RSRAE+ can yield better scores than those of RSRAE. We use the same datasets of \Secref{subsec:res} with an outlier ratio of $c=0.5$ and simultaneously test the following values of either $\lambda_1$ or  $\lambda_2$: $0.01, 0.02, 0.05, 0.1, 0.2, 0.5, 1.0, 2.0$. Figs.~\ref{fig:hyperparameters} and \ref{fig:hyperparameters2} report the AUC and AP scores for these values and datasets (with $c=0.5$). 
For each subfigure, the above values of $\lambda_1$ and $\lambda_2$ are recorded on the $x$ and $y$ axes, respectively. The darker colors of the heat map correspond to larger scores. For comparison, the corresponding AUC or AP score of RSRAE is indicated in the title of each subfigure. 

We note that RSRAE+ is more sensitive to $\lambda_1$ than $\lambda_2$. Furthermore, as $\lambda_1$ increases the scores are often more stable to changes in $\lambda_1$. That is, the magnitudes of the derivatives of the scores with respect to $\lambda_1$ seem to generally decrease with $\lambda_1$. In \Secref{subsec:cprnorm} we used 
$\lambda_1 = \lambda_2 = 0.1$ as this choice seemed optimal for the independent set of 20 Newsgroup. We note though that optimal hyperparameters depend on the dataset and it is thus not a good idea to optimize them using different datasets. They also depend on the choice of $c$, but for brevity we only test them with $c=0.5$.

At last we note that the AUC and AP scores of RSRAE are comparable to the fine-tuned ones of RSRAE+ (where $c=0.5$). We thus advocate using the alternating minimization of RSRAE, which is independent of $\lambda_1$ and $\lambda_2$.

\begin{figure}[htbp]
\centering
\begin{minipage}[t]{0.48\textwidth}
\rotatebox{90}{\null \qquad Caltech 101}
\centering
\includegraphics[width=6cm]{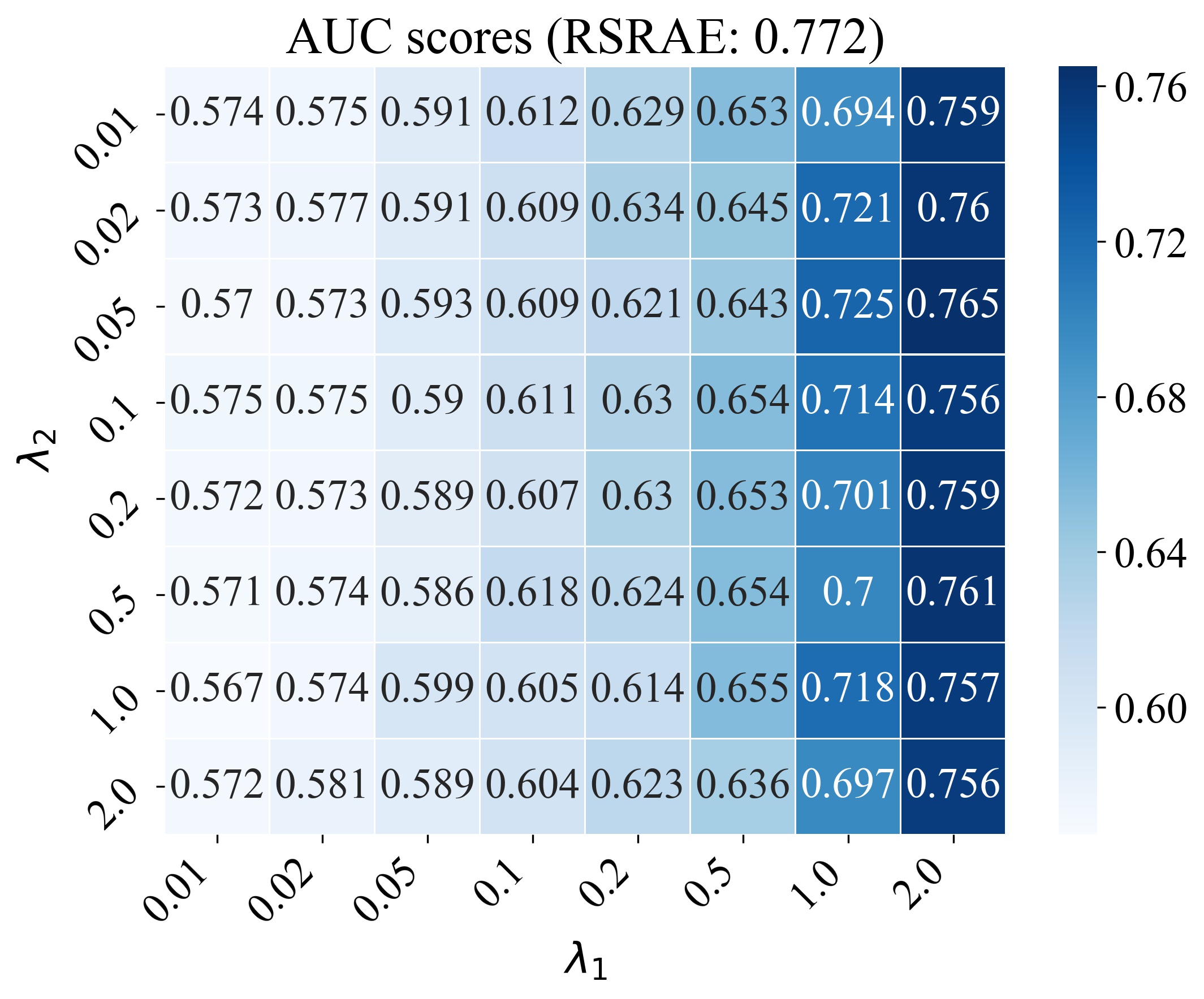}
\end{minipage}
\begin{minipage}[t]{0.48\textwidth}
\centering
\includegraphics[width=6cm]{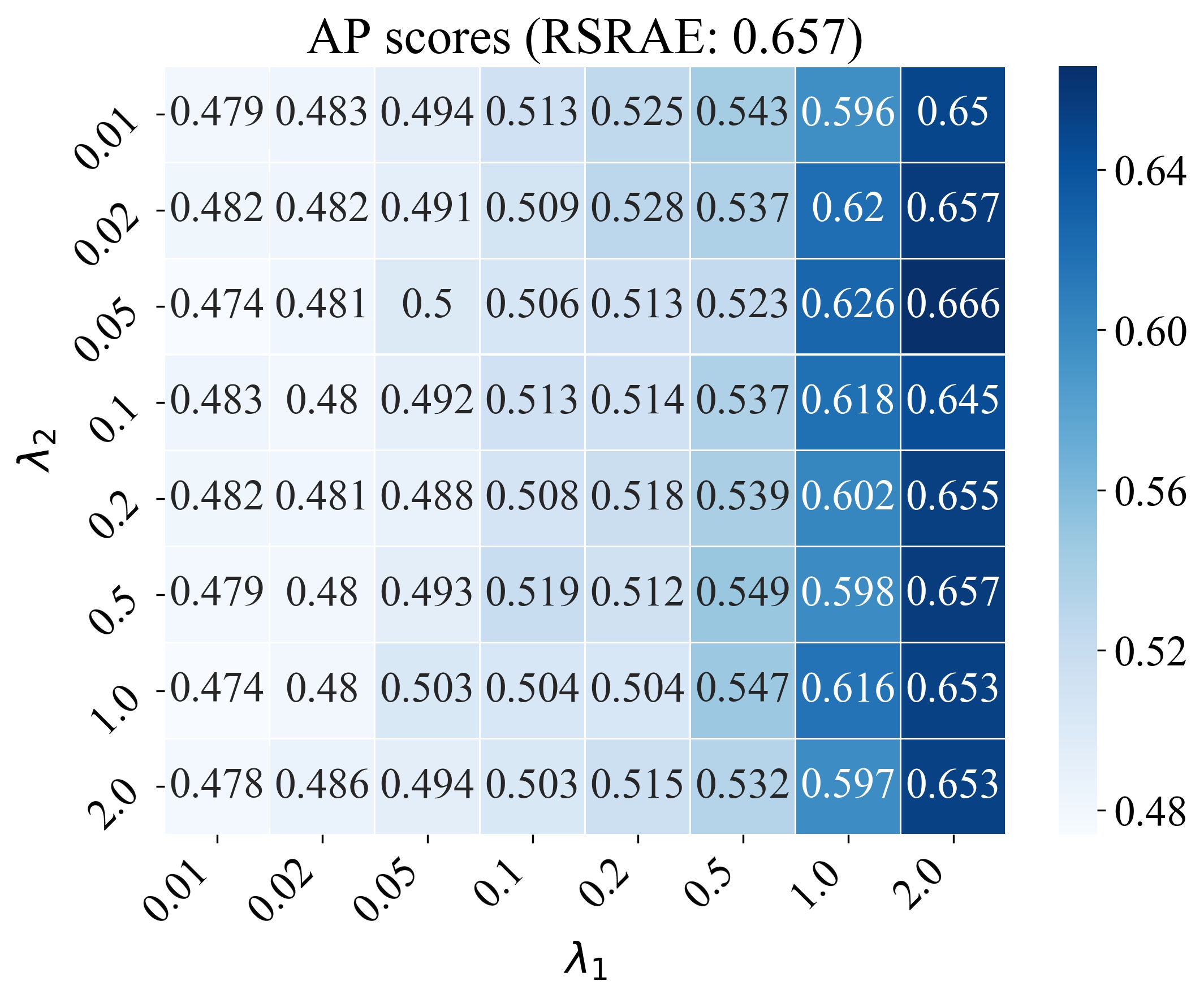}
\end{minipage}

\centering
\begin{minipage}[t]{0.48\textwidth}
\rotatebox{90}{\null \qquad Fashion MNIST}
\centering
\includegraphics[width=6cm]{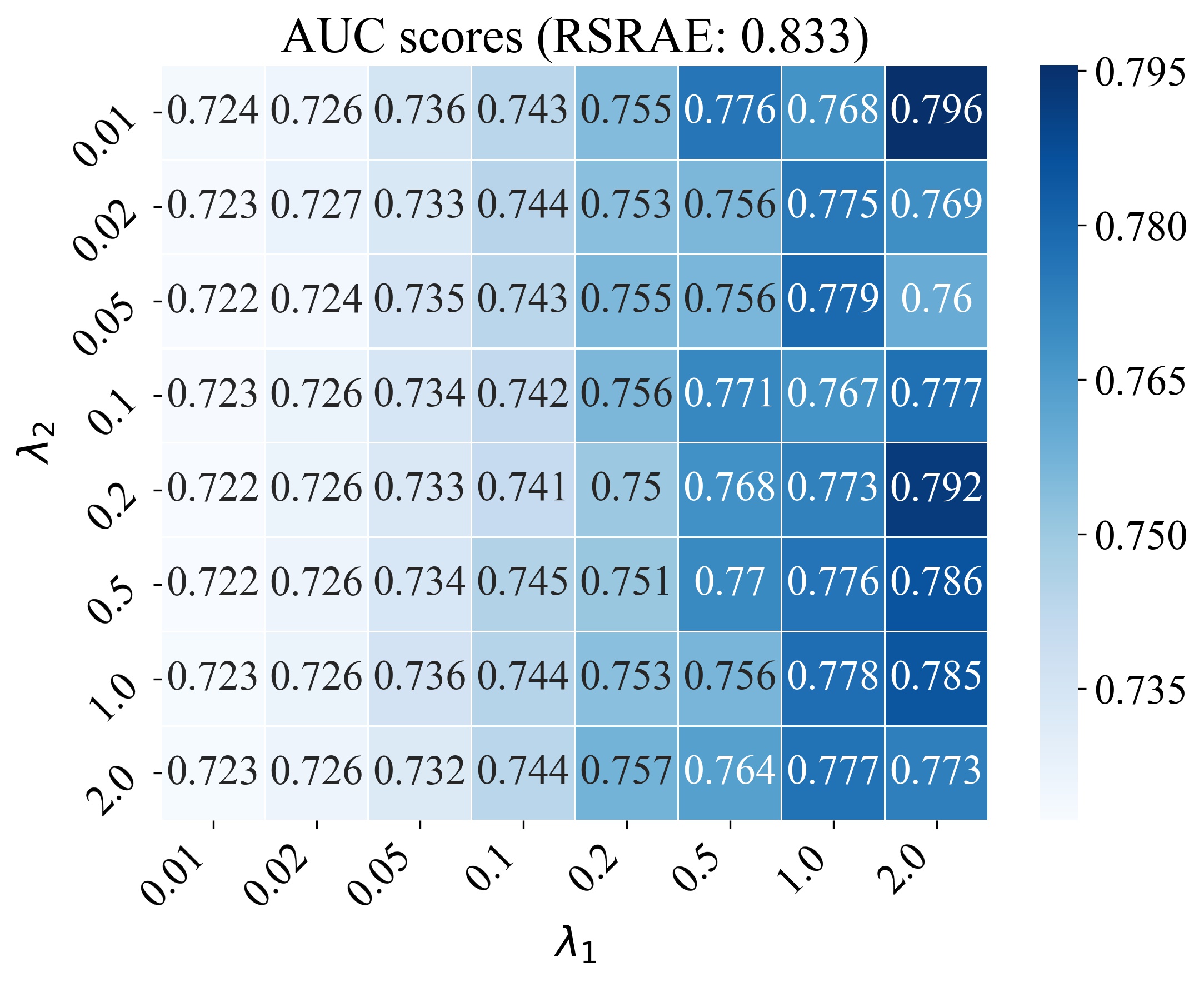}
\end{minipage}
\begin{minipage}[t]{0.48\textwidth}
\centering
\includegraphics[width=6cm]{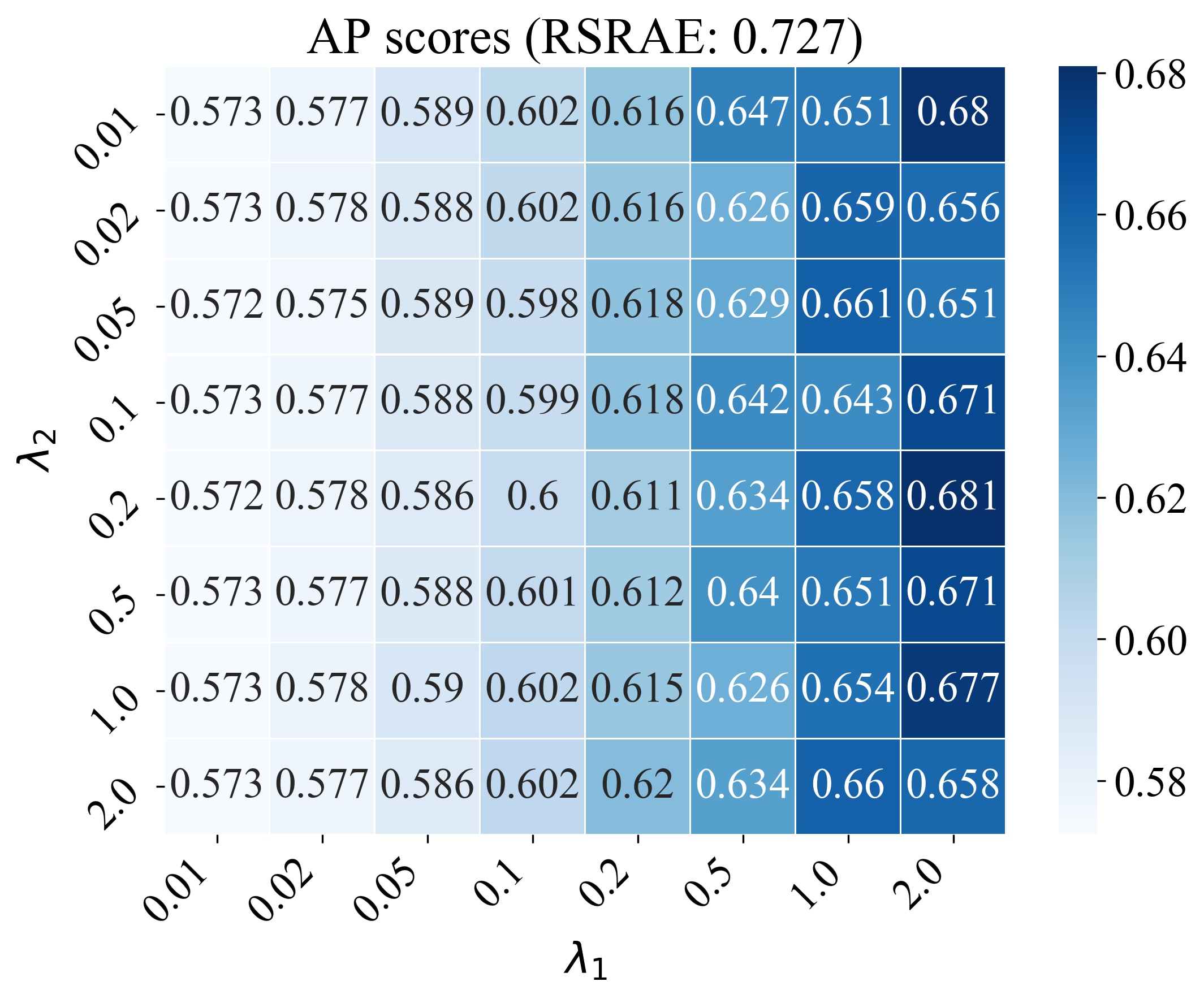}
\end{minipage}

\centering
\begin{minipage}[t]{0.48\textwidth}
\rotatebox{90}{\null \qquad Tiny Imagenet}
\centering
\includegraphics[width=6cm]{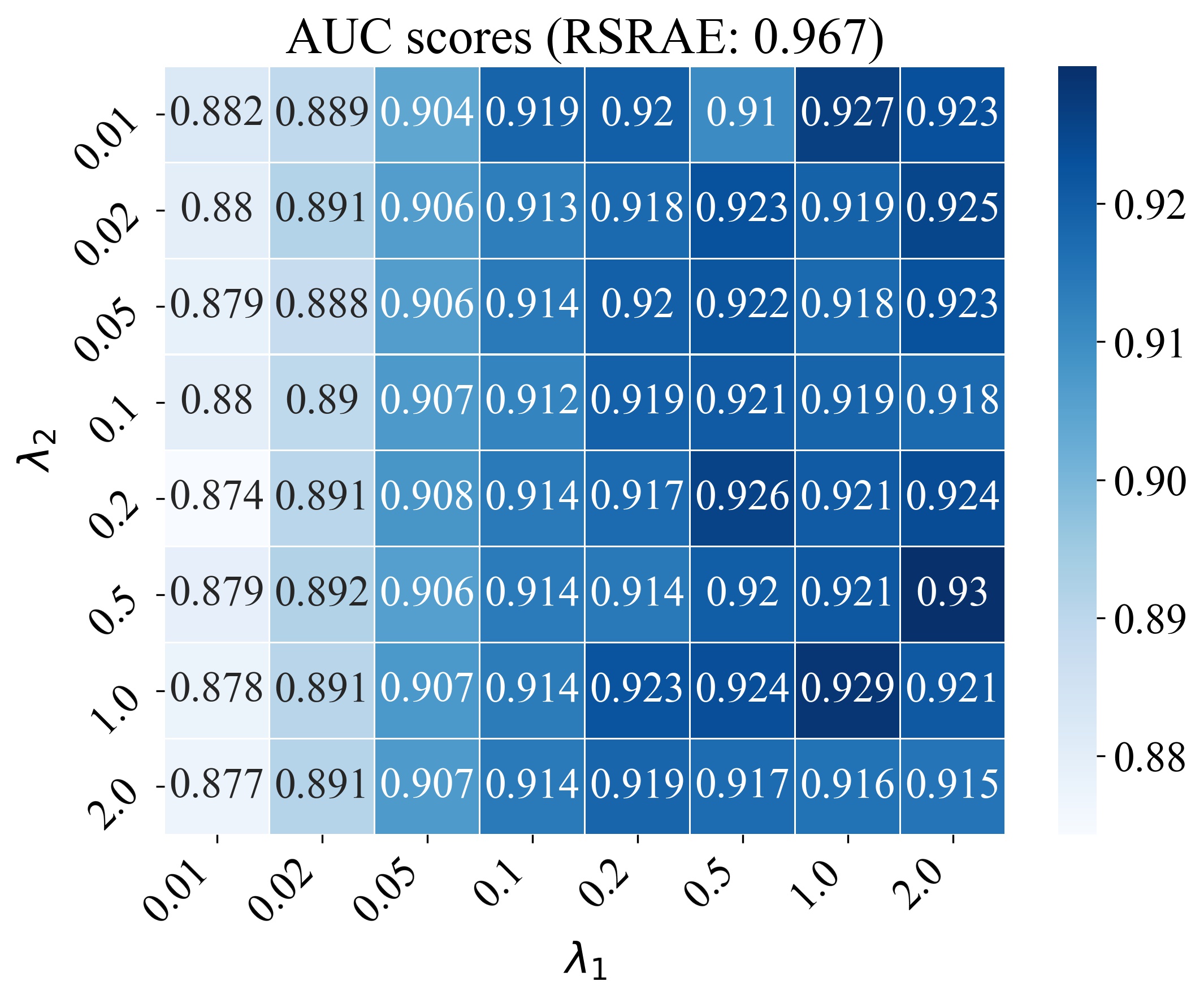}
\end{minipage}
\begin{minipage}[t]{0.48\textwidth}
\centering
\includegraphics[width=6cm]{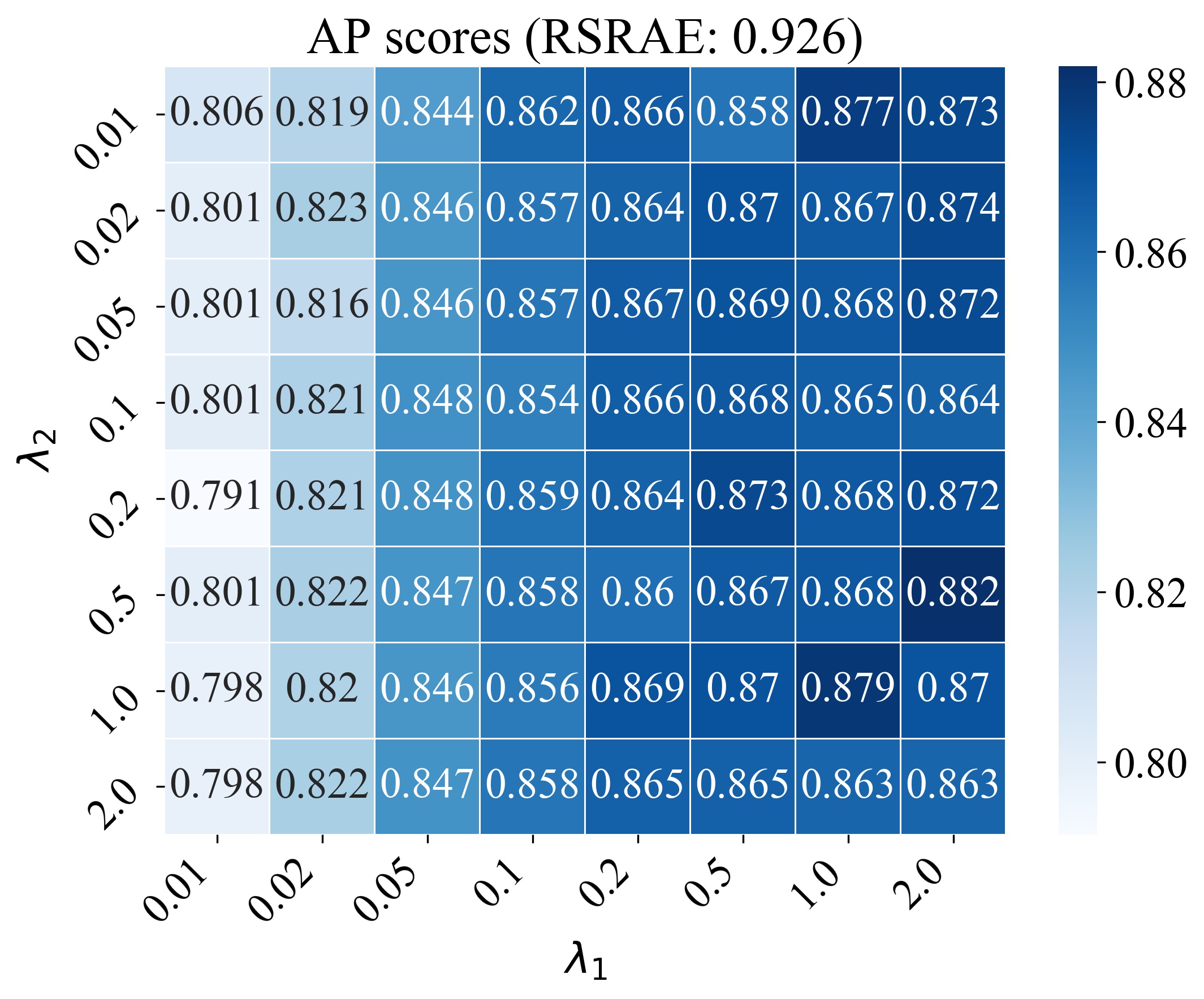}
\end{minipage}

\caption{AUC and AP scores for RSRAE+ with various choices of $\lambda_1$ and $\lambda_2$ for Caltech 101, Fashion MNIST and Tiny Imagenet with deep features, where $c=0.5$.}
\label{fig:hyperparameters}

\end{figure}

\begin{figure}

\centering
\begin{minipage}[t]{0.48\textwidth}
\rotatebox{90}{\null \qquad Reuters-21578}
\centering
\includegraphics[width=6cm]{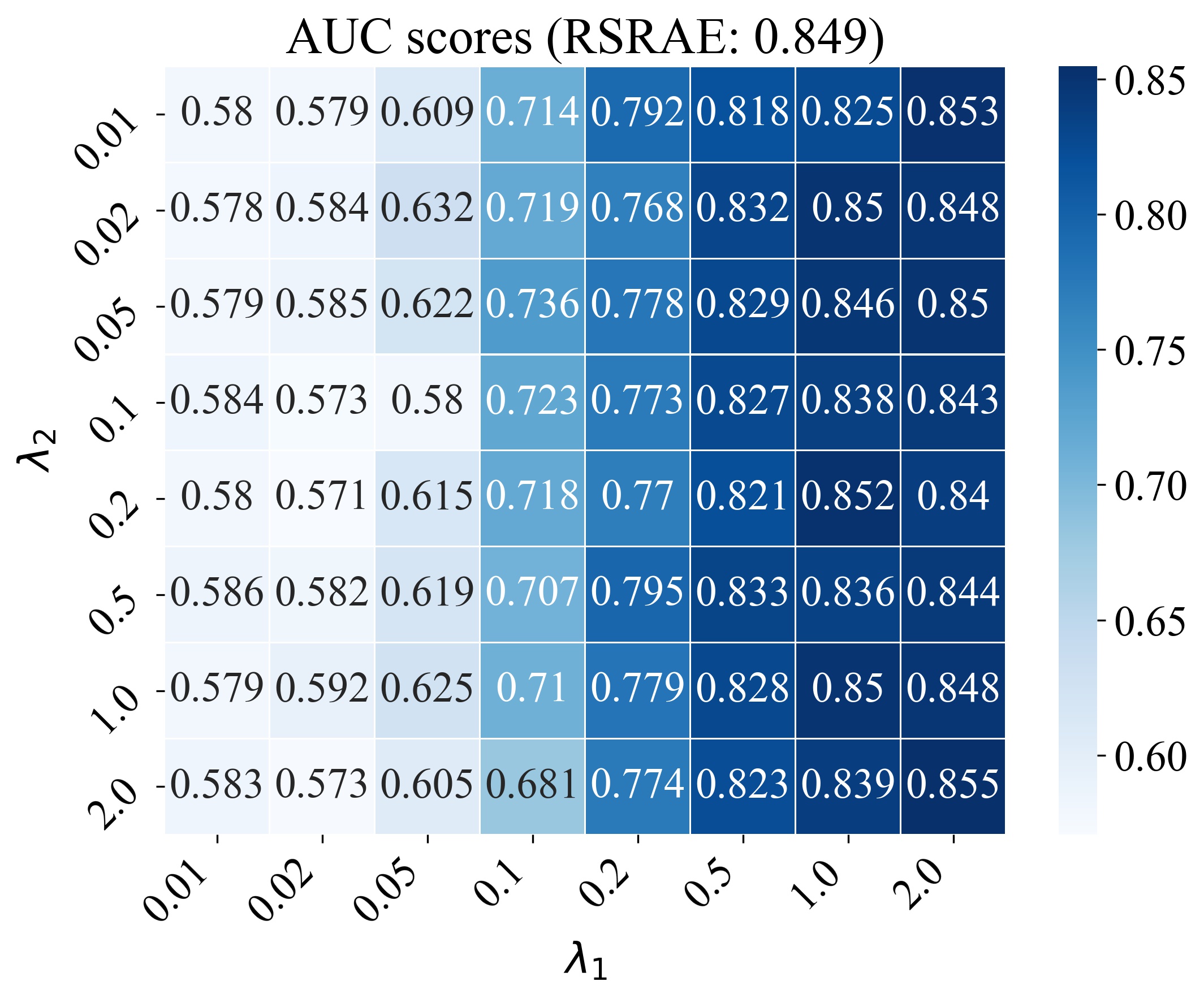}
\end{minipage}
\begin{minipage}[t]{0.48\textwidth}
\centering
\includegraphics[width=6cm]{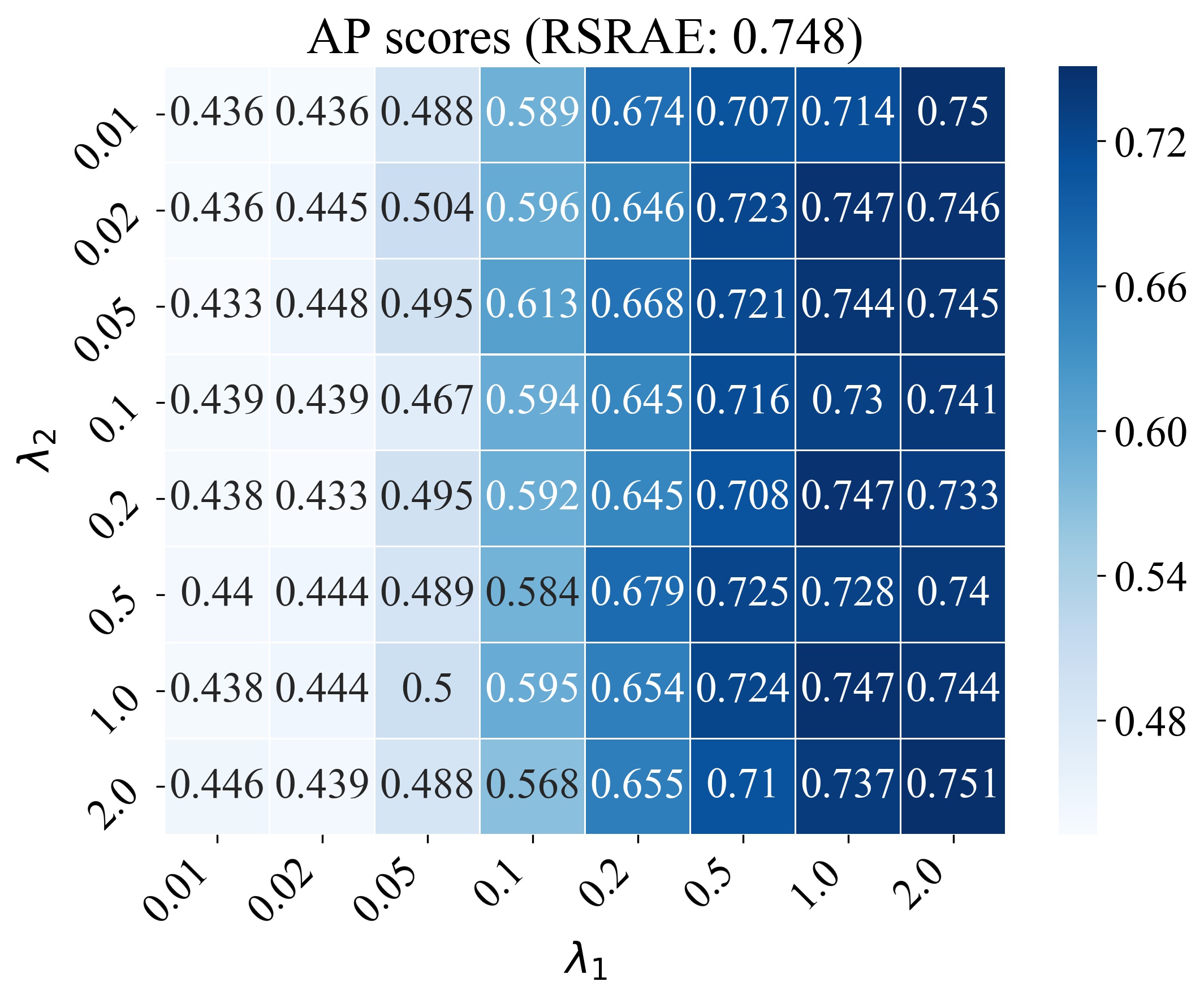}
\end{minipage}

\centering
\begin{minipage}[t]{0.48\textwidth}
\rotatebox{90}{\null \qquad 20 Newsgroups}
\centering
\includegraphics[width=6cm]{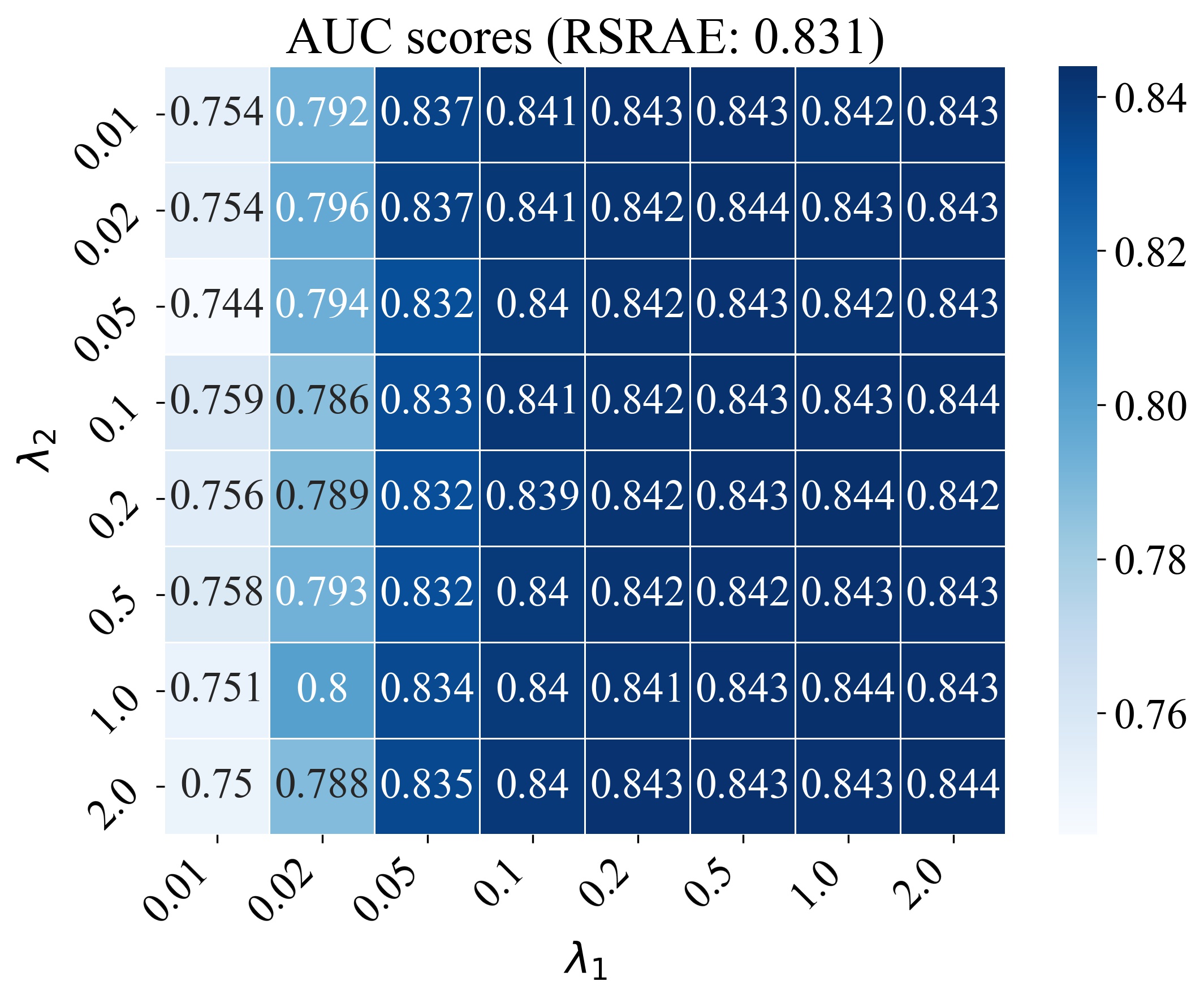}
\end{minipage}
\begin{minipage}[t]{0.48\textwidth}
\centering
\includegraphics[width=6cm]{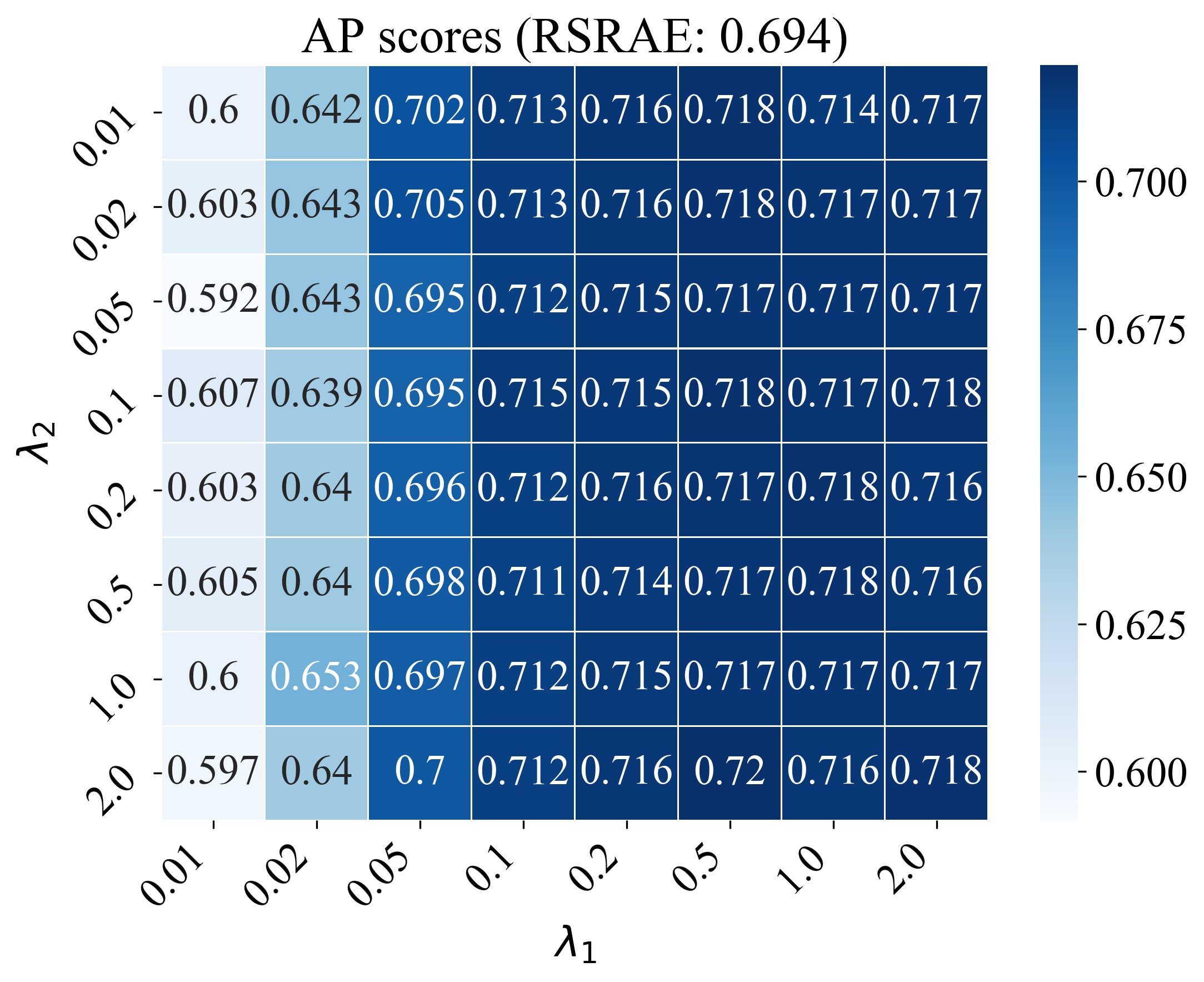}
\end{minipage}

\caption{AUC and AP scores for RSRAE+ with various choices of $\lambda_1$ and $\lambda_2$ using Reuters-21578 and 20 Newsgroup, where $c=0.5$.
}
\label{fig:hyperparameters2}
\end{figure}

\newpage
\section{Runtime comparison}
\label{sec:runtime}
Table~\ref{tab:runtime} records runtimes for all the methods and datasets in \Secref{subsec:res} with the choice of $c=0.5$. More precisely, a runtime is the the time needed to complete a single experiment, where 200 epoches were used for the neural networks. The table averages each runtime over the different classes.

Note that LOF, OCSVM and IF are faster than the rest of methods since they do not require training neural networks. 
We also note that the runtime of RSRAE is competitive in comparison to the other tested methods, that is, 
DSEBMs, DAGMM, and GT. The neural network structures of these four methods are the same, and thus the difference in runtime is mainly due to different pre and post processing.

\begin{table}[th]
\centering
\caption{Runtime comparison: runtimes (in seconds) are reported for all methods and datasets in \Secref{subsec:res}, where the outlier ratio is $c=0.5$.
Since  GT was only applied to the image datasets without deep features, its runtime is not available (N/A) for the last three datasets.
}
\label{tab:runtime}
\resizebox{\textwidth}{!}{\begin{tabular}{|c|c|c|c|c|c|}
\hline
\diagbox{Benchmarks}{Datasets}        & Caltech 101 & Fashion MNIST &Tiny Imagenet & Reuters-21578 & 20 Newsgroups \\ \hline
LOF    & 0.233                & 7.163                  & 0.707                  & 25.342                 & 10.516                 \\ \hline
OCSVM  & 0.120                & 3.151                  & 0.473                  & 8.726                  & 4.169                  \\ \hline
IF     & 0.339                & 1.485                  & 0.511                  & 20.481                 & 6.751                  \\ \hline
GT     & 21.681               & 87.729                 &        N/A                &       N/A                 &           N/A             \\ \hline
DSEBMs & 14.293               & 46.933                 & 25.194                 & 41.083                 & 33.852                 \\ \hline
DAGMM  & 21.066               & 71.632                 & 41.211                 & 83.551                 & 60.720                 \\ \hline
RSRAE  & 6.305                & 33.853                 & 10.940                 & 32.061                 & 18.869                 \\ \hline
\end{tabular}}

\end{table}

\section{Additional results}

We include some supplementary numerical results. In \Secref{subsec:tinyimagenetwithoutdeep} we show the results for Tiny Imagenet without deep features. In \Secref{subsec:cprnormnotshown} we extend the results reported in \secref{subsec:cprnorm} for the other datasets.

\subsection{Tiny Imagenet without deep features}
% \label{sec:tinyimagenet}
\label{subsec:tinyimagenetwithoutdeep}

Fig.~\ref{fig:tinyimagenet} presents the results for Tiny Imagenet without deep features. We see that RSRAE performs the best, but in general all the methods do not perform well. Indeed, the performance is significantly worse to that with deep features.

\begin{figure}[ht]
\centering
\begin{minipage}[t]{0.48\textwidth}
\rotatebox{90}{\null \qquad Tiny Imagenet}
\centering
\includegraphics[width=6cm]{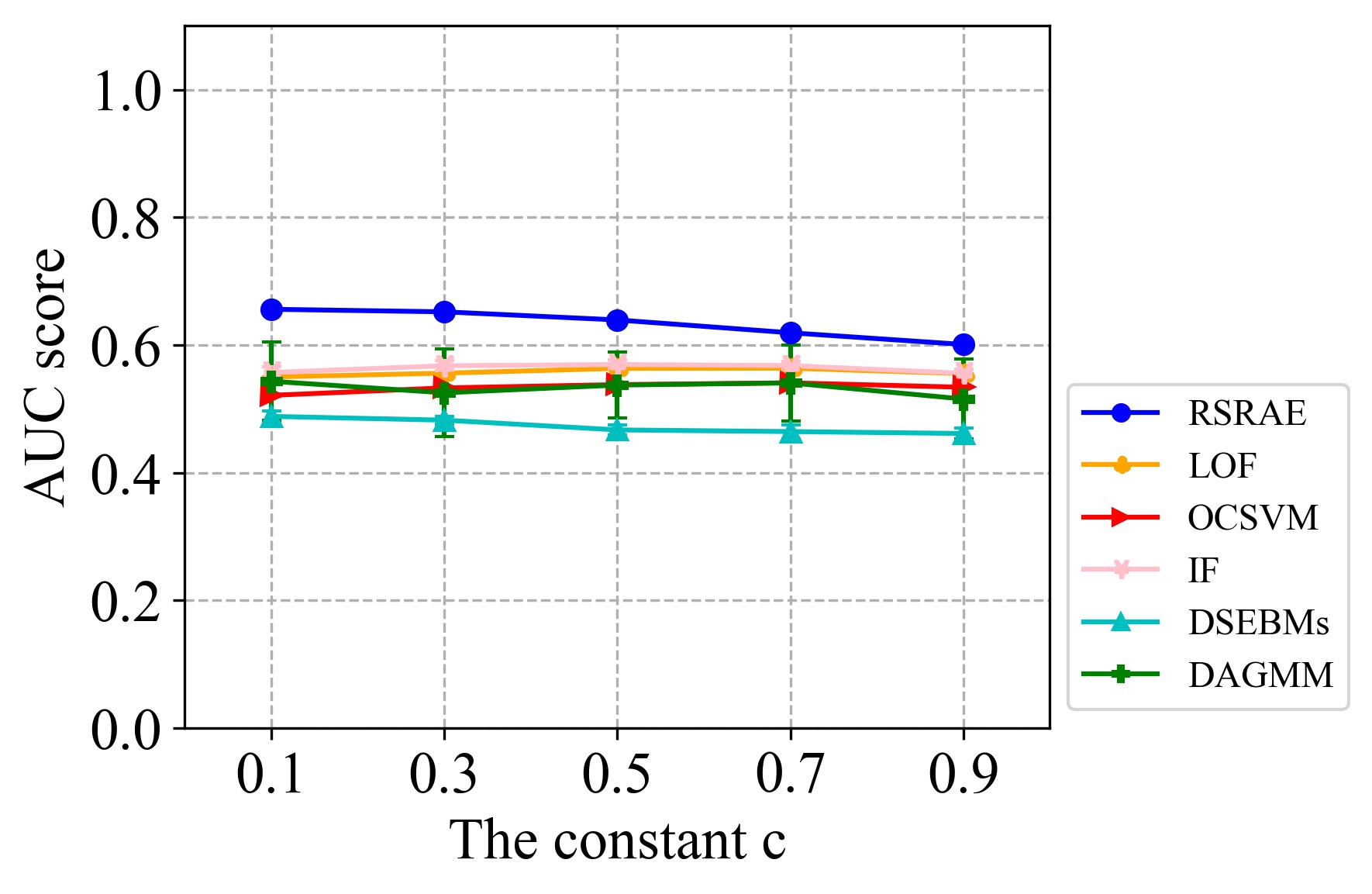}
\end{minipage}
\begin{minipage}[t]{0.48\textwidth}
\centering
\includegraphics[width=6cm]{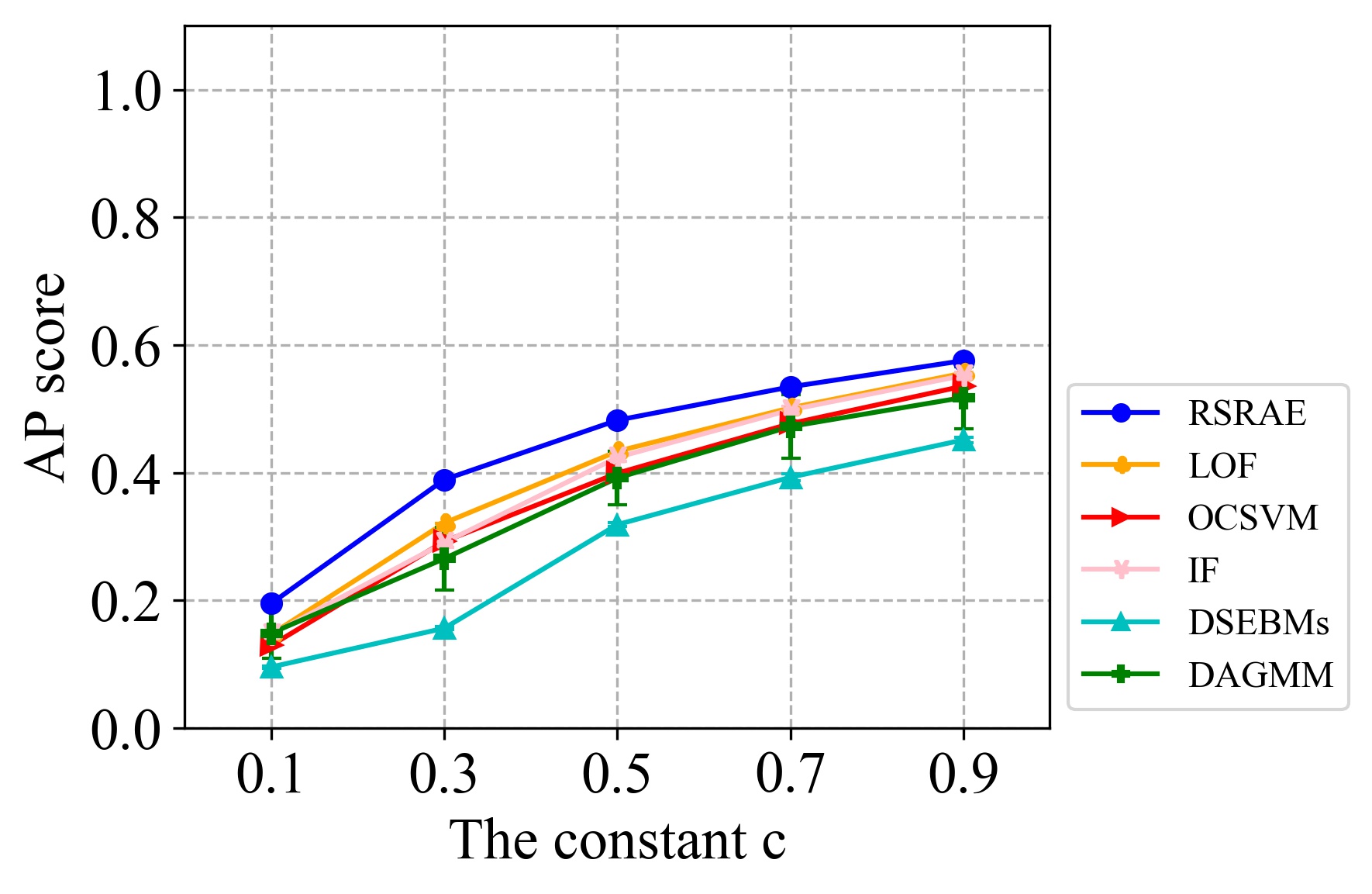}
\end{minipage}
\caption{AUC and AP scores for the Tiny Imagenet without using the deep features.}
\label{fig:tinyimagenet}
\end{figure}

\subsection{Additional comparison with variations of RSRAE}
\label{subsec:cprnormnotshown}

Figs.~\ref{fig:cprrest} and \ref{fig:cprrest2}  extend the comparisons in \Secref{subsec:cprnorm} for additional datasets. The conclusion is the same. In general, RSRAE performs better by a large margin than AE and AE-1. On the other hand, RSRAE+ is often in between RSRAE and AE/AE-1. However, for 20 Newsgroups, RSRAE+ performs similarly to RSRAE, and possibly slightly better, than RSRAE. It seems that in this case our choice of $\lambda_1$ and $\lambda_2$ is good.

\begin{figure}[ht]

\centering
\begin{minipage}[t]{0.48\textwidth}
\rotatebox{90}{\null \qquad Fashion MNIST}
\centering
\includegraphics[width=6cm]{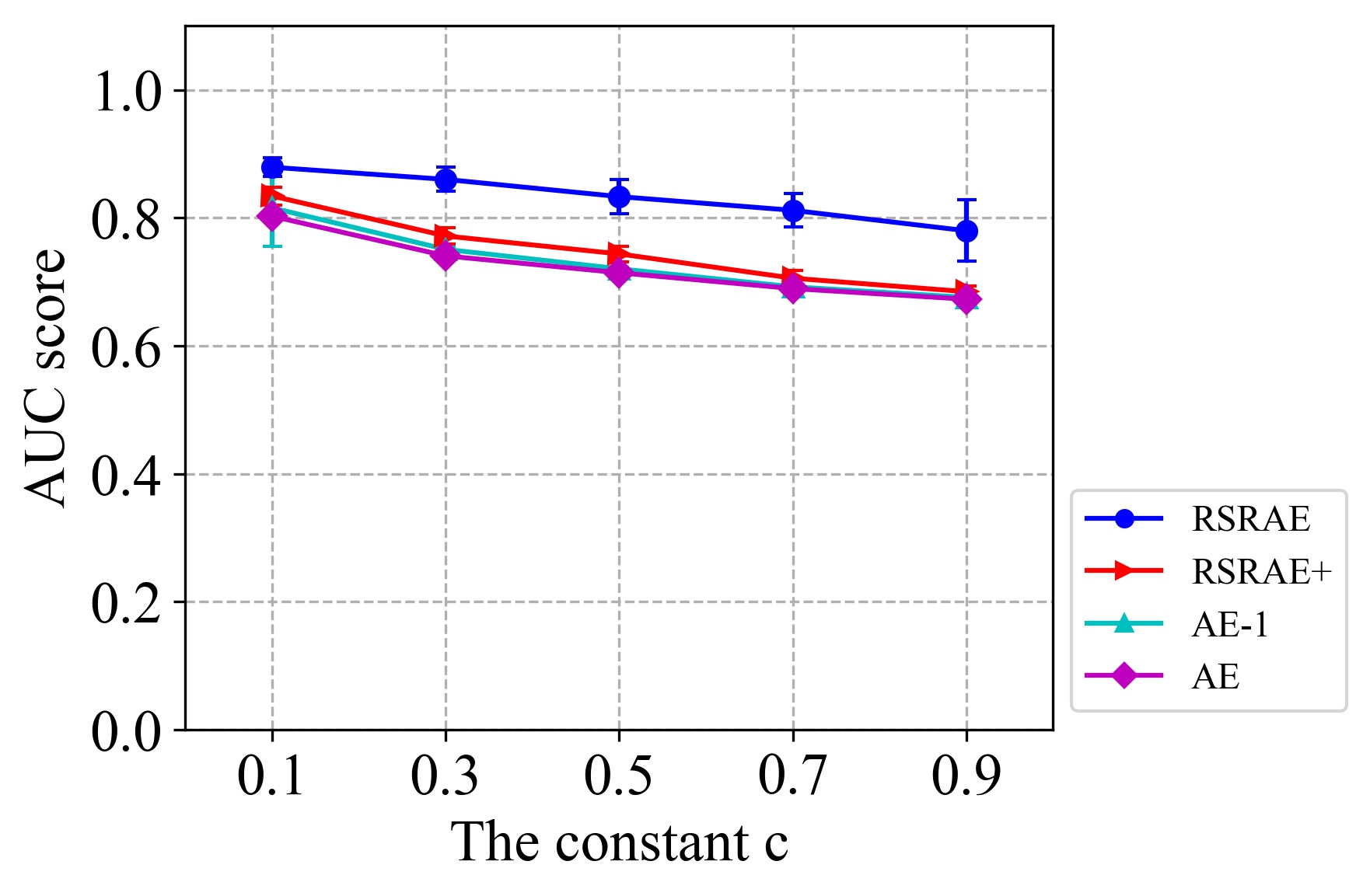}
\end{minipage}
\begin{minipage}[t]{0.48\textwidth}
\centering
\includegraphics[width=6cm]{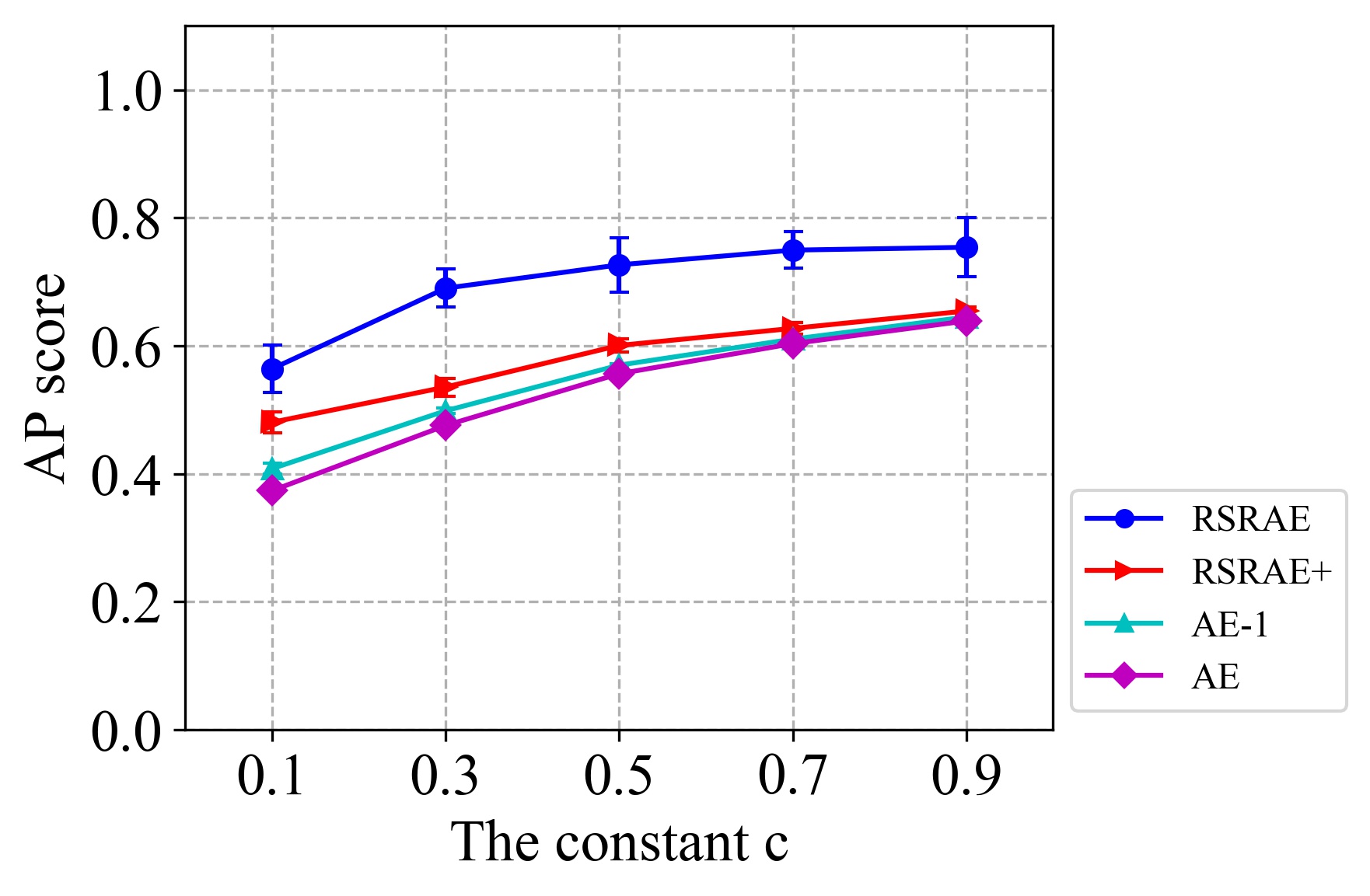}
\end{minipage}

\centering
\begin{minipage}[t]{0.48\textwidth}
\rotatebox{90}{\null \qquad Tiny Imagenet (deep)}
\centering
\includegraphics[width=6cm]{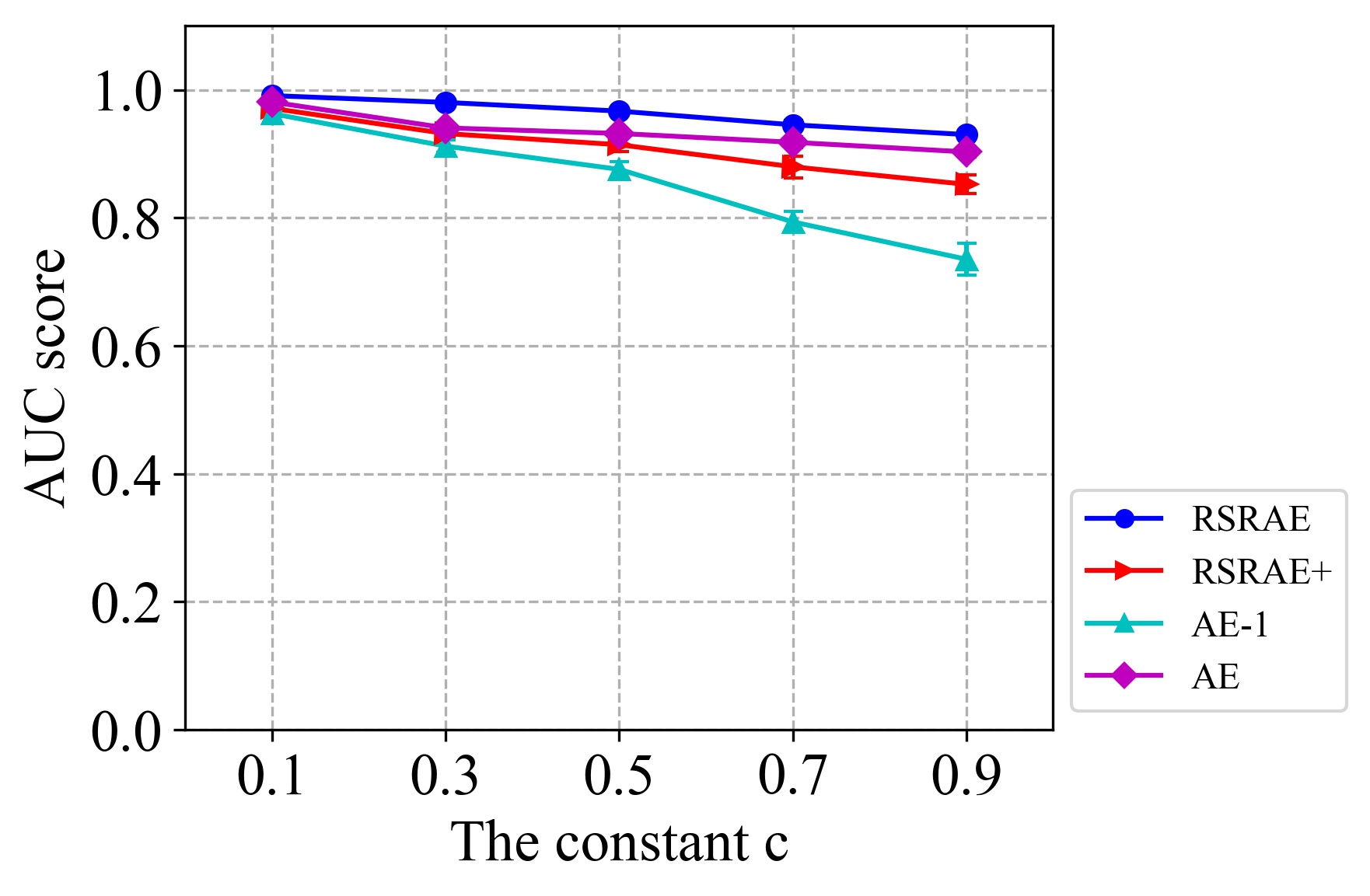}
\end{minipage}
\begin{minipage}[t]{0.48\textwidth}
\centering
\includegraphics[width=6cm]{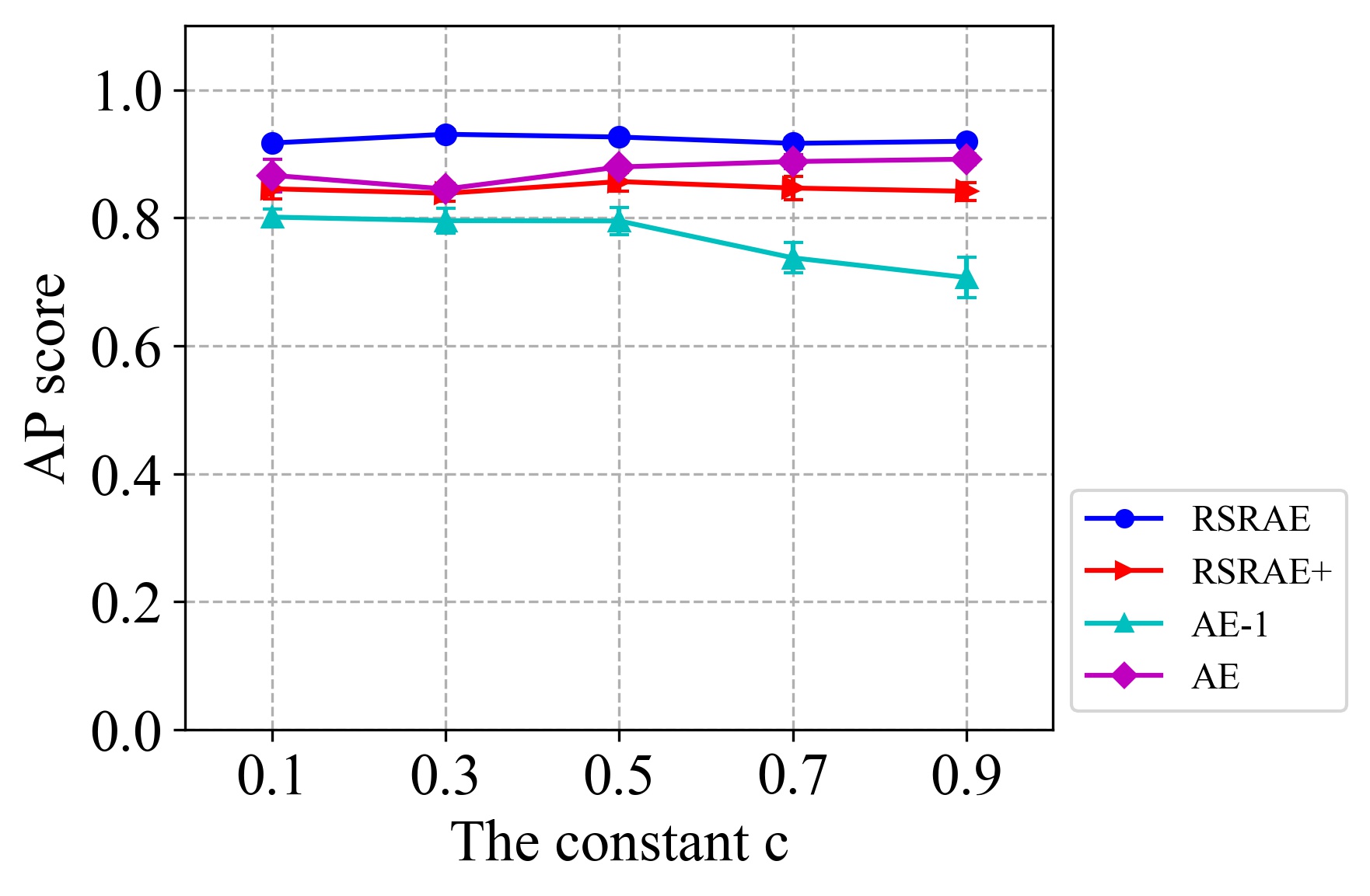}
\end{minipage}

\caption{AUC and AP scores for RSRAE and alternative formulations using Fashion MNIST and deep features of Tiny Imagenet, where $c=0.5$.}
\label{fig:cprrest}
\end{figure}
    
\centering
\begin{minipage}[t]{0.48\textwidth}
\rotatebox{90}{\null \qquad Tiny Imagenet}
\centering
\includegraphics[width=6cm]{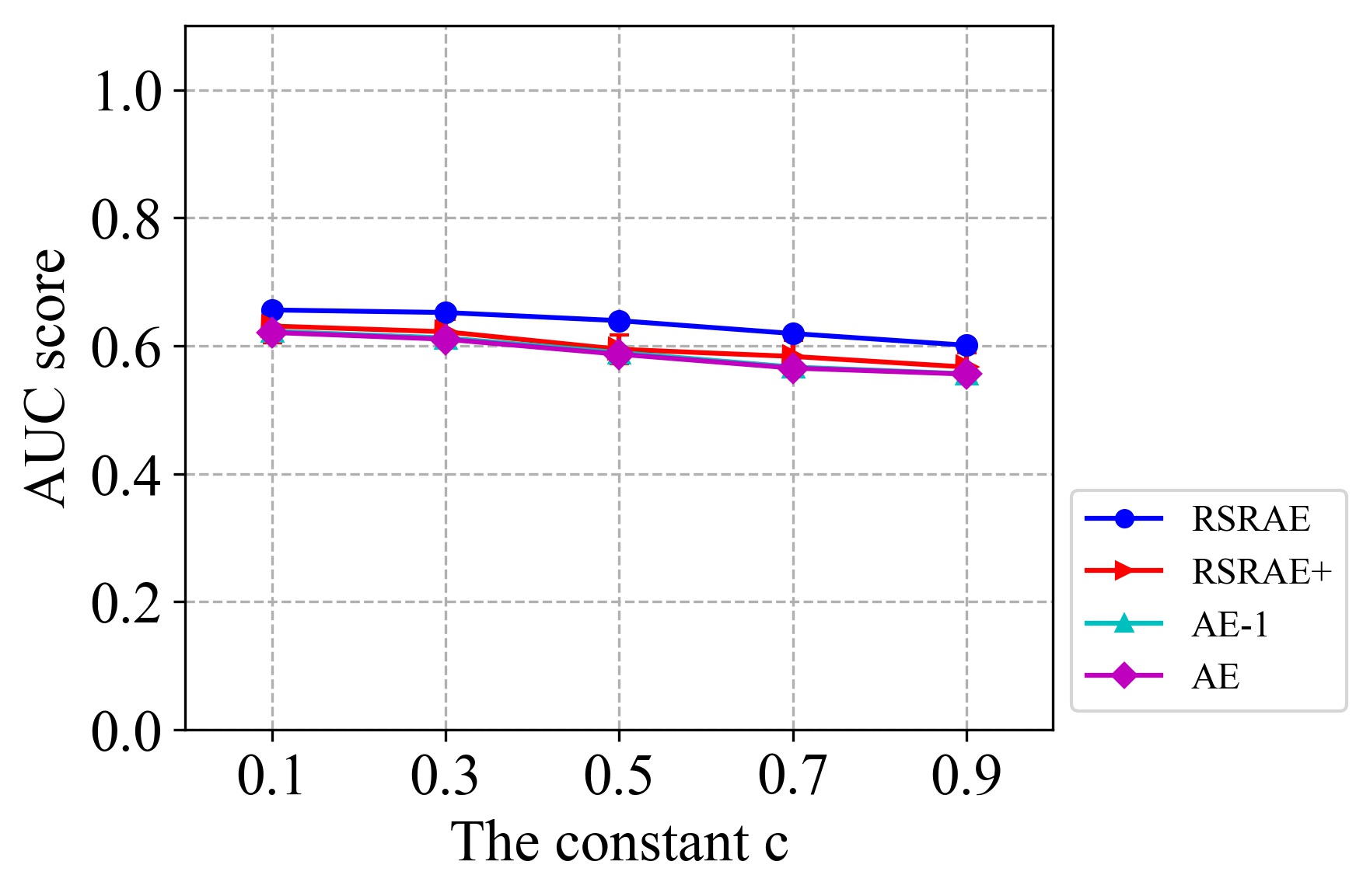}
\end{minipage}
\begin{minipage}[t]{0.48\textwidth}
\centering
\includegraphics[width=6cm]{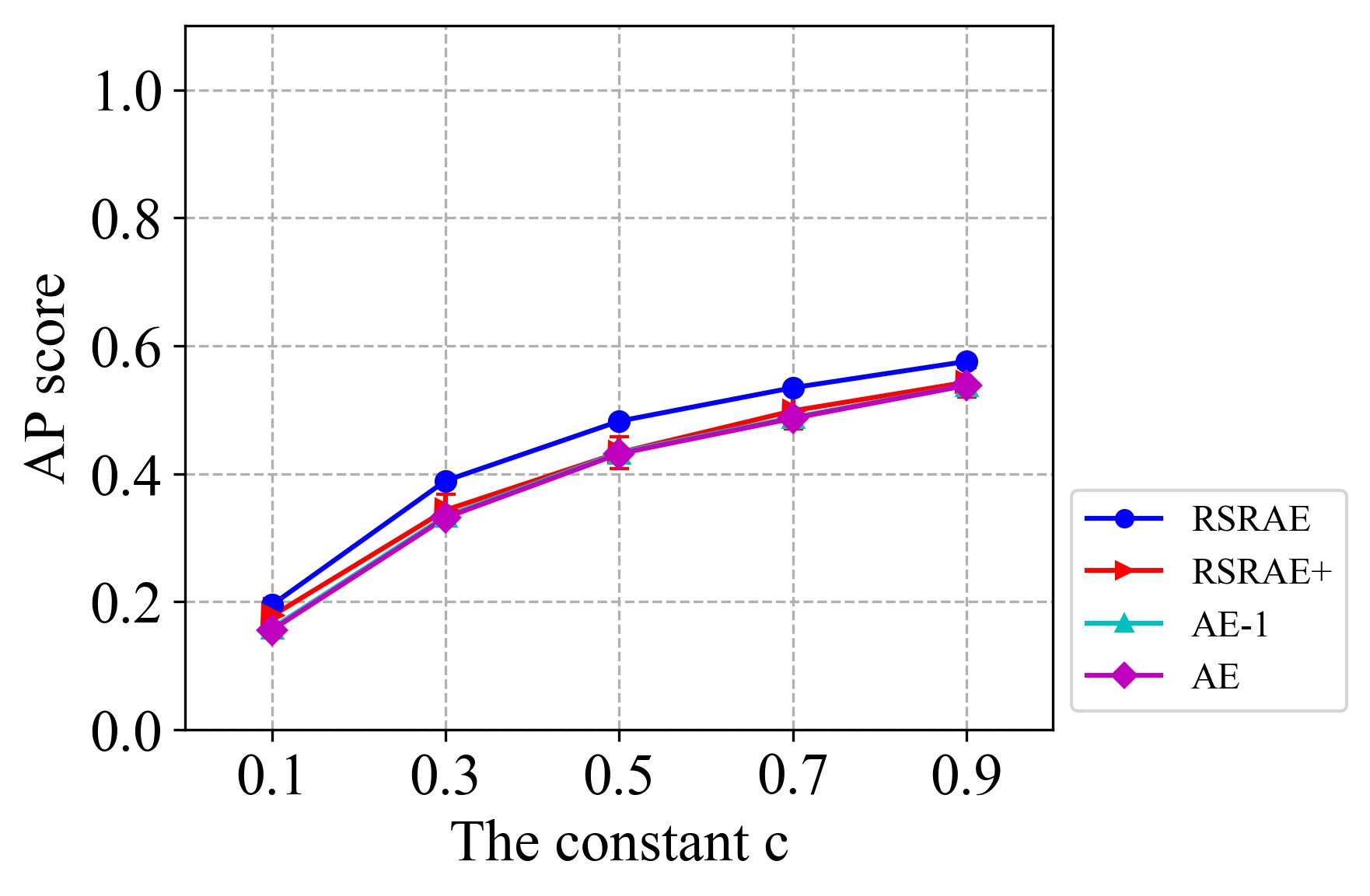}
\end{minipage}

\begin{figure}[ht]%\ContinuedFloat

\centering
\begin{minipage}[t]{0.48\textwidth}
\rotatebox{90}{\null \qquad 20 Newsgroups}
\centering
\includegraphics[width=6cm]{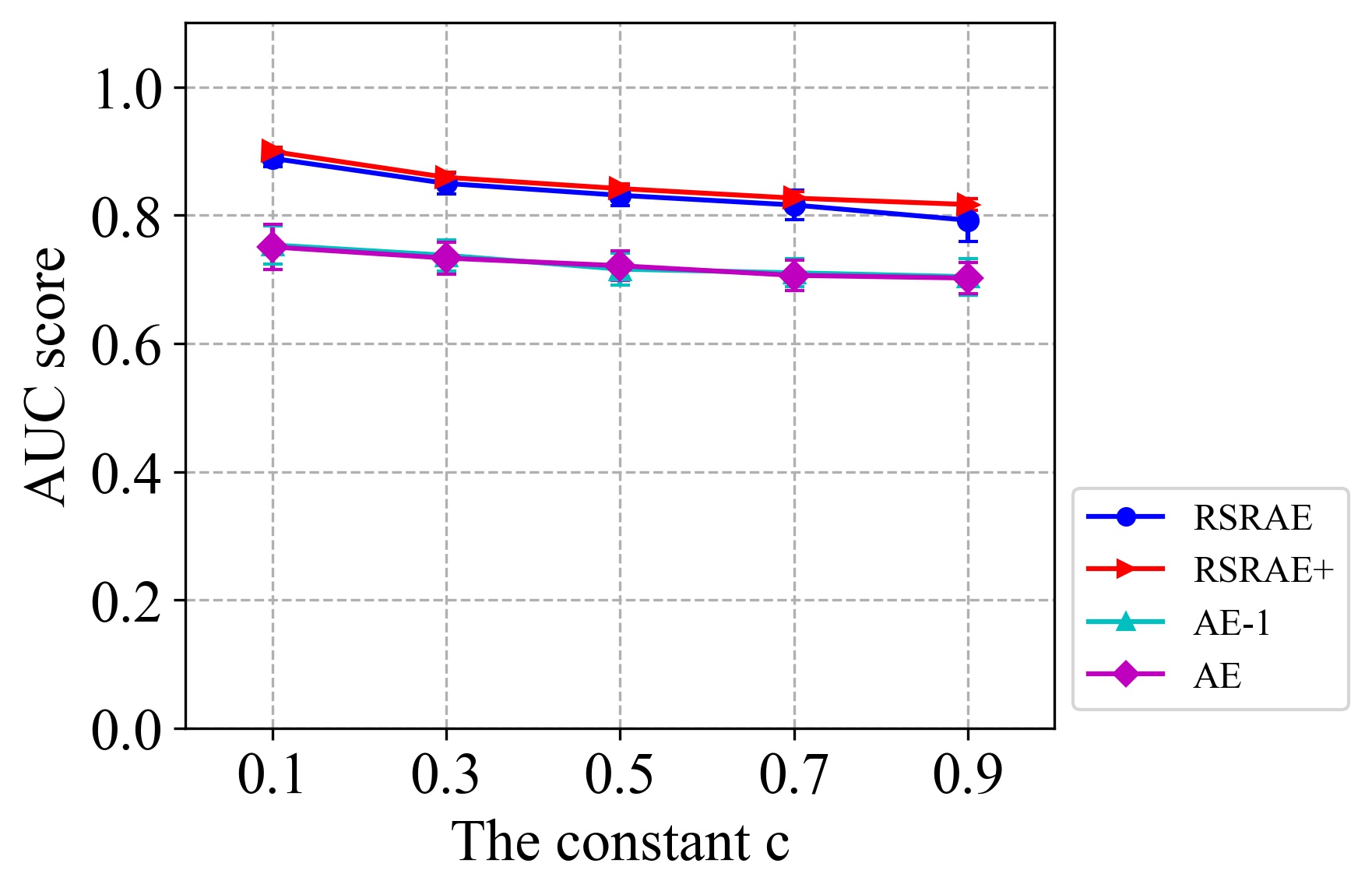}
\end{minipage}
\begin{minipage}[t]{0.48\textwidth}
\centering
\includegraphics[width=6cm]{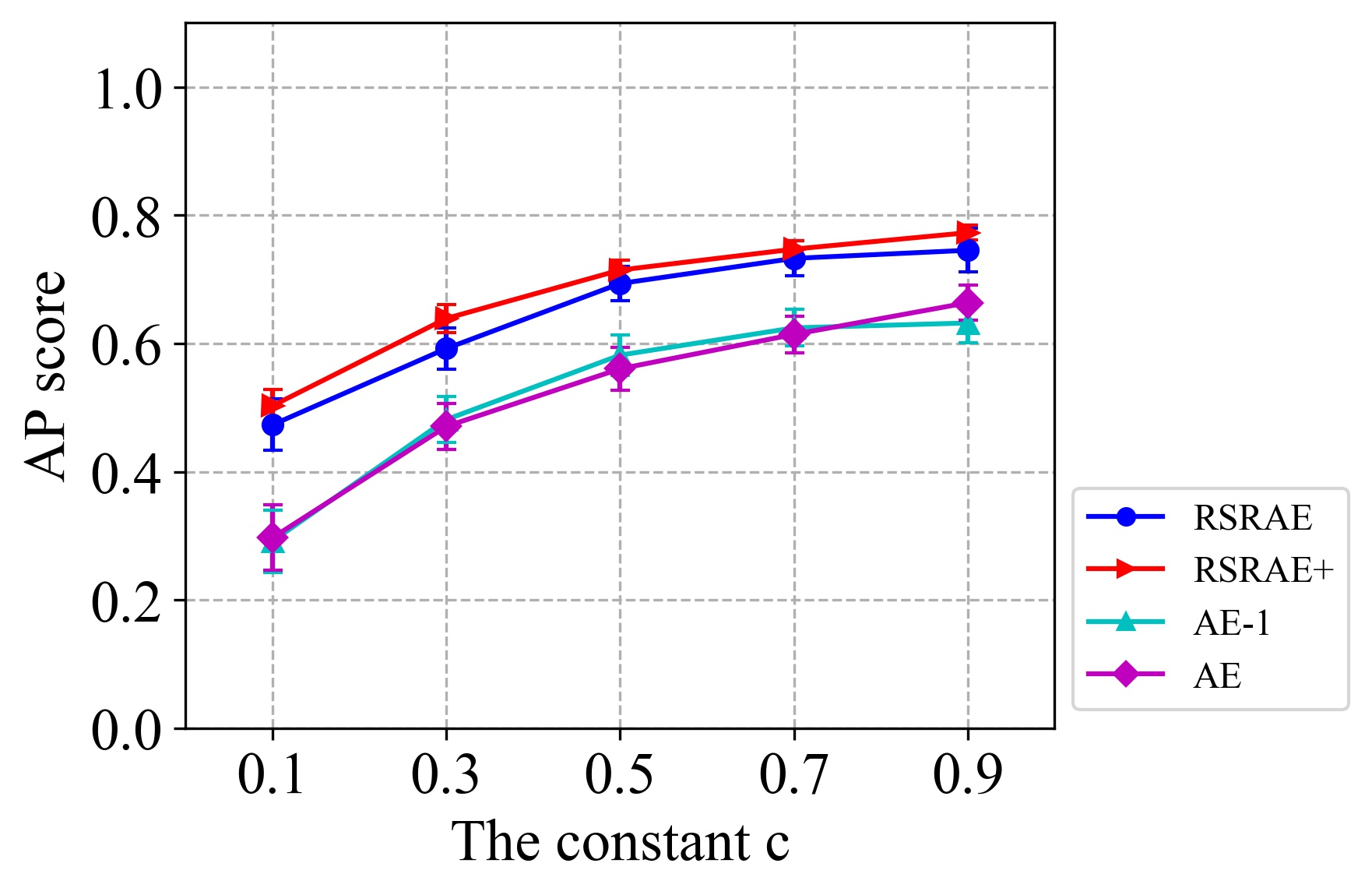}
\end{minipage}

\caption{
AUC and AP scores for RSRAE and alternative formulations using Tiny Imagenet (images) and 20 Newsgroup, where $c=0.5$.}

\label{fig:cprrest2}

\end{figure}

\vspace{10cm}
\null

\end{document}

%% file: math_commands.tex
\usepackage{amsmath,amsfonts,bm}

\def\secref#1{section~\ref{#1}}
\def\Secref#1{Section~\ref{#1}}
\def\eqref#1{(\ref{#1})}
\def\1{\bm{1}}

\def\eps{{\epsilon}}
\def\LL{{\mathscr{L}}}

\def\rvepsilon{{\mathbf{\epsilon}}}

\def\rvm{{\mathbf{m}}}

\def\rvmu{{\mathbf{\mu}}}
\def\rvnu{{\mathbf{\nu}}}

\def\rvx{{\mathbf{x}}}
\def\rvy{{\mathbf{y}}}
\def\rvz{{\mathbf{z}}}

\def\rmA{{\mathbf{A}}}

\def\rmD{{\mathbf{D}}}
\def\rmE{{\mathbf{E}}}

\def\rmI{{\mathbf{I}}}

\def\rmO{{\mathbf{O}}}
\def\rmP{{\mathbf{P}}}

\def\rmSigma{{\mathbf{\Sigma}}}

\def\rmU{{\mathbf{U}}}

\def\rmX{{\mathbf{X}}}

\def\rmZ{{\mathbf{Z}}}

\def\vtheta{{\bm{\theta}}}
\def\vvarphi{{\bm{\varphi}}}

\DeclareMathAlphabet{\mathsfit}{\encodingdefault}{\sfdefault}{m}{sl}
\SetMathAlphabet{\mathsfit}{bold}{\encodingdefault}{\sfdefault}{bx}{n}

\newcommand{\E}{\mathbb{E}}

\newcommand{\R}{\mathbb{R}}

%% file: iclr2020_conference.bbl
\begin{thebibliography}{55}
\providecommand{\natexlab}[1]{#1}
\providecommand{\url}[1]{\texttt{#1}}
\expandafter\ifx\csname urlstyle\endcsname\relax
  \providecommand{\doi}[1]{doi: #1}\else
  \providecommand{\doi}{doi: \begingroup \urlstyle{rm}\Url}\fi

\bibitem[Amer et~al.(2013)Amer, Goldstein, and Abdennadher]{amer2013enhancing}
Mennatallah Amer, Markus Goldstein, and Slim Abdennadher.
\newblock Enhancing one-class support vector machines for unsupervised anomaly
  detection.
\newblock In \emph{Proceedings of the ACM SIGKDD Workshop on Outlier Detection
  and Description}, pp.\  8--15. ACM, 2013.

\bibitem[An \& Cho(2015)An and Cho]{an2015variational}
Jinwon An and Sungzoon Cho.
\newblock Variational autoencoder based anomaly detection using reconstruction
  probability.
\newblock \emph{Speial Lecture on IE}, 2:\penalty0 1--18, 2015.

\bibitem[Arjovsky et~al.(2017)Arjovsky, Chintala, and
  Bottou]{arjovsky2017wasserstein}
Martin Arjovsky, Soumith Chintala, and L{\'e}on Bottou.
\newblock {W}asserstein generative adversarial networks.
\newblock In \emph{Proceedings of the 34th International Conference on Machine
  Learning}, volume~70 of \emph{Proceedings of Machine Learning Research}, pp.\
   214--223, International Convention Centre, Sydney, Australia, 06--11 Aug
  2017. PMLR.
\newblock URL \url{http://proceedings.mlr.press/v70/arjovsky17a.html}.

\bibitem[Aytekin et~al.(2018)Aytekin, Ni, Cricri, and
  Aksu]{aytekin2018clustering}
Caglar Aytekin, Xingyang Ni, Francesco Cricri, and Emre Aksu.
\newblock Clustering and unsupervised anomaly detection with $l_2$ normalized
  deep auto-encoder representations.
\newblock In \emph{2018 International Joint Conference on Neural Networks
  ({IJCNN})}, pp.\  1--6, July 2018.
\newblock \doi{10.1109/IJCNN.2018.8489068}.

\bibitem[Breunig et~al.(2000)Breunig, Kriegel, Ng, and Sander]{breunig2000lof}
Markus~M Breunig, Hans-Peter Kriegel, Raymond~T Ng, and J{\"o}rg Sander.
\newblock {LOF}: identifying density-based local outliers.
\newblock In \emph{ACM sigmod record}, volume 29, 2, pp.\  93--104. ACM, 2000.

\bibitem[Chalapathy et~al.(2017)Chalapathy, Menon, and
  Chawla]{chalapathy2017robust}
Raghavendra Chalapathy, Aditya~Krishna Menon, and Sanjay Chawla.
\newblock Robust, deep and inductive anomaly detection.
\newblock In \emph{Joint European Conference on Machine Learning and Knowledge
  Discovery in Databases}, pp.\  36--51. Springer, 2017.

\bibitem[Chandola et~al.(2009)Chandola, Banerjee, and
  Kumar]{chandola2009anomaly}
Varun Chandola, ArindaFm Banerjee, and Vipin Kumar.
\newblock Anomaly detection: A survey.
\newblock \emph{ACM computing surveys (CSUR)}, 41\penalty0 (3):\penalty0 15,
  2009.

\bibitem[David \& Semmes(1993)David and Semmes]{DS93}
Guy David and Stephen Semmes.
\newblock \emph{Analysis of and on uniformly rectifiable sets}, volume~38 of
  \emph{Mathematical surveys and monographs}.
\newblock American Mathematical Society, Providence, RI, 1993.

\bibitem[Davis \& Goadrich(2006)Davis and Goadrich]{davis2006relationship}
Jesse Davis and Mark Goadrich.
\newblock The relationship between precision-recall and roc curves.
\newblock In \emph{Proceedings of the 23rd International Conference on Machine
  Learning}, ICML '06, pp.\  233--240, New York, NY, USA, 2006. ACM.
\newblock ISBN 1-59593-383-2.
\newblock \doi{10.1145/1143844.1143874}.
\newblock URL \url{http://doi.acm.org/10.1145/1143844.1143874}.

\bibitem[De~La~Torre \& Black(2003)De~La~Torre and Black]{Torre:03}
Fernando De~La~Torre and Michael~J Black.
\newblock A framework for robust subspace learning.
\newblock \emph{International Journal of Computer Vision}, 54\penalty0
  (1-3):\penalty0 117--142, 2003.

\bibitem[Ding et~al.(2006)Ding, Zhou, He, and Zha]{ding2006r}
Chris Ding, Ding Zhou, Xiaofeng He, and Hongyuan Zha.
\newblock {R1}-{PCA}: rotational invariant $l_1$-norm principal component
  analysis for robust subspace factorization.
\newblock In \emph{Proceedings of the 23rd international conference on Machine
  learning}, pp.\  281--288. ACM, 2006.

\bibitem[Fei-Fei et~al.(2007)Fei-Fei, Fergus, and Perona]{fei2007learning}
Li~Fei-Fei, Rob Fergus, and Pietro Perona.
\newblock Learning generative visual models from few training examples: An
  incremental bayesian approach tested on 101 object categories.
\newblock \emph{Computer vision and Image understanding}, 106\penalty0
  (1):\penalty0 59--70, 2007.

\bibitem[Golan \& El-Yaniv(2018)Golan and El-Yaniv]{golan2018deep}
Izhak Golan and Ran El-Yaniv.
\newblock Deep anomaly detection using geometric transformations.
\newblock In \emph{Advances in Neural Information Processing Systems}, pp.\
  9781--9791, 2018.

\bibitem[Goldstein \& Uchida(2016)Goldstein and
  Uchida]{goldstein2016comparative}
Markus Goldstein and Seiichi Uchida.
\newblock A comparative evaluation of unsupervised anomaly detection algorithms
  for multivariate data.
\newblock \emph{PloS one}, 11\penalty0 (4):\penalty0 e0152173, 2016.

\bibitem[Goodfellow et~al.(2014)Goodfellow, Pouget-Abadie, Mirza, Xu,
  Warde-Farley, Ozair, Courville, and Bengio]{goodfellow2014generative}
Ian Goodfellow, Jean Pouget-Abadie, Mehdi Mirza, Bing Xu, David Warde-Farley,
  Sherjil Ozair, Aaron Courville, and Yoshua Bengio.
\newblock Generative adversarial nets.
\newblock In \emph{Advances in neural information processing systems}, pp.\
  2672--2680, 2014.

\bibitem[Goodfellow et~al.(2016)Goodfellow, Bengio, and
  Courville]{goodfellow2016deep}
Ian Goodfellow, Yoshua Bengio, and Aaron Courville.
\newblock \emph{Deep learning}.
\newblock MIT press, 2016.

\bibitem[Gulrajani et~al.(2017)Gulrajani, Ahmed, Arjovsky, Dumoulin, and
  Courville]{gulrajani2017improved}
Ishaan Gulrajani, Faruk Ahmed, Martin Arjovsky, Vincent Dumoulin, and Aaron~C
  Courville.
\newblock Improved training of {W}asserstein {GAN}s.
\newblock In \emph{Advances in Neural Information Processing Systems}, pp.\
  5767--5777, 2017.

\bibitem[He et~al.(2016)He, Zhang, Ren, and Sun]{he2016identity}
Kaiming He, Xiangyu Zhang, Shaoqing Ren, and Jian Sun.
\newblock Identity mappings in deep residual networks.
\newblock In \emph{European conference on computer vision}, pp.\  630--645.
  Springer, 2016.

\bibitem[Ji et~al.(2017)Ji, Zhang, Li, Salzmann, and Reid]{ji2017deep}
Pan Ji, Tong Zhang, Hongdong Li, Mathieu Salzmann, and Ian Reid.
\newblock Deep subspace clustering networks.
\newblock In \emph{Advances in Neural Information Processing Systems}, pp.\
  24--33, 2017.

\bibitem[Jones(1990)]{Jones90}
Peter~W Jones.
\newblock Rectifiable sets and the traveling salesman problem.
\newblock \emph{Invent Math}, 102\penalty0 (1):\penalty0 1--15, 1990.

\bibitem[Kannan et~al.(2017)Kannan, Woo, Aggarwal, and Park]{kannan2017outlier}
Ramakrishnan Kannan, Hyenkyun Woo, Charu~C. Aggarwal, and Haesun Park.
\newblock Outlier detection for text data.
\newblock In \emph{Proceedings of the 2017 SIAM International Conference on
  Data Mining}, pp.\  489--497, 2017.
\newblock URL \url{https://epubs.siam.org/doi/abs/10.1137/1.9781611974973.55}.

\bibitem[Kingma \& Ba(2014)Kingma and Ba]{kingma2014adam}
Diederik~P Kingma and Jimmy Ba.
\newblock Adam: A method for stochastic optimization.
\newblock In \emph{3rd International Conference for Learning Representations.
  arXiv:1412.6980}, 2014.

\bibitem[Kingma \& Welling(2013)Kingma and Welling]{kingma2013auto}
Diederik~P Kingma and Max Welling.
\newblock Auto-encoding variational {B}ayes.
\newblock In \emph{2nd International Conference for Learning Representations.
  arXiv:1312.6114}, 2013.

\bibitem[Kliger \& Fleishman(2018)Kliger and Fleishman]{kliger2018novelty}
Mark Kliger and Shachar Fleishman.
\newblock Novelty detection with {GAN}, 2018.
\newblock URL \url{https://openreview.net/forum?id=Hy7EPh10W}.

\bibitem[Lang(1995)]{Lang95}
Ken Lang.
\newblock Newsweeder: Learning to filter netnews.
\newblock In \emph{Proceedings of the Twelfth International Conference on
  Machine Learning}, pp.\  331--339, 1995.

\bibitem[Lerman(2003)]{Lerman03}
Gilad Lerman.
\newblock Quantifying curvelike structures of measures by using {$L\sb 2$}
  {J}ones quantities.
\newblock \emph{Comm. Pure Appl. Math.}, 56\penalty0 (9):\penalty0 1294--1365,
  2003.
\newblock ISSN 0010-3640.

\bibitem[Lerman \& Maunu(2017)Lerman and Maunu]{lerman2017fast}
Gilad Lerman and Tyler Maunu.
\newblock Fast, robust and non-convex subspace recovery.
\newblock \emph{Information and Inference: A Journal of the IMA}, 7\penalty0
  (2):\penalty0 277--336, 2017.

\bibitem[Lerman \& Maunu(2018)Lerman and Maunu]{lerman2018overview}
Gilad Lerman and Tyler Maunu.
\newblock An overview of robust subspace recovery.
\newblock \emph{Proceedings of the IEEE}, 106\penalty0 (8):\penalty0
  1380--1410, 2018.

\bibitem[Lerman \& Zhang(2014)Lerman and Zhang]{lp_recovery_part1_11}
Gilad Lerman and Teng Zhang.
\newblock {$l_p$}-recovery of the most significant subspace among multiple
  subspaces with outliers.
\newblock \emph{Constructive Approximation}, 40\penalty0 (3):\penalty0
  329--385, 2014.

\bibitem[Lerman et~al.(2015)Lerman, McCoy, Tropp, and Zhang]{lerman2015robust}
Gilad Lerman, Michael~B McCoy, Joel~A Tropp, and Teng Zhang.
\newblock Robust computation of linear models by convex relaxation.
\newblock \emph{Foundations of Computational Mathematics}, 15\penalty0
  (2):\penalty0 363--410, 2015.

\bibitem[Lewis(1997)]{lewis1997reuters}
David Lewis.
\newblock Reuters-21578 text categorization test collection.
\newblock \emph{Distribution 1.0, AT\&T Labs-Research}, 1997.

\bibitem[Lezama et~al.(2018)Lezama, Qiu, Mus{\'e}, and Sapiro]{lezama2018ole}
Jos{\'e} Lezama, Qiang Qiu, Pablo Mus{\'e}, and Guillermo Sapiro.
\newblock {OL\'E}: Orthogonal low-rank embedding-a plug and play geometric loss
  for deep learning.
\newblock In \emph{Proceedings of the IEEE Conference on Computer Vision and
  Pattern Recognition}, pp.\  8109--8118, 2018.

\bibitem[Liu et~al.(2012)Liu, Ting, and Zhou]{liu2012isolation}
Fei~Tony Liu, Kai~Ming Ting, and Zhi-Hua Zhou.
\newblock Isolation-based anomaly detection.
\newblock \emph{ACM Transactions on Knowledge Discovery from Data (TKDD)},
  6\penalty0 (1):\penalty0 3, 2012.

\bibitem[Maunu \& Lerman(2019)Maunu and Lerman]{maunu2019robust}
Tyler Maunu and Gilad Lerman.
\newblock Robust subspace recovery with adversarial outliers.
\newblock \emph{CoRR}, abs/1904.03275, 2019.
\newblock URL \url{http://arxiv.org/abs/1904.03275}.

\bibitem[Maunu et~al.(2017)Maunu, Zhang, and Lerman]{maunu2017well}
Tyler Maunu, Teng Zhang, and Gilad Lerman.
\newblock A well-tempered landscape for non-convex robust subspace recovery.
\newblock \emph{arXiv preprint arXiv:1706.03896}, 2017.

\bibitem[McCoy \& Tropp(2011)McCoy and Tropp]{mccoy2011two}
Michael McCoy and Joel~A Tropp.
\newblock Two proposals for robust {PCA} using semidefinite programming.
\newblock \emph{Electronic Journal of Statistics}, 5:\penalty0 1123--1160,
  2011.

\bibitem[Paffenroth et~al.(2018)Paffenroth, Kay, and
  Servi]{paffenroth2018robust}
Randy Paffenroth, Kathleen Kay, and Les Servi.
\newblock Robust {PCA} for anomaly detection in cyber networks.
\newblock \emph{arXiv preprint arXiv:1801.01571}, 2018.

\bibitem[Pedregosa et~al.(2011)Pedregosa, Varoquaux, Gramfort, Michel, Thirion,
  Grisel, Blondel, Prettenhofer, Weiss, Dubourg, VanderPlas, Passos,
  Cournapeau, Brucher, Perrot, and Duchesnay]{scikit-learn}
Fabian Pedregosa, Ga{\"{e}}l Varoquaux, Alexandre Gramfort, Vincent Michel,
  Bertrand Thirion, Olivier Grisel, Mathieu Blondel, Peter Prettenhofer, Ron
  Weiss, Vincent Dubourg, Jake VanderPlas, Alexandre Passos, David Cournapeau,
  Matthieu Brucher, Matthieu Perrot, and Edouard Duchesnay.
\newblock Scikit-learn: Machine learning in {P}ython.
\newblock \emph{Journal of Machine Learning Research}, 12:\penalty0 2825--2830,
  2011.

\bibitem[Rajaraman \& Ullman(2011)Rajaraman and Ullman]{rajaraman2011mining}
Anand Rajaraman and Jeffrey~David Ullman.
\newblock \emph{Mining of massive datasets}.
\newblock Cambridge University Press, 2011.

\bibitem[Russakovsky et~al.(2015)Russakovsky, Deng, Su, Krause, Satheesh, Ma,
  Huang, Karpathy, Khosla, Bernstein, et~al.]{russakovsky2015imagenet}
Olga Russakovsky, Jia Deng, Hao Su, Jonathan Krause, Sanjeev Satheesh, Sean Ma,
  Zhiheng Huang, Andrej Karpathy, Aditya Khosla, Michael Bernstein, et~al.
\newblock Imagenet large scale visual recognition challenge.
\newblock \emph{International journal of computer vision}, 115\penalty0
  (3):\penalty0 211--252, 2015.

\bibitem[Sch{\"o}lkopf et~al.(2000)Sch{\"o}lkopf, Williamson, Smola,
  Shawe-Taylor, and Platt]{scholkopf2000support}
Bernhard Sch{\"o}lkopf, Robert~C Williamson, Alex~J Smola, John Shawe-Taylor,
  and John~C Platt.
\newblock Support vector method for novelty detection.
\newblock In \emph{Advances in neural information processing systems}, pp.\
  582--588, 2000.

\bibitem[Shyu et~al.(2003)Shyu, Chen, Sarinnapakorn, and
  Chang]{shyu2003novelnew}
Mei-Ling Shyu, Shu-Ching Chen, Kanoksri Sarinnapakorn, and LiWu Chang.
\newblock A novel anomaly detection scheme based on principal component
  classifier.
\newblock In \emph{Proc. ICDM Foundation and New Direction of Data Mining
  workshop, 2003}, pp.\  172--179, 2003.

\bibitem[Vasilev et~al.(2018)Vasilev, Golkov, Lipp, Sgarlata, Tomassini, Jones,
  and Cremers]{vasilev2018q}
Aleksei Vasilev, Vladimir Golkov, Ilona Lipp, Eleonora Sgarlata, Valentina
  Tomassini, Derek~K Jones, and Daniel Cremers.
\newblock q-space novelty detection with variational autoencoders.
\newblock \emph{arXiv preprint arXiv:1806.02997}, 2018.

\bibitem[Vaswani \& Narayanamurthy(2018)Vaswani and
  Narayanamurthy]{vaswani2018static}
Namrata Vaswani and Praneeth Narayanamurthy.
\newblock Static and dynamic robust {PCA} and matrix completion: A review.
\newblock \emph{Proceedings of the IEEE}, 106\penalty0 (8):\penalty0
  1359--1379, 2018.

\bibitem[Watson(2001)]{watson2001some}
G.~Alistair Watson.
\newblock \emph{Some Problems in Orthogonal Distance and Non-Orthogonal
  Distance Regression}.
\newblock Defense Technical Information Center, 2001.
\newblock URL \url{http://books.google.com/books?id=WKKWGwAACAAJ}.

\bibitem[Wright et~al.(2009)Wright, Ganesh, Rao, Peng, and
  Ma]{wright2009robust}
John Wright, Arvind Ganesh, Shankar Rao, Yigang Peng, and Yi~Ma.
\newblock Robust principal component analysis: Exact recovery of corrupted
  low-rank matrices via convex optimization.
\newblock In \emph{Advances in neural information processing systems}, pp.\
  2080--2088, 2009.

\bibitem[Xia et~al.(2015)Xia, Cao, Wen, Hua, and Sun]{xia2015learning}
Yan Xia, Xudong Cao, Fang Wen, Gang Hua, and Jian Sun.
\newblock Learning discriminative reconstructions for unsupervised outlier
  removal.
\newblock In \emph{Proceedings of the IEEE International Conference on Computer
  Vision}, pp.\  1511--1519, 2015.

\bibitem[Xiao et~al.(2017)Xiao, Rasul, and Vollgraf]{xiao2017fashion}
Han Xiao, Kashif Rasul, and Roland Vollgraf.
\newblock Fashion-{MNIST}: a novel image dataset for benchmarking machine
  learning algorithms.
\newblock \emph{arXiv preprint arXiv:1708.07747}, 2017.

\bibitem[Xu et~al.(2012)Xu, Caramanis, and Sanghavi]{xu2012robust}
Huan Xu, Constantine Caramanis, and Sujay Sanghavi.
\newblock Robust {PCA} via outlier pursuit.
\newblock \emph{{IEEE} Trans. Information Theory}, 58\penalty0 (5):\penalty0
  3047--3064, 2012.
\newblock \doi{10.1109/TIT.2011.2173156}.

\bibitem[Zenati et~al.(2018)Zenati, Foo, Lecouat, Manek, and
  Chandrasekhar]{zenati2018efficient}
Houssam Zenati, Chuan~Sheng Foo, Bruno Lecouat, Gaurav Manek, and
  Vijay~Ramaseshan Chandrasekhar.
\newblock Efficient {GAN}-based anomaly detection, 2018.
\newblock URL \url{https://openreview.net/forum?id=BkXADmJDM}.

\bibitem[Zhai et~al.(2016)Zhai, Cheng, Lu, and Zhang]{zhai2016deep}
Shuangfei Zhai, Yu~Cheng, Weining Lu, and Zhongfei Zhang.
\newblock Deep structured energy based models for anomaly detection.
\newblock In \emph{Proceedings of the 33rd International Conference on
  International Conference on Machine Learning - Volume 48}, pp.\  1100--1109,
  2016.

\bibitem[Zhang \& Lerman(2014)Zhang and Lerman]{zhang2014novel}
Teng Zhang and Gilad Lerman.
\newblock A novel {M}-estimator for robust {PCA}.
\newblock \emph{Journal of Machine Learning Research}, 15\penalty0
  (1):\penalty0 749--808, 2014.

\bibitem[Zhang et~al.(2009)Zhang, Szlam, and Lerman]{zhang2009median}
Teng Zhang, Arthur Szlam, and Gilad Lerman.
\newblock Median {K}-flats for hybrid linear modeling with many outliers.
\newblock In \emph{Computer Vision Workshops (ICCV Workshops), 2009 IEEE 12th
  International Conference on}, pp.\  234--241. IEEE, 2009.

\bibitem[Zhou \& Paffenroth(2017)Zhou and Paffenroth]{zhou2017anomaly}
Chong Zhou and Randy~C Paffenroth.
\newblock Anomaly detection with robust deep autoencoders.
\newblock In \emph{Proceedings of the 23rd ACM SIGKDD International Conference
  on Knowledge Discovery and Data Mining}, pp.\  665--674. ACM, 2017.

\bibitem[Zong et~al.(2018)Zong, Song, Min, Cheng, Lumezanu, Cho, and
  Chen]{zong2018deep}
Bo~Zong, Qi~Song, Martin~Renqiang Min, Wei Cheng, Cristian Lumezanu, Daeki Cho,
  and Haifeng Chen.
\newblock Deep autoencoding gaussian mixture model for unsupervised anomaly
  detection.
\newblock In \emph{International Conference on Learning Representations}, 2018.
\newblock URL \url{https://openreview.net/forum?id=BJJLHbb0-}.

\end{thebibliography}
